\documentclass[lettersize,journal]{IEEEtran}
\usepackage{amsmath,amsfonts}
\usepackage{algorithm}
\usepackage{algpseudocode}
\usepackage{array}
\usepackage[caption=false,font=normalsize,labelfont=sf,textfont=sf]{subfig}
\usepackage{textcomp}
\usepackage{stfloats}
\usepackage{url}
\usepackage{verbatim}
\usepackage{graphicx}
\usepackage[capitalise]{cleveref}
\usepackage{cite}
\usepackage{tikz}
\usetikzlibrary{positioning, fit, backgrounds, shapes.arrows}
\usepackage{amsthm}
\usepackage{contour}
\usepackage{pifont}
\usepackage{booktabs}
\usepackage{overpic}
\usepackage{multirow}
\usepackage{array,makecell}
\usepackage{xcolor}
\usepackage[table]{xcolor}
\usepackage{sansmath}
\usepackage{wrapfig}

\usepackage{subcaption}
\usepackage{caption}

\theoremstyle{plain}
\newtheorem{theorem}{Theorem}

\newtheorem{corollary}[theorem]{Corollary}

\theoremstyle{remark}
\newtheorem*{interpretation}{Interpretation}
\newtheorem*{takeaway}{Takeaway}

\usepackage{pgfplots}
\pgfplotsset{
    compat=1.17,
    every axis/.append style={
        title={\small\sffamily #1}, 
        xlabel={\small\sffamily #1}, 
        ylabel={\small\sffamily #1}, 
        legend style={font=\small\sffamily}, 
        tick label style={font=\small\sffamily}, 
        title style={font=\small\sffamily}, 
        xticklabel style={font=\footnotesize\sansmath\sffamily}, 
        yticklabel style={font=\footnotesize\sansmath\sffamily} 
    },
    every axis label/.append style={font=\small\sffamily}
}

\definecolor{tabfirst}{rgb}{1, 0.7, 0.7} 
\definecolor{tabsecond}{rgb}{1, 0.85, 0.7} 
\definecolor{tabthird}{rgb}{1, 1, 0.7} 
\definecolor{customdarkgreen}{rgb}{0.0, 0.35, 0.0}

\begin{document}

\title{Unlearning the Unpromptable: Prompt-free \\Instance Unlearning in Diffusion Models}

\author{%
Kyungryeol Lee$^*$, 
Kyeonghyun Lee$^*$, 
Seongmin Hong$^*$, 
Byung Hyun Lee,
and Se Young Chun
\thanks{This work was supported in part by Institute of Information \& communications Technology Planning \& Evaluation (IITP) grants funded by the Korea government(MSIT) [NO.RS-2021-II211343, Artificial Intelligence Graduate School Program (Seoul National University)], (No.RS-2025-02314125, Effective Human-Machine Teaming With Multimodal Hazy Oracle Models). Also, the authors acknowledged the financial support from the BK21 FOUR program of the Education and Research Program for Future ICT Pioneers, Seoul National University. (Corresponding author: Se Young Chun)}
\thanks{KR Lee, KH Lee and BH Lee are with the Department of Electrical and Computer Engineering (ECE), Seoul National University (SNU), Seoul, 08826, Republic of Korea. (e-mail: \{kr.lee, litiphysics, ldlqudgus756\}@snu.ac.kr)}
\thanks{SM Hong is with Institute of New Media and Communications (INMC), SNU, Seoul, 08826, Republic of Korea. (e-mail: smhongok@snu.ac.kr)}
\thanks{SY Chun is with ECE, INMC, IPAI, SNU, Seoul, 08826, Republic of Korea. (e-mail: sychun@snu.ac.kr)}
\thanks{$^*$ KR Lee, KH Lee and SM Hong are equally contributed.}
}



\maketitle

\begin{abstract}
Machine unlearning aims to remove specific outputs from trained models, often at the concept level, such as forgetting all occurrences of a particular celebrity or filtering content via text prompts. However, many undesired outputs, such as an individual’s face or generations culturally or factually misinterpreted, cannot often be specified by text prompts. We address this underexplored setting of instance unlearning for outputs that are undesired but unpromptable, where the goal is to forget target outputs selectively while preserving the rest. To this end, we introduce an effective surrogate-based unlearning method that leverages image editing, timestep-aware weighting, and gradient surgery to guide trained diffusion models toward forgetting specific outputs. Experiments on conditional (Stable Diffusion 3) and unconditional (DDPM-CelebA) diffusion models demonstrate that our prompt-free method uniquely unlearns unpromptable outputs, such as faces and culturally inaccurate depictions, with preserved integrity, unlike prompt-based and prompt-free baselines. Our proposed method would serve as a practical hotfix for diffusion model providers to ensure privacy protection and ethical compliance.
\end{abstract}


\section{Introduction}
\IEEEPARstart{G}{enerative} models, particularly diffusion models (DMs), have become widely adopted in both image~\cite{nichol2022glide,chang2023muse,zhang2023adding,peebles2023scalable,esser2024scaling} and video \cite{hongcogvideo,blattmann2023stable,guoanimatediff,chen2024videocrafter2, wan2025wan} applications due to their high-quality and diverse outputs. Despite their success, it is becoming more important to control what these models should or should not generate, especially in situations involving ethics or privacy~\cite{dalle2preview2022,rando2022red,hunter2023ai,schramowski2023safe,somepalli2023diffusion,liu2025privacy}. Accordingly, prior efforts such as dataset curation \cite{rombach2022stable2.0, esser2024scaling} and post-generation filtering \cite{rando2022red, laborde2020nsfw} have been explored, but have remained costly ~\cite{rombach2022stable2.0} and easily circumvented~\cite{rando2022red}. These challenges have led to the concept of \emph{machine unlearning}, removing specific information by fine-tuning trained models~\cite{gandikota2023erasing, kumari2023ablating, fan2024salun, chundawat2023zero}.

To address the challenges, a series of prior arts have significantly enriched prompt-based unlearning approaches \cite{zhang2024forget, huang2023receler, gong2024reliable, buierasing, zhang2024defensive} to remove prompt-induced target concepts while preserving model integrity \cite{fan2024salun, lu2024mace, lyu2024one, park2024direct, lee2025concept,yoonsafree,thakral2025fine,fantastic}.
Specifically, these methods typically involve constructing prompts that reliably induce the generation of the target concepts, and then guiding the model to generate nonsensitive content from those prompts. Through this process, they shift the distribution of generated images associated with the target concepts toward distributions with benign or neutral concepts.

However, their reliance on prompts or conceptual descriptors to guide removal could inherently restrict their applicability~\cite{heng2023selective, lu2024mace, park2024direct,tsai2024ring,phamcircumventing,zhang2023generate}, especially in fine-grained instance-level erasing. As shown in Fig.~\ref{fig:fig1}, instructing an unconditional DM that lacks the notion of prompts to forget an unpromptable specific instance remains fundamentally challenging. In conditional DMs, removing an entire prompt because some of its outputs are undesirable may unnecessarily eliminate valid generations. In practice, a large number of targets we wish to remove are often unpromptable—that is, they cannot be distinguished using any prompt—making prompt-based unlearning ineffective for such fine-grained cases. Although prompt-free approaches exist, they often struggle to maintain model integrity, resulting in unintended degradation or excessive distortion of the model~\cite{golatkar2020eternal, wu2024erasediff, silas2024data}.

Beyond technical considerations, prompt-free unlearning is also crucial from both business and legal perspectives. In practice, many image generative model services screen user prompts and block problematic requests before the model produces an output, reducing the need for prompt-based unlearning. Therefore, prompt-free instance unlearning becomes important to the service providers in order to remove sensitive contents directly at the instance level without reliance on prompts. Moreover, legal frameworks such as the General Data Protection Regulation (GDPR)~\cite{EU2016GDPR} enshrine the “right to be forgotten,” requiring the removal of personal information such as identifiable human faces. This places strong obligations on service providers to support unlearning of unpromptable data points, making fine-grained, prompt-free unlearning not only a technical challenge but also an operational and regulatory imperative.

In this work, we propose a simple yet effective approach for \textbf{prompt-free instance unlearning}. We construct surrogate examples by applying off-the-shelf image editing methods to alter the identity of the forgetting target or remove undesired attributes while preserving overall structure. To balance forgetting and remembering~\cite{zhang2024forgetting, wang2023machine}, we leverage gradient surgery~\cite{yu2020gradient, huang2024learning} to resolve their conflicting gradients, combined with a timestep-aware loss weighting scheme tailored to the nature of diffusion models. This process is explicitly designed to avoid any reliance on prompts or concept-level supervision. Our method applies to both unconditional (trained with CelebA~\cite{karras2018progressive}) and conditional DMs (Stable Diffusion 3~\cite{esser2024scaling}), enabling precise unlearning of unpromptable targets while preserving model integrity.
Our contributions can be summarized as follows:
\begin{itemize}
    \item \textbf{Problem.} We investigate a novel unlearning target: undesired yet unpromptable outputs that arise in not only unconditional but also conditional diffusion models.
    \item \textbf{Challenge.} We demonstrate that such cases fall outside the scope of existing prompt-based unlearning methods.
    \item \textbf{Our Solution.} We introduce a prompt-free instance unlearning method that successfully forgets target outputs while preserving model integrity in both unconditional and conditional DMs.
\end{itemize}

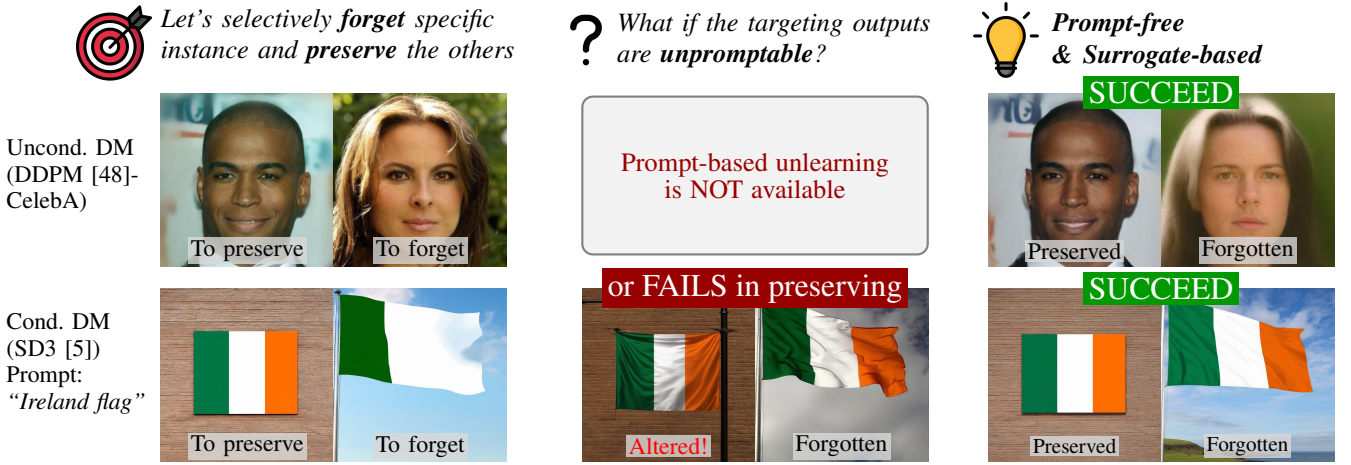
\begin{figure*}[htbp]
  \centering
\begin{tikzpicture}[
    img/.style={inner sep=0pt, anchor=north},
    label/.style={anchor=north},
    textblock/.style={font=\small, align=center, text width=4.2cm},
    yA/.store in=\yA, yB/.store in=\yB, yC/.store in=\yC, yD/.store in=\yD, yE/.store in=\yE
]

\def\yA{-0.2}  
\def\yB{-2.5}  
\def\yC{-2.7}  
\def\yD{-5.2}  
\def\yE{0.0}   

\begin{scope}[shift={(0,0)}]
  \node[textblock, align=left] at (2.1,\yA-1.1) {\color{black}Uncond. DM \\(DDPM~\cite{ho2020denoising}-\\CelebA)};
  \node[textblock, align=left] at (2.1,\yC-1.1) {\color{black}Cond. DM \\ (SD3~\cite{esser2024scaling})\\Prompt:\\\textit{``Ireland flag''}};
\end{scope}

\begin{scope}[shift={(2.1,0)}]
  \node[img] (p1a) at (1.1,\yA) {\includegraphics[width=2.3cm]{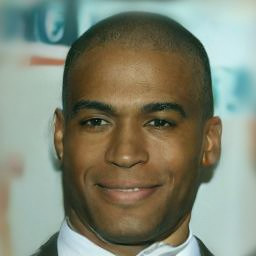}};
  \node[img] (p1b) at (3.4,\yA) {\includegraphics[width=2.3cm]{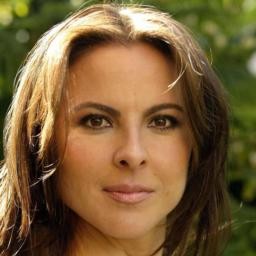}};
  \node[fill=gray!30, fill opacity=0.8, text opacity=1, inner sep=1pt] at ([xshift=0cm,yshift=-2.1cm]p1a.north) {\small To preserve};
  \node[fill=gray!30, fill opacity=0.8, text opacity=1, inner sep=1pt] at ([xshift=0cm,yshift=-2.1cm]p1b.north) {\small To forget};

  \node[img] (p2a) at (1.1,\yC-0.1) {\includegraphics[width=2.3cm]{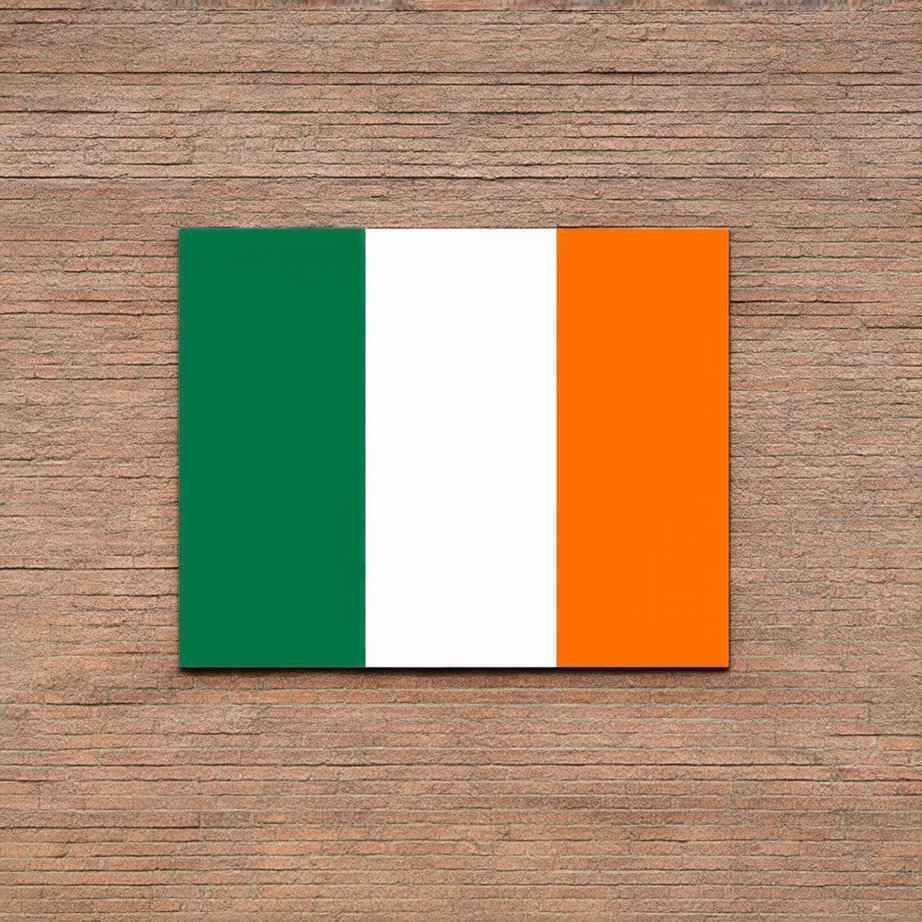}};
  \node[img] (p2b) at (3.4,\yC-0.1) {\includegraphics[width=2.3cm]{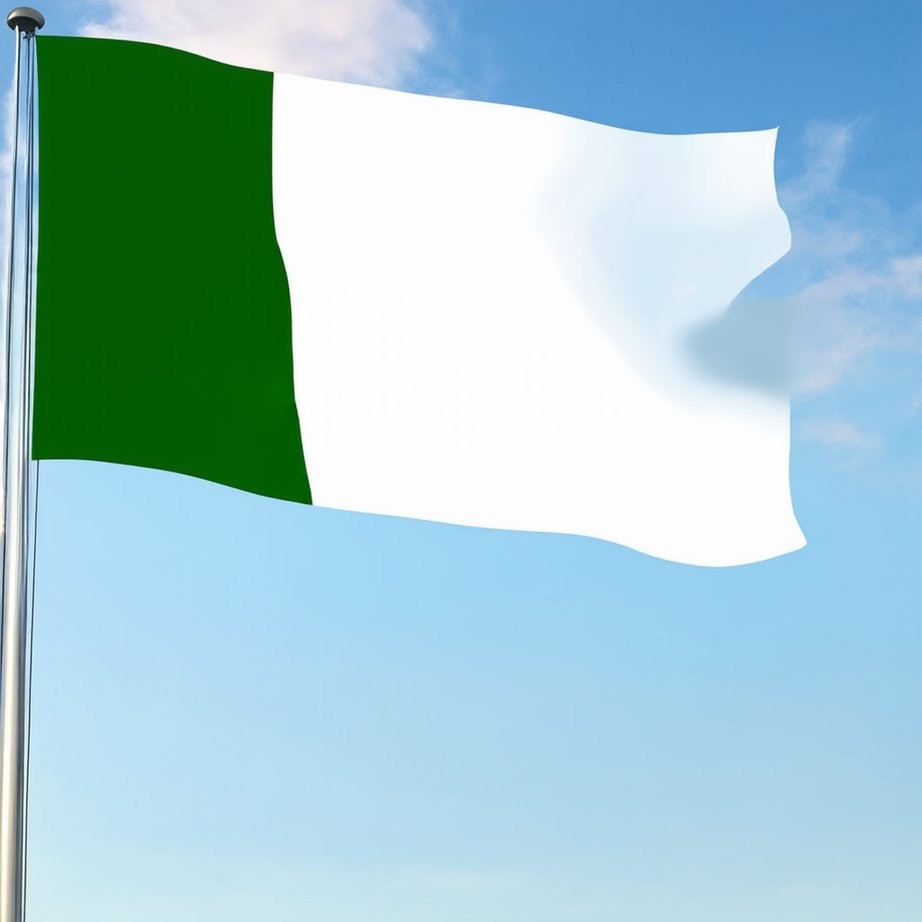}};
  \node[fill=gray!30, fill opacity=0.8, text opacity=1, inner sep=1pt] at ([xshift=0cm,yshift=-2.1cm]p2a.north) {\small To preserve};
  \node[fill=gray!30, fill opacity=0.8, text opacity=1, inner sep=1pt] at ([xshift=0cm,yshift=-2.1cm]p2b.north) {\small To forget};
\end{scope}

\begin{scope}[shift={(7.7,0.2)}]
  \node[textblock, draw=gray, thick, rounded corners, inner sep=5pt, text width=4.24cm, minimum height=2.1cm, fill=gray!10] at (2.25, \yA-1.3) 
    {\normalsize\color{red!60!black} Prompt-based unlearning\\is NOT available};
\end{scope}
\begin{scope}[shift={(7.7,0)}]

  \node[img] (c2a) at (1.1,\yC-0.1) {\includegraphics[width=2.3cm]{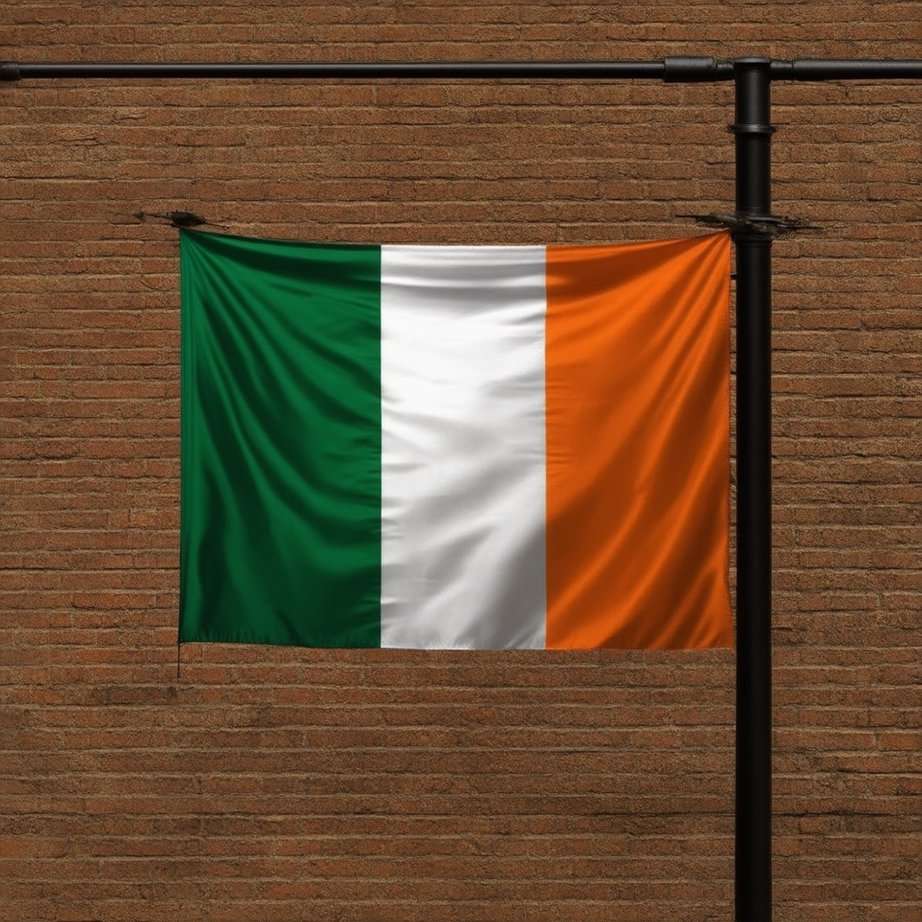}};
  \node[img] (c2b) at (3.4,\yC-0.1) {\includegraphics[width=2.3cm]{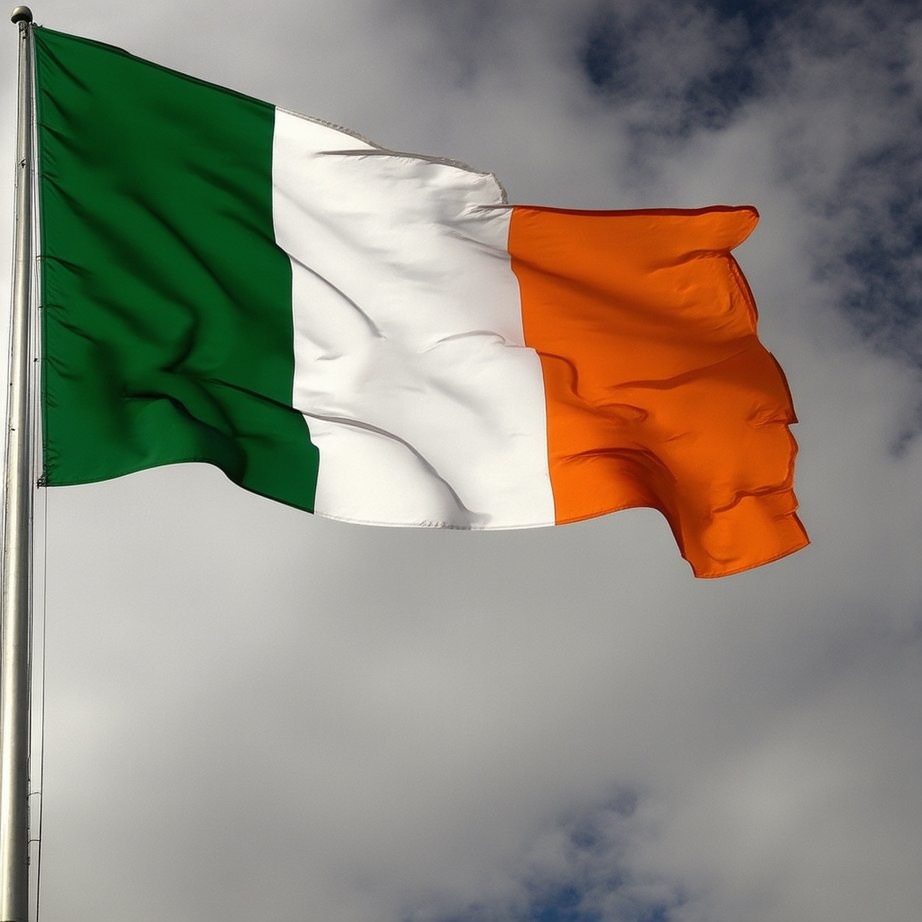}};
  \node[fill=gray!30, fill opacity=0.8, text opacity=1, inner sep=1pt] at ([xshift=0cm,yshift=-2.1cm]c2a.north) {\small\color{red} Altered!};
  \node[fill=gray!30, fill opacity=0.8, text opacity=1, inner sep=1pt] at ([xshift=0cm,yshift=-2.1cm]c2b.north) {\small Forgotten};
  \node[textblock] at (2.25,\yB-0.3) {\large\color{white}\setlength{\fboxsep}{2pt}\colorbox{red!60!black}{or FAILS in preserving}};
  
\end{scope}

\begin{scope}[shift={(13.1,0)}]

  \node[img] (r1a) at (1.1,\yA) {\includegraphics[width=2.3cm]{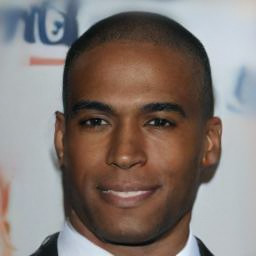}};
  \node[img] (r1b) at (3.4,\yA) {\includegraphics[width=2.3cm]{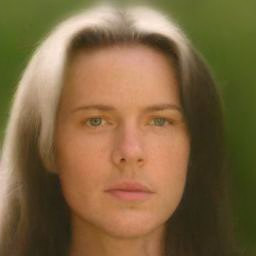}};
  \node[fill=gray!30, fill opacity=0.8, text opacity=1, inner sep=1pt] at ([xshift=0cm,yshift=-2.1cm]r1a.north) {\small Preserved};
  \node[fill=gray!30, fill opacity=0.8, text opacity=1, inner sep=1pt] at ([xshift=0cm,yshift=-2.1cm]r1b.north) {\small Forgotten};

  \node[img] (r2a) at (1.1,\yC-0.1) {\includegraphics[width=2.3cm]{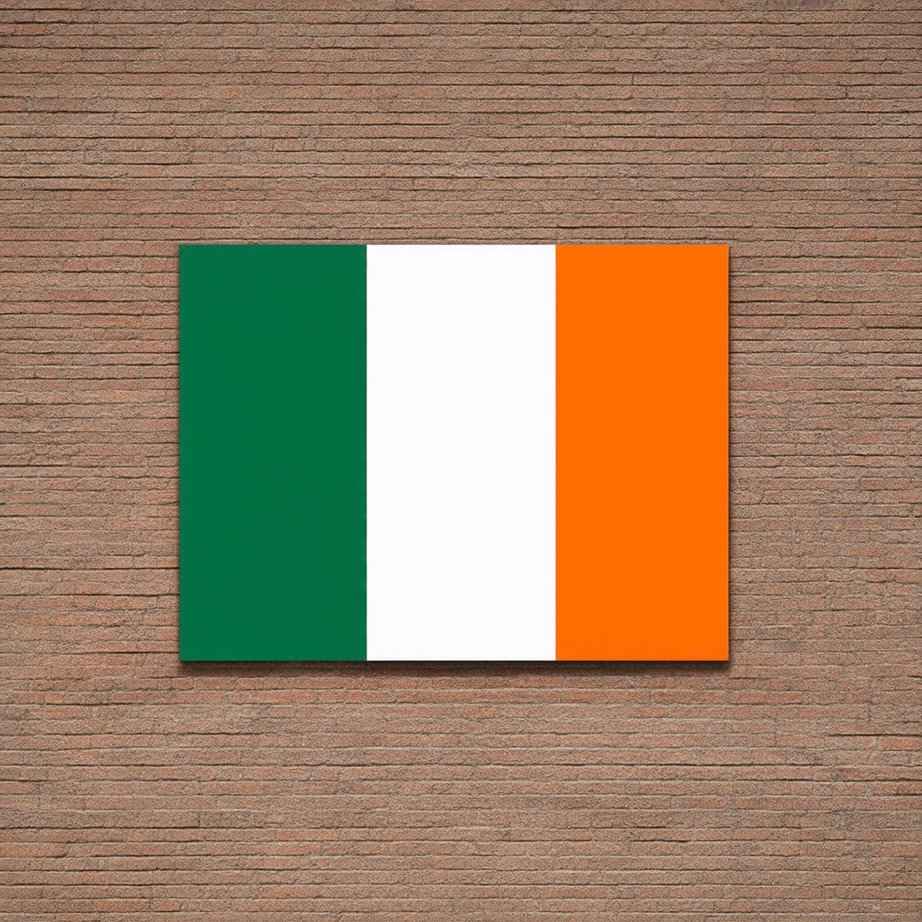}};
  \node[img] (r2b) at (3.4,\yC-0.1) {\includegraphics[width=2.3cm]{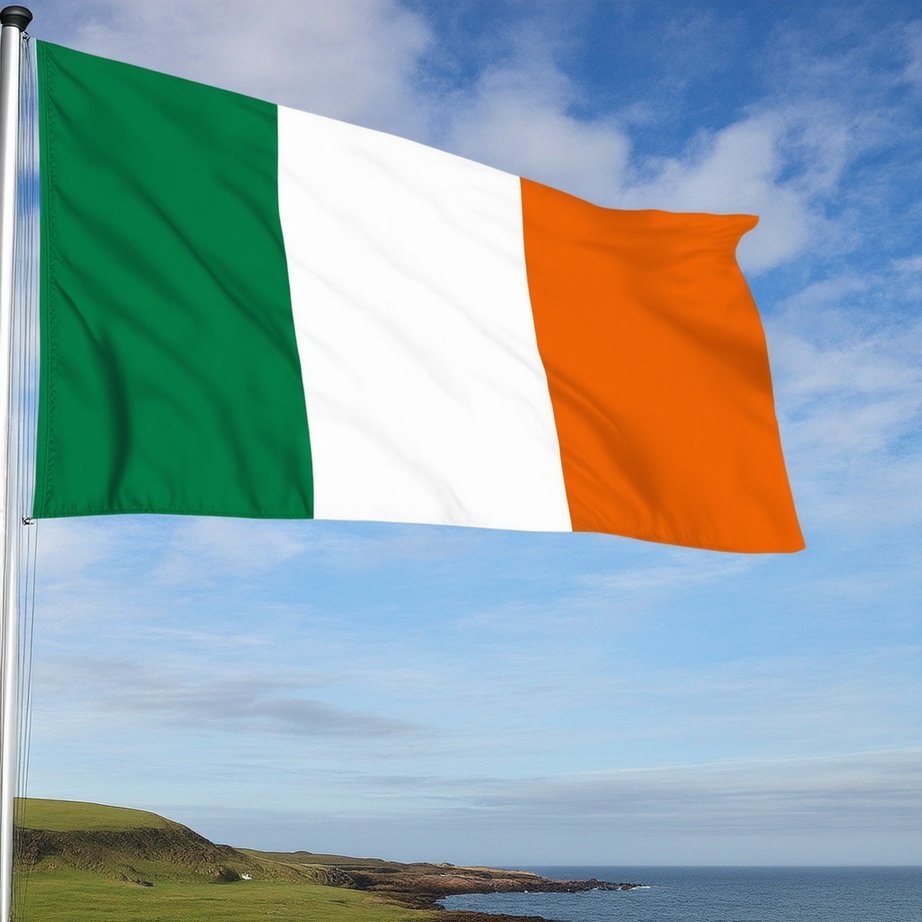}};
  \node[fill=gray!30, fill opacity=0.8, text opacity=1, inner sep=1pt] at ([xshift=0cm,yshift=-2.1cm]r2a.north) {\footnotesize Preserved};
  \node[fill=gray!30, fill opacity=0.8, text opacity=1, inner sep=1pt] at ([xshift=0cm,yshift=-2.1cm]r2b.north) {\footnotesize Forgotten};

    \node[textblock] at (2.25,\yE-0.2) {\large\color{white}\setlength{\fboxsep}{2pt}\colorbox{green!60!black}{SUCCEED}};

  \node[textblock] at (2.25,\yB-0.3) {\large\color{white}\setlength{\fboxsep}{2pt}\colorbox{green!60!black}{SUCCEED}};
\end{scope}

\node at (3.9,0.5) {
  \begin{tikzpicture}[baseline]
    \node[anchor=west, inner sep=0pt] at (0,-0.1) {\includegraphics[width=1.0cm]{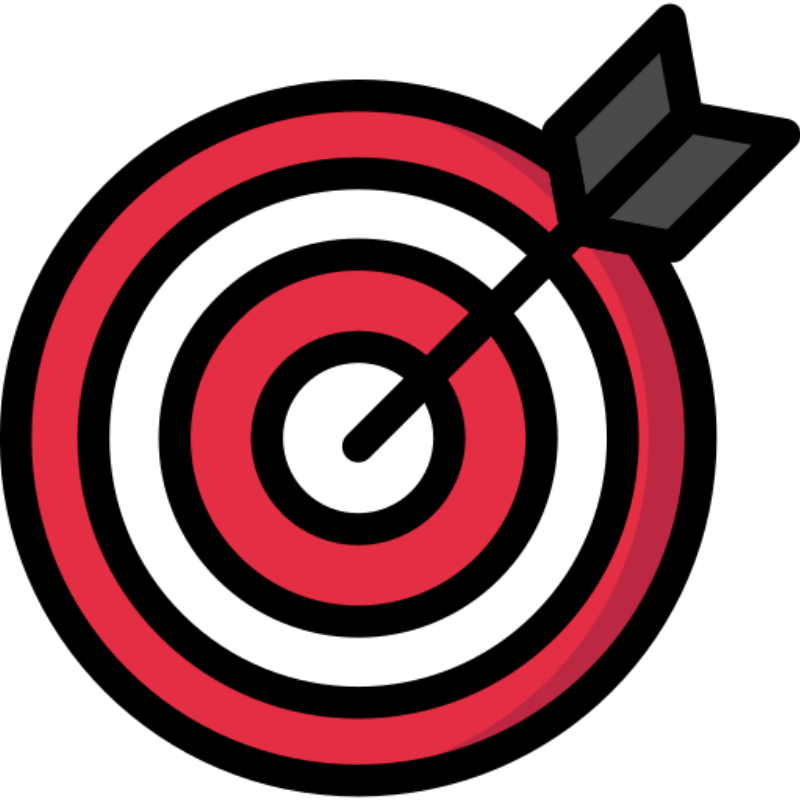}};
    \node[anchor=west, align=left] at (1.0,0) {\itshape\shortstack[l]{Let’s selectively \textbf{forget} specific\\ instance and \textbf{preserve} the others}};
  \end{tikzpicture}
};

\node at (9.8,0.5) {
  \begin{tikzpicture}[baseline]
    \node[anchor=west, inner sep=0pt] at (0,0) {\includegraphics[width=1.0cm]{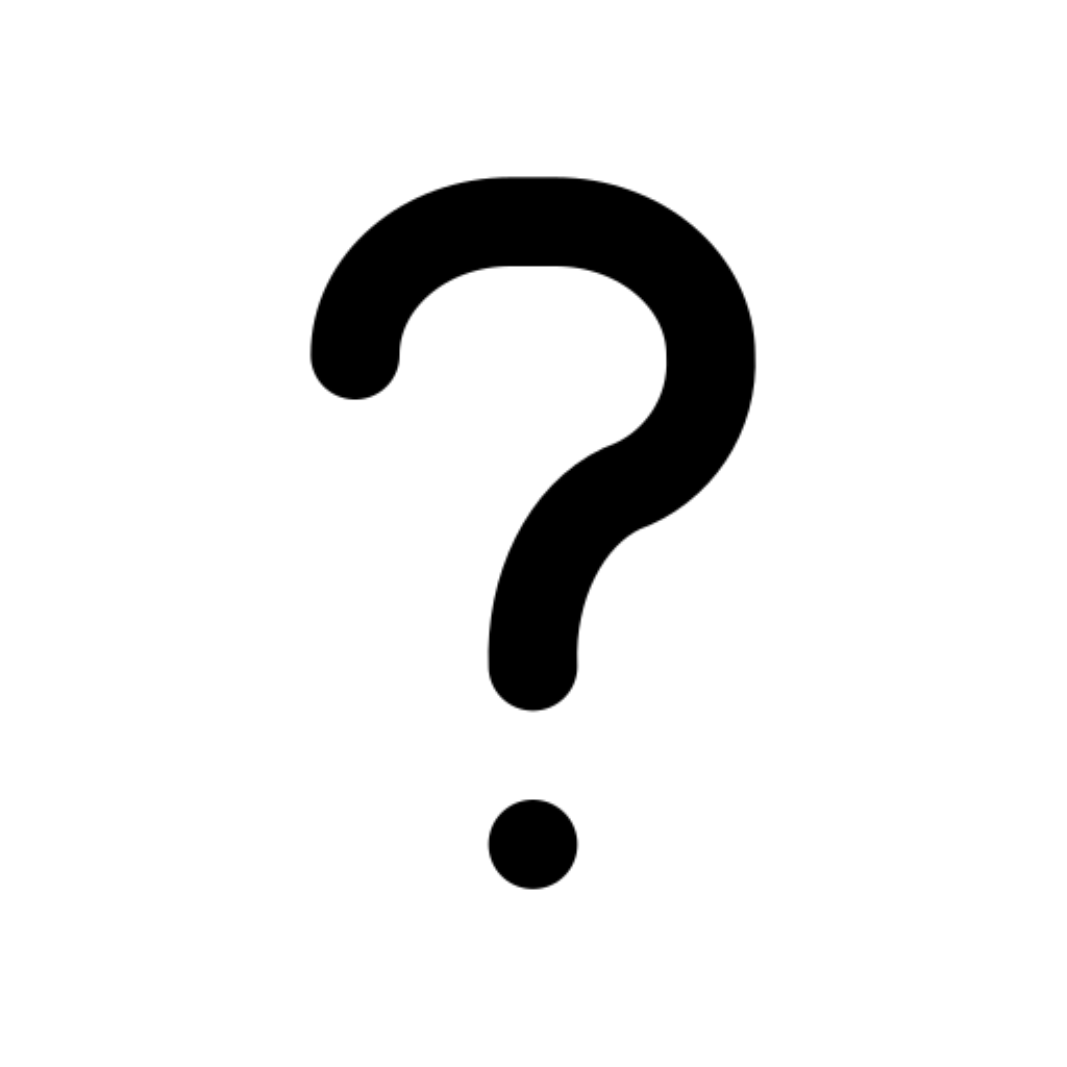}};
    \node[anchor=west, align=left] at (0.8,0) {\itshape\shortstack[l]{What if the targeting outputs \\ are \textbf{unpromptable}?}};
  \end{tikzpicture}
};

\node at (14.8,0.5) {
  \begin{tikzpicture}[baseline]
    \node[anchor=west, inner sep=0pt] at (0,0) {\includegraphics[width=1.0cm]{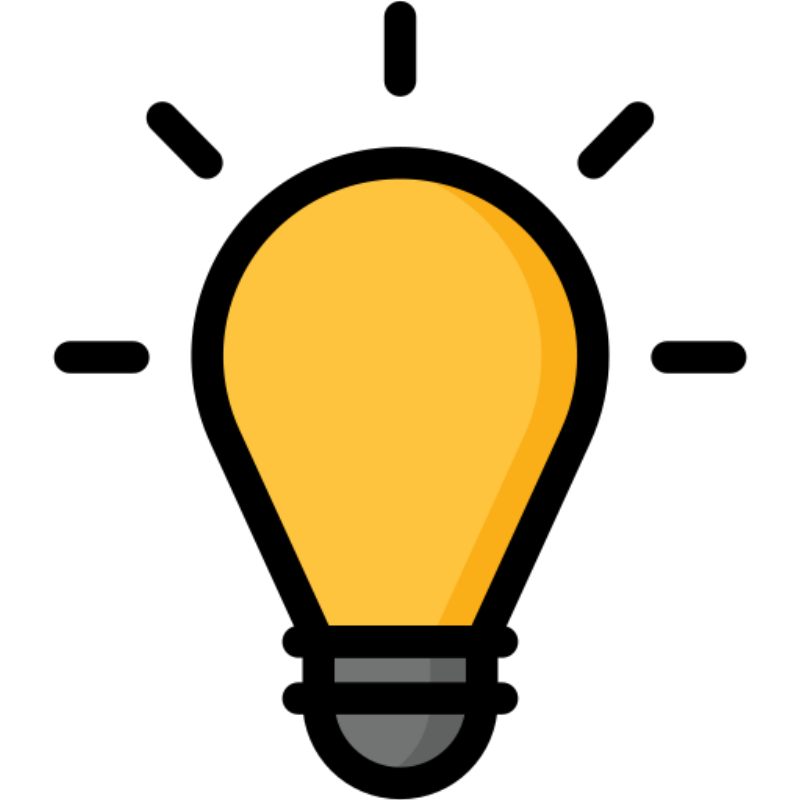}};
    \node[anchor=west, align=left] at (1.0,0) {\itshape\shortstack[l]{\textbf{Prompt-free} \\ \textbf{\& Surrogate-based}}};
  \end{tikzpicture}
};

\end{tikzpicture}
    \caption{Challenge and our solution for instance unlearning in diffusion models.}
    \label{fig:fig1}
\end{figure*}

\section{Problem Formulation}
\label{sec:background}
\subsection{Diffusion Models}
\label{subsec:back_gen_model}

Diffusion models (DMs) generate samples by adding noise in a forward process and learning to reverse it~\cite{sohl2015deep,ho2020denoising,song2020denoising}. Let $D\subset\mathbb{R}^d$ be the dataset that the diffusion model is trained on. For a real sample \(x_0\in \mathbb{R}^{d}\), the forward process in DDIM~\cite{song2020denoising} maps $x_0$ to $x_t\in\mathbb{R}^d$ at diffusion timestep $t\in\{1,\cdots,T\}$:
\begin{equation}\label{eq:diffusion_forward}
    x_t = \sqrt{\bar{\alpha}_t} x_0 + \sqrt{1 - \bar{\alpha}_t}\epsilon,
\end{equation}
where \(\epsilon \sim \mathcal{N}(0,I)\) and \(\bar{\alpha}_t=\prod_{s=1}^{t}{\alpha_s}\) is the cumulative product of the noise scheduler \(\{\alpha_s\}_{s=1}^{T}\). 
The denoising network \(\epsilon_{\theta}:\mathbb{R}^{d}\times\{1,\cdots,T\}\rightarrow\mathbb{R}^d\) with parameters \(\theta\) is trained to predict the added noise \(\epsilon\). The training objective is typically given by:
\begin{equation}\label{eq:diffusion_loss}
    \mathcal{L}_{\mathrm{DM}} = 
    \mathbb{E}_{x_0, \epsilon, t} \!\left[ \|\epsilon - \epsilon_{\theta}(x_t, t)\|^2 \right].
\end{equation}

After training, new samples are generated by starting with noise at \(x_T\) and iteratively denoising. We build upon this DDIM framework to incorporate our unlearning objectives, ensuring the model forgets specific instances while maintaining model integrity.

\subsection{Prompt-based Unlearning for Diffusion Models}
\label{subsec:prompt-based_unlearning}
Prompt-based unlearning for DMs focuses on preventing the generation of specific styles such that a certain prompt no longer influences the model output~\cite{gandikota2023erasing}. 
Let $p_\theta(x|c)$ denote the conditional data distribution given a conditional input (prompt) $c$. The forget dataset $\mathcal{D}_f$ consists of samples to forget and the remember dataset $\mathcal{D}_r$ consists of samples to remember.
The objective of prompt-based unlearning is to find a modified parameters $\theta'$ such that:
\begin{equation}\label{eq:concept_erasing}
    p_{\theta'} (\mathcal{D}_f | c) = 0, ~ p_{\theta'}(D|c) \simeq p_{\theta}(D|c) \quad \forall c, \forall D \subset \mathcal{D}_r
\end{equation}
where $\theta$ is the original parameter of the model and $p_\theta$ is the learned distribution.
Typically, $\mathcal{D}_f$ comprises images sharing a common characteristic (\textit{e.g.}, artistic style, object)~\cite{gandikota2023erasing}, which are to be forgotten, whereas $\mathcal{D}_r$ refers to the data which we want to remember. 
Consequently, we can easily define a subset of conditional inputs $\mathcal{C}_f$ such that $p_\theta(\mathcal{D}_f|c) \gg p_\theta(\mathcal{D}_r | c)$ for $\forall c \in \mathcal{C}_f$ and $p_\theta(\mathcal{D}_f|c) \ll p_\theta(\mathcal{D}_r | c)$ for $\forall c \not\in \mathcal{C}_f$. In that case, Eq. \ref{eq:concept_erasing} can be relaxed into:
\begin{align}
    & p_{\theta'} (\mathcal{D}_f | c)=0     && \forall c \in \mathcal{C}_f\label{eq:concept_erasing_relaxed1} \\ 
    & p_{\theta'} (D | c) \simeq  p_{\theta} (D | c)    && \forall c \in \mathcal{C}_r, \forall D \subset \mathcal{D}_r. \label{eq:concept_erasing_relaxed2}
\end{align}
Because concept erasing relies on the presence of a prompt $c$ to define the unwanted generation, it is specifically applicable to conditional DMs.

\begin{figure*}[ht]
\begin{center}
\begin{tikzpicture}[    
    markcheck/.style={font=\bfseries\LARGE\color{green!50!black}},
    markcross/.style={font=\bfseries\LARGE\color{red}},
    textblock/.style={font=\small, align=center, text width=6cm}]
    \node[anchor=south west,inner sep=0] (image) at (-0.5,0) {\includegraphics[width=0.95\linewidth]{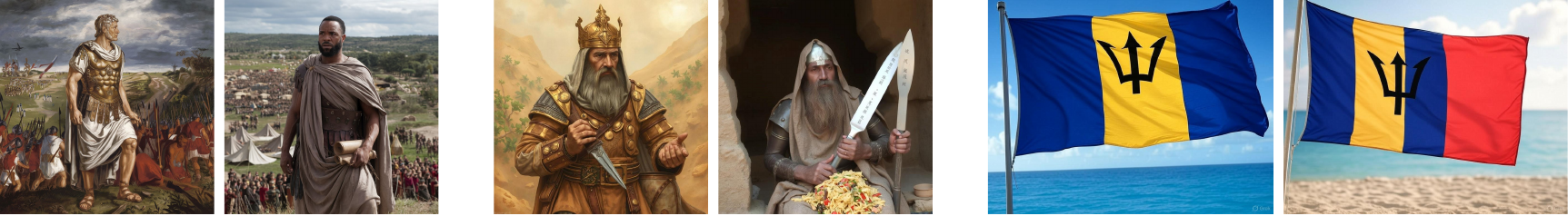}};
    \begin{scope}[x={(image.south east)},y={(image.north west)}]
        \node[markcheck] at (0.078, 1.1) {\contour{black}{\small \textcolor{white}{Desired} \Large\checkmark}};
        \node[markcross] at (0.22, 1.09) {\contour{black}{\small \textcolor{white}{Undesired} \Large\ding{55}}};
        \node[textblock] at (0.108, -0.2) {\emph{``Scipio Africanus in Punic War'',} \\ \url{ideogram.ai}};
        
        \node[markcheck] at (0.4, 1.1) {\contour{black}{\small \textcolor{white}{Desired} \Large\checkmark}};
        \node[markcross] at (0.545, 1.09) {\contour{black}{\small \textcolor{white}{Undesired} \Large\ding{55}}};
        \node[textblock] at (0.435, -0.2) {\emph{``Saladin'',} \\ \url{gencraft.com}};

        \node[markcheck] at (0.751, 1.1) {\contour{black}{\small \textcolor{white}{Desired} \Large\checkmark}};
        \node[markcross] at (0.936, 1.09) {\contour{black}{\small \textcolor{white}{Undesired} \Large\ding{55}}};
        \node[textblock] at (0.804, -0.2) {\emph{``Barbados flag'',} \\ \url{grok.com}};


    \end{scope}
\end{tikzpicture}
\caption{Cultural and semantic misrepresentation highlight the need for instance unlearning in commercial generative models.}
\label{fig:commercial_example}
\end{center}
\end{figure*}
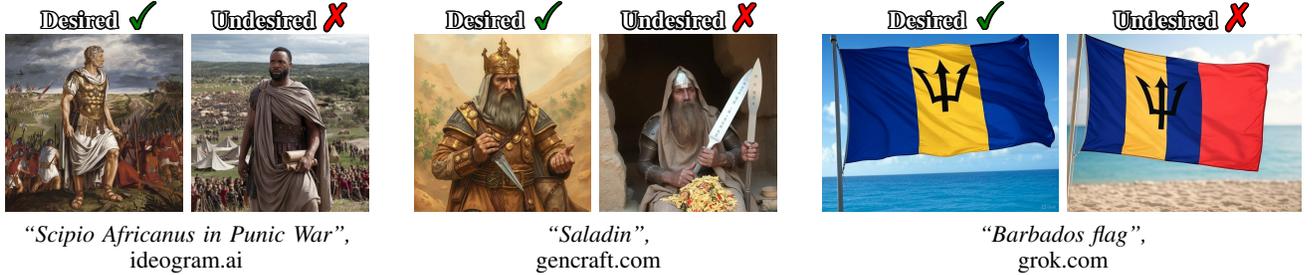

\subsection{Prompt-free Instance Unlearning for Diffusion Models}
\label{subsec:back_instance_unlearning}
\noindent\textbf{Prompt-free instance unlearning for unconditional DMs. }
In this work, we propose a \emph{prompt-free instance unlearning} method. The key difference to prompt-based unlearning is that we cannot rely on $c$ in the case of unconditional DMs. The goal is defined as follows:

\begin{align}
    & \text{[Forgetting]} \quad && p_{\theta'}(\mathcal{D}_f) = 0, \label{eq:forgetting} \\
    & \text{[Model integrity]} \quad && p_{\theta'}(D) \simeq p_{\theta}(D) \quad \forall D \subset \mathcal{D}_r. \label{eq:model_integrity}
\end{align}

\noindent\textbf{Prompt-free instance in conditional DMs.}
\label{subsec:mitigating_bias_discussion}
While prompt-free instance unlearning has been primarily studied in the context of unconditional DMs~\cite{golatkar2020eternal, wu2024erasediff, silas2024data}, its application in conditional DMs has remained relatively limited. However, we show that such needs arise in some conditional settings, including Stable Diffusion 3 (SD3)~\cite{esser2024scaling}. As shown in Fig.~\ref{fig:commercial_example}, prompts involving historical figures or national flags can lead to misrepresentations, such as incorrect race or cultural symbols. For example, models have mistakenly depicted the Roman general \textit{African}us as an \textit{African} with dark skin, the Islamic leader \textit{Salad}in as stacked \textit{salad}s, or incorrectly depicted the Barbados flag. These issues cannot be resolved through prompt engineering alone and instead call for instance-level unlearning, emphasizing the importance of prompt-free methods.

\section{Related Work and Challenges}
\subsection{Related Work}
\noindent\textbf{Prompt-based machine unlearning. } 
AblCon \cite{kumari2023ablating} and ESD \cite{gandikota2023erasing} retrain the model to erase target concepts by replacing them with designated mapping concepts. 
To maintain model integrity, prior arts fine-tune gradient-based salient weight in a model \cite{fan2024salun}, introduce a regularization inspired by continual learning \cite{heng2023selective,kirkpatrick2017overcoming,rolnick2019experience}, or adjust linear projections within attention layers \cite{gandikota2024unified, lu2024mace}. SPM \cite{lyu2024one} and CPE \cite{lee2025concept} adopt low-rank adaptation \cite{hu2022lora} to selectively erase target concepts.
As a preference optimization problem, DUO \cite{park2024direct} 
directly modifies local image regions containing target concepts using synthesized images, thereby better preserving the model integrity.
Additionally, GIE \cite{chengrowth} and Concept Corrector \cite{meng2025concept} correct image regions containing target concepts by erasing them in image feature space.

\noindent\textbf{Prompt-free machine unlearning. }
Research in prompt-free machine unlearning includes methods for non-diffusion~\cite{malnick2024taming, seo2024generative} and DMs~\cite{golatkar2020eternal, wu2024erasediff, silas2024data}. These approaches typically define forgetting (Eq.~\ref{eq:forgetting}) and remembering (Eq.~\ref{eq:model_integrity}) objectives and address them with remember loss $\mathcal{L}_r(\mathcal{D}_{r})$ and forget loss $\mathcal{L}_f(\mathcal{D}_{f})$, respectively.\

NegGrad~\cite{golatkar2020eternal} employs a loss function of the form $\mathcal{L}_r - \lambda\mathcal{L}_f$, aiming to minimize $\mathcal{L}_r$ while maximizing $\mathcal{L}_f$. However, this often leads to a significant conflict between $\mathcal{L}_r$ and $-\lambda\mathcal{L}_f$, making it difficult to find a balanced optimum. SISS~\cite{silas2024data} combines two objectives in a single pass via importance sampling from a mixture distribution. However, they still face conflicts similar to ~\cite{golatkar2020eternal} due to their reliance on the same core objective. EraseDiff~\cite{wu2024erasediff} replaces scalarizing two objectives with an optimization-based meta-learning approach, but it still exhibits limitations in maintaining model integrity. These methods commonly struggle to preserve model integrity, highlighting the need for a better unlearning method.

\subsection{Challenges: The Fundamental Mismatch in Applying Prompt-based to Prompt-free Unlearning} 
Several prior works~\cite{kumari2023ablating, fan2024salun, heng2023selective, park2024direct} have explored machine unlearning with a similar objective---to remove specific concepts or images---but they take advantage of prompt distinguishability (\textit{e.g.}, prompts related to \textit{``nudity''} and \textit{``non-nudity''} are used during the unlearning process). Since we are trying to solve Eqs. \ref{eq:forgetting} and \ref{eq:model_integrity}, directly applying such works to this problem is not straightforward, and Fig. \ref{fig:comparison_unlearnings} illustrates this challenge. Prompt-based unlearning benefits from conditional inputs $c$, as the forget dataset $\mathcal{D}_f$ is usually highly correlated to a certain set $\mathcal{C}_f$. Thus, by performing unlearning for a specific $c_f$ as illustrated in (a), one can successfully achieve unlearning with model integrity. However, in (b) instance unlearning, it is ideal not to harm $p(x)$ for $x \not\in \mathcal{D}_f$ (see (b)-above). Therefore, the methods in (a) do not directly apply to the case in (b).

To the best of our knowledge, only a few methods~\cite{golatkar2020eternal, wu2024erasediff, silas2024data} aim at truly prompt-free instance unlearning. Image-based unlearning works~\cite{park2024direct, chengrowth, meng2025concept} were exploited prompts. For example, DUO~\cite{park2024direct} guides unlearning via different prompts to $\mathcal{D}_f$ and $\mathcal{D}_r$, such as \textit{``naked woman''} and \textit{``dressed woman''}, respectively.
When applying DUO to prompt-free unlearning, it cannot provide guidance using different prompts, failing to preserve model integrity as shown in Fig. \ref{fig:fig1}.

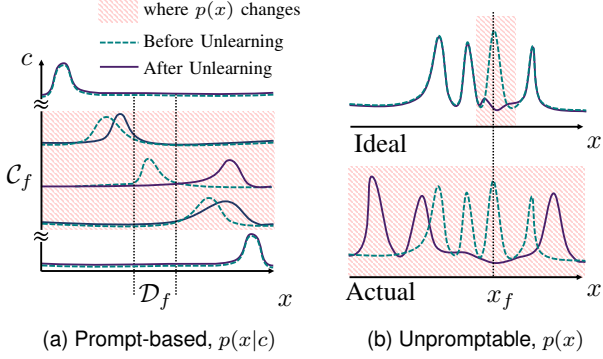
\begin{figure}[t]
    \centering
    \subfloat[Prompt-based, $p(x|c)$]{%
        \begin{tikzpicture}
            \node[anchor=south west,inner sep=0] (image) at (0,0)
              {\includegraphics[width=0.45\linewidth]{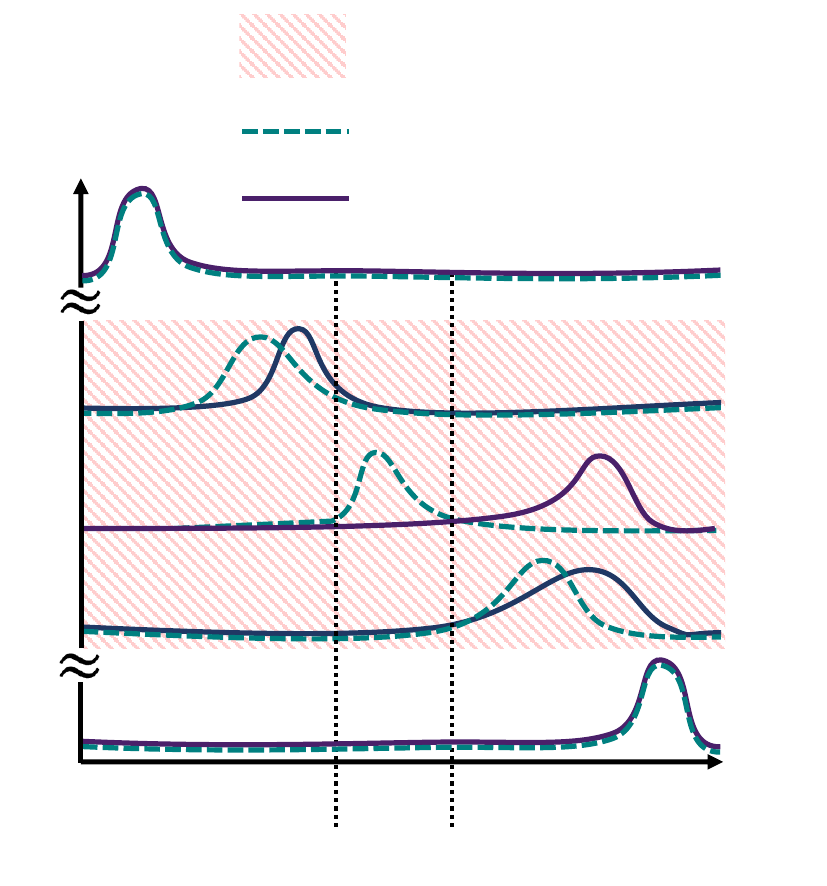}};
            \begin{scope}[x={(image.south east)},y={(image.north west)}]
                \node[align=left] at (0.7, 0.86) {\scriptsize where $p(x)$ changes \\ \scriptsize Before Unlearning \\[-2pt] \scriptsize After Unlearning};
                \node[] at (0.91, 0.05) {$x$};
                \node[] at (0.05, 0.8) {$c$};
                \node[] at (0.03, 0.43) {$\mathcal{C}_{f}$};
                \node[] at (0.48, 0.06) {$\mathcal{D}_f$};
            \end{scope}
        \end{tikzpicture}
    }
    \subfloat[Unpromptable, $p(x)$]{%
        \begin{tikzpicture}
            \node[anchor=south west,inner sep=0] (image) at (0,0)
              {\includegraphics[width=0.45\linewidth]{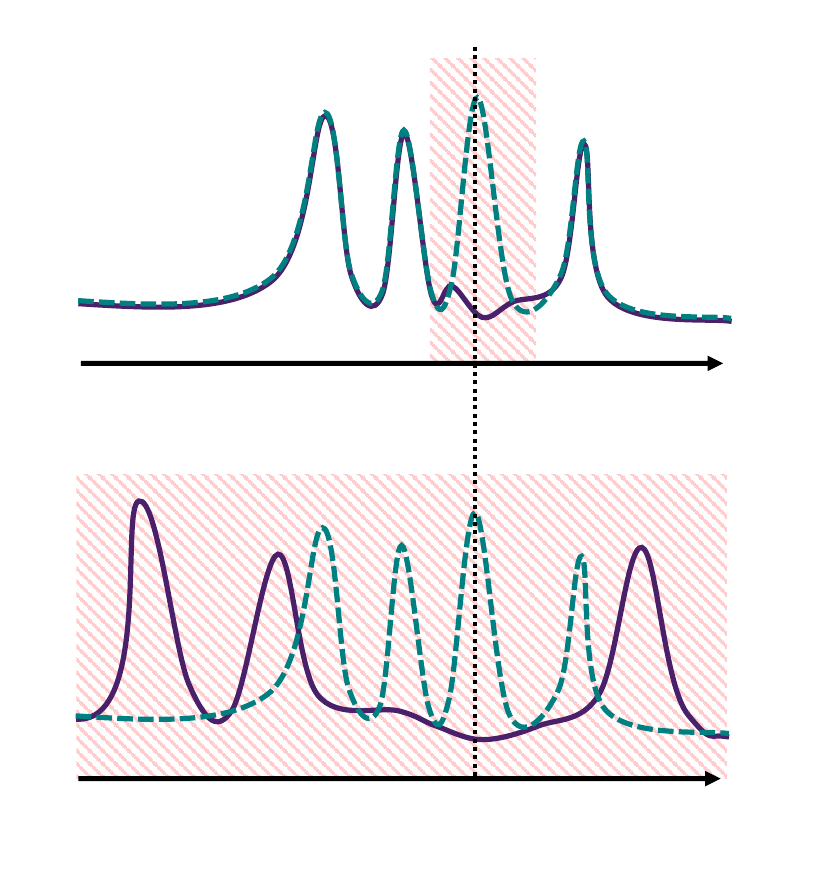}};
            \begin{scope}[x={(image.south east)},y={(image.north west)}]
                \node[] at (0.2, 0.54) {Ideal}; 
                \node[] at (0.2, 0.08) {Actual}; 
                \node[] at (0.91, 0.08) {$x$};
                \node[] at (0.91, 0.54) {$x$};
                \node[] at (0.6, 0.05) {$x_f$};
            \end{scope}
        \end{tikzpicture}
    }
    \caption{Comparison between (a) prompt-based and (b) instance unlearning.}
    \label{fig:comparison_unlearnings}
\end{figure}

\section{Methods}
\label{sec:methods}
\subsection{Unlearning Methods}
\label{subsec:loss_design}
\noindent\textbf{Surrogate-based forgetting objective.}
Remembering objective Eq. \ref{eq:model_integrity} ensures the model integrity on $\mathcal{D}_{r}$. Denoting the forward process of an element $x_0^r\in\mathcal{D}_r$ as $x_t^r=\sqrt{\bar{\alpha}_t}x_0^r+\sqrt{1-\bar{\alpha}_t}\epsilon$, a common approach is to adopt the training loss Eq. \ref{eq:diffusion_loss} on $\mathcal{D}_{r}$:
\begin{equation}\label{eq:remember_loss}
    \mathcal{L}_{r} = 
    \mathbb{E}_{x_0^{r}, \epsilon, t} \left[ \left\| \epsilon -   \epsilon_\theta(x_t^{r}, t)\right\|_{2}^2 \right],
\end{equation} 
where $ x_0^{r} \sim \mathcal{D}_{r} $, $ \epsilon \sim \mathcal{N}(0, I)$, and $ t \sim \text{Uniform}(0, T)$.
In comparison, the forgetting objective Eq. \ref{eq:forgetting} is designed to disturb the generation of data in $\mathcal{D}_{f}$.
Our method builds on the idea that, since DMs learn to map between two distributions, small perturbations to the noise can lead the model to forget specific data. Revisiting Eq. \ref{eq:diffusion_forward} for the case of remembering data $x_0^r$, the noise term $\epsilon$ can be expressed as:
\begin{equation}
    \epsilon = \dfrac{x_{t}^{r}-\sqrt{\bar{\alpha}_{t}}x_{0}^{r}}{\sqrt{1-\bar{\alpha}_{t}}}.
\end{equation}
Recalling that DM training aims to map $x_{t}^{r}$ to $\epsilon$, we hypothesize that perturbing this mapping for $x_{t}^{f}$ can lead to successful unlearning of the instance, which corresponds to the forgetting target $x_0^f$. To accomplish this, we propose adjusting $\epsilon$ by perturbing $x_0^f$ to a surrogate. Specifically, we construct surrogate data that closely resembles the target data while exhibiting distinct characteristics. Using the surrogate data $x_{0}^{s}$, we define a modified $\epsilon'$:
\begin{equation}
    \epsilon'(x_t^f, x_0^s) = \dfrac{x_{t}^{f}-\sqrt{\bar{\alpha}_{t}}x_{0}^{s}}{\sqrt{1-\bar{\alpha}_{t}}}.
    \label{eq:modified_noise}
\end{equation}

Replacing $\epsilon'$ in place of $\epsilon$ in the remembering objective, we define the forget loss as follows, with $x_{0}^{s} \sim \mathcal{D}_{s}$:
\begin{equation}\label{eq:forget_loss}
    \mathcal{L}_{f} = 
    \mathbb{E}_{(x_0^{f},  x_0^s), t}  \| \epsilon' -   \epsilon_\theta(x_t^{f}, t)\|_{2}^2.
\end{equation}

\noindent\textbf{Timestep-aware weighting.}
Based on the analysis of the sampling process of DMs proposed in Diff-Tuning~\cite{zhong2024diffusion}, we propose to use an adaptive $\lambda$ that effectively balances the two objectives of remembering and forgetting. Since early timesteps affect fine details~\cite{esser2024scaling, ma2024deepcache, jeong2025upsample}, we emphasize $\mathcal{L}_r$, while later timesteps focus on generic shape and therefore $\mathcal{L}_f$ is emphasized. Specifically, we define $\lambda(t) = 1 - \beta t$, with $\beta > 0$ as a hyperparameter. This dynamic weighting strategy effectively balances the two objectives.

\noindent\textbf{Gradient surgery.}
It is prevalent to resolve conflict of the gradients from the different losses, not only in machine unlearning studies~\cite{patel2025learning, shamsian2025go}, but also more general multi-task optimization~\cite{yu2020gradient, guangyuanrecon, jin2025unlearning}. We also resolve destructive interference between two conflicting gradients by projecting one ($\nabla \mathcal{L}_r$) onto the other ($\nabla \mathcal{L}_f$) as follows:
\begin{align}
    g_r &= \lambda \nabla \mathcal{L}_r, \quad g_f = (1-\lambda) \nabla \mathcal{L}_f, \\
    \quad g_f' &= \begin{cases} 
        g_f - \mathrm{proj}_{g_r}^{g_f}, & \textnormal{if} \quad g_r \cdot g_f < 0, \\
        g_f, & \textnormal{otherwise.}
    \end{cases} \label{eq:gradient_surgery}
\end{align}
This strategy, inspired by~\cite{huang2024learning}, is especially beneficial for instance unlearning, where forgetting and retaining objectives directly compete. In our case, $g=g_r+g_f'$ is the final gradient. This surgery refines the unlearning process and further improves generation stability. The overall framework is illustrated in Algorithm~\ref{alg:unlearning}.

\begin{algorithm}[t]
\caption{Prompt-free Instance Unlearning}
\label{alg:unlearning}
\begin{algorithmic}[1]
\Require Denoising network $\epsilon_{\theta}$, forget dataset $\mathcal{D}_f$, remember dataset $\mathcal{D}_r$, total iterations $N$, learning rate $\eta$, max timestep $T$, timestep-aware weighting hyperparameter $\beta$
\For{$x_0^f \in \mathcal{D}_f$}
    \State Construct surrogate image $x_0^s$ from $x_0^f$ \Comment{~\ref{subsec:surrogate_selection}}
    \For{$l = 1, \dots, N$}
        \State Sample remember image $x_0^r \in \mathcal{D}_r$
        \State $t \sim \text{Uniform}(1, T)$
        \State $\epsilon \sim \mathcal{N}(0, I)$
        \State $x_t^r \gets \sqrt{\bar{\alpha}_t}x_0^r + \sqrt{1-\bar{\alpha}_t}\epsilon$
        \State $x_t^f \gets \sqrt{\bar{\alpha}_t}x_0^f + \sqrt{1-\bar{\alpha}_t}\epsilon$
        \State $\mathcal{L}_r \gets ||\epsilon - \epsilon_{\theta}(x_t^r, t)||_2^2$ 
        \Comment{Remember loss}
        \State $\epsilon' \gets \frac{x_t^f - \sqrt{\bar{\alpha}_t}x_0^s}{\sqrt{1-\bar{\alpha}_t}}$ \Comment{Eq.~\ref{eq:modified_noise}}
        \State $\mathcal{L}_f \gets ||\epsilon' - \epsilon_{\theta}(x_t^f, t)||_2^2$ 
        \Comment{Forget loss}
        \State $\lambda(t) \gets 1 - \beta t$ \Comment{Timestep-aware weighting}
        \State $g_r \gets \lambda(t) \nabla_\theta \mathcal{L}_r$
        \State $g_f \gets (1 - \lambda(t)) \nabla_\theta \mathcal{L}_f$
        \If{$g_r \cdot g_f < 0$} \Comment{Gradient surgery}
            \State $g_f' \gets g_f - \frac{g_r \cdot g_f}{||g_r||^2} g_r$
        \Else
            \State $g_f' \gets g_f$
        \EndIf
        \State $g \gets g_r + g_f'$ \Comment{Combine gradients for update}
        \State $\theta \gets \text{optimizer}(\theta, g, \eta)$
    \EndFor
\EndFor
\end{algorithmic}
\end{algorithm}

\subsection{Surrogate vs. Exact Unlearning: A Theoretical View}  
\label{sec:surrogate-vs-exact-simplified}
Here, we explain our method (using a surrogate) is better for preserving the original mapping even than exact unlearning. We consider a simple ridge-regression setting to illustrate how the learned parameter \(\theta\) changes when (1) \emph{removing} a single data point \((x_i, y_i)\) (exact unlearning) vs.\ (2) \emph{replacing} it with a surrogate \((x_i', y_i')\).  
Below, we state Theorems~\ref{thm:1}--\ref{thm:2} with proofs and interpretations. Note that \cref{thm:1} was proven long ago by \cite{golub1979generalized}, with no novelty on our part. We restate it in our notation solely for alignment with Theorem~\ref{thm:2} and Corollary~\ref{cor:3}.

\begin{theorem}[Exact Unlearning, restated from \cite{golub1979generalized}]\label{thm:1}
Let \(\theta^*\in \mathbb{R}^d\) be the parameter vector obtained by solving the ridge-regression problem:    
\begin{equation}
\theta^*
=
\arg\min_{\theta} 
\Bigl\{ 
\tfrac12 \|X\theta - y\|^2
+
\tfrac{\lambda}{2}\|\theta\|^2
\Bigr\},
\end{equation}
where \(X \in \mathbb{R}^{n\times d}\), \(y\in \mathbb{R}^n\), and \(\lambda \ge 0\).  Denote 
\begin{equation}
A 
= 
X^T X + \lambda I,
\end{equation}
so that
\begin{equation}
\theta^*
=
A^{-1}  X^T y.
\end{equation}
Now remove the \(i\)-th row \(\bigl(x_i, y_i\bigr)\) from \(X, y\), producing \(\widetilde{X}, \widetilde{y}\). The new solution, trained from scratch on \(\widetilde{X}, \widetilde{y}\), is
\begin{equation}
\widetilde{\theta}
=
\arg\min_\theta 
\Bigl\{
\tfrac12 \|\widetilde{X}\theta - \widetilde{y}\|^2
+
\tfrac{\lambda}{2} \|\theta\|^2
\Bigr\}.
\end{equation}
Then, the difference between the two solutions satisfies the following closed-form expression:
\begin{equation}
\widetilde{\theta} - \theta^*
=
\frac{\bigl(x_i^T \theta^* - y_i\bigr)}
{1 - x_i^T A^{-1} x_i}
A^{-1} x_i,
\end{equation}\label{eqn:1.5}
where \(x_i \in \mathbb{R}^d\) is the feature vector of the removed row, and \(A^{-1} x_i\) is the direction in parameter space that adjusts for removing \((x_i,y_i)\).
\end{theorem}

\begin{interpretation}
Exact unlearning strictly removes the data point's influence, but it may cause a larger shift in the solution because \(\widetilde{\theta}-\theta^*\) is fully determined by that row's removal.
\end{interpretation}

\begin{proof}
   Since \(\theta^*\) is the minimizer of the ridge objective, it satisfies
   \begin{equation}
   \bigl(X^T X + \lambda I\bigr)\theta^*
   =
   X^T y,
   \end{equation}
   or equivalently \(A \theta^* = X^T y\).
 Removing row \(i\) yields \(\widetilde{X}, \widetilde{y}\). One checks:
   \begin{equation}
   \widetilde{X}^T \widetilde{X}
   =
   X^T X - x_i x_i^T,
   \quad
   \widetilde{X}^T \widetilde{y}
   =
   X^T y - x_i y_i.
   \end{equation}
   The new solution \(\widetilde{\theta}\) satisfies
   \begin{equation}
   \bigl(\widetilde{X}^T \widetilde{X} + \lambda I\bigr)
   \widetilde{\theta}
   =
   \widetilde{X}^T \widetilde{y}
   \Longleftrightarrow
   \bigl(A - x_i x_i^T\bigr)\widetilde{\theta}
   =
   A\theta^* - x_i y_i.
   \end{equation}
Let us express \(\widetilde{\theta}\) via a rank-1 update. To this end,
   Let 
   \begin{equation}
   U = x_i x_i^T,
   \quad
   v = x_iy_i,
   \quad
   \alpha = x_i^T \theta^*.
   \end{equation}
   Then \(\widetilde{\theta}\) can be written as 
   \begin{equation}
   (A - U)^{-1} (A\theta^* - v).
   \end{equation}
   Subtract \(\theta^*\), insert \(\theta^*\) in a convenient form, and simplify:
   \begin{equation}
    \begin{aligned}
   \widetilde{\theta} - \theta^*
   &=
   (A - U)^{-1}(A\theta^* - v)
   -
   \theta^*\\
   &=
   (A - U)^{-1}
   \Bigl[
      (A\theta^* - v) 
      -
      (A - U)\theta^*
   \Bigr].
   \end{aligned}
   \end{equation}
   Observe that \( (A - U)\theta^* = A\theta^* - x_i x_i^T \theta^* = X^T y - x_i \alpha \).  Hence
   \begin{equation}
   \begin{aligned}
   (A\theta^* - v)& - (A\theta^*) + U\theta^*
   =
   -v + U\theta^*\\
   &=
   x_i\alpha - x_iy_i
   =
   x_i(\alpha - y_i).
   \end{aligned}
   \end{equation}
   Therefore,
   \begin{equation}
   \widetilde{\theta} - \theta^*
   =
   (A - x_i x_i^T)^{-1}
   \Bigl[x_i(\alpha - y_i)\Bigr].
   \end{equation}
   Applying the Sherman–Morrison (rank-1) inverse lemma tells us 
   \begin{equation}
   (A - x_i x_i^T)^{-1} x_i
   =
   \frac{1}{1 - x_i^T A^{-1} x_i}
   A^{-1} x_i.
   \end{equation}
   Finally,
   \begin{equation}
   \widetilde{\theta} - \theta^*
   =
   \frac{(\alpha - y_i)}{1 - x_i^T A^{-1} x_i}
   A^{-1} x_i
   =
   \frac{\bigl(x_i^T \theta^* - y_i\bigr)}
        {1 - x_i^T A^{-1} x_i}
   A^{-1}x_i.
   \end{equation}
   This completes the derivation.
\end{proof}

\begin{theorem}[Surrogate-based Unlearning]\label{thm:2}
Let $X, y, \lambda, \theta^*, A$ the same to the \cref{thm:1}.
Suppose we modify exactly one row \(i\) from \(\bigl(x_i,y_i\bigr)\) to a new (surrogate) pair \(\bigl(x_i',y_i'\bigr)\).  Define the new design/label set:
\begin{equation}
X^\dagger = X + e_i(x_i' - x_i)^T,
\quad
y^\dagger = y + e_i(y_i' - y_i),
\end{equation}
where \(e_i \in \mathbb{R}^n\) is the \(i\)-th standard basis vector.  The new ridge solution is
\begin{equation}
\theta^\dagger
=
\arg\min_{\theta}\Bigl\{\tfrac12\|X^\dagger\theta - y^\dagger\|^2 + \tfrac{\lambda}{2}\|\theta\|^2\Bigr\}.
\end{equation}
Then the following holds.  
\begin{equation}
\begin{aligned}
\theta^\dagger - \theta^* = &
(A + M)^{-1}\bigl[X^T y + x_i (y_i' -y_i) \\ 
& + (x_i - x_i') y_i' \bigr] -A^{-1}X^T y,
\end{aligned}
\end{equation}
where \begin{equation}M = x_i (x_i'-x_i)^T + (x_i'-x_i)x_i^T + (x_i'-x_i)(x_i'-x_i)^T\end{equation}  
\end{theorem}

\begin{interpretation}
By introducing a surrogate row \(\bigl(x_i', y_i'\bigr)\), we adjust the model less drastically if \(x_i'\) is sufficiently close to \(x_i\). This allows the unlearning effect (on \(y_i\)) while partially preserving the original parameter \(\theta^*\).
\end{interpretation}

\begin{proof}
We have
   \begin{equation}
   A 
   = 
   X^T X + \lambda I,
   \quad
   \theta^* 
   =
   A^{-1}X^T y.
   \end{equation}
For simplicity, let:  
   \begin{equation}\label{eqn:2p.2}
   r = x_i' - x_i,
   \quad
   s = y_i' - y_i.
   \end{equation}
   Then the update can be written as
   \begin{equation}\label{eqn:2p.3}
   X^\dagger = X + e_ir^T,
   \quad
   y^\dagger = y + e_is.
   \end{equation}
The new solution \(\theta^\dagger\) satisfies the ridge normal equation:
\begin{equation}
\bigl({X^\dagger}^T X^\dagger + \lambda I\bigr)\theta^\dagger
=
{X^\dagger}^T y^\dagger,
\end{equation}
where ${X^\dagger}^T X^\dagger$ can be computed as
   \begin{equation}
   \begin{aligned}
   &{X^\dagger}^T X^\dagger 
   =
   (X + e_ir^T)^T(X + e_ir^T)\\
   &=
   X^T X + \underbrace{X^T(e_ir^T) + (e_ir^T)^T X + (e_ir^T)^T(e_ir^T)}_{\text{sum of rank-1 terms}}.
   \end{aligned}
   \end{equation}
   One finds:
   \begin{equation}
   X^T (e_ir^T) 
   = 
   x_ir^T,
   (e_ir^T)^T X 
   = 
   rx_i^T,
   (e_ir^T)^T(e_ir^T)
   =
   rr^T.
   \end{equation}
   Altogether,
   \begin{equation}
   {X^\dagger}^T X^\dagger
   =
   X^T X + \bigl[x_ir^T + rx_i^T + rr^T\bigr].
   \end{equation}
   Hence
   \begin{equation}\label{eqn:2p.8}
   A^\dagger 
   =
   {X^\dagger}^T X^\dagger + \lambda I
   =
   \underbrace{(X^T X + \lambda I)}_{=A}
   +
   \underbrace{\bigl[x_ir^T + rx_i^T + rr^T\bigr]}_{=:M}
   \end{equation}
Using \cref{eqn:2p.3}, we have:
   \begin{equation}
   {X^\dagger}^T y^\dagger
   =
   (X + e_ir^T)^T (y + e_is)
   =
   X^T y + \bigl[x_is + ry_i'\bigr].
   \end{equation}
   Denote
   \begin{equation}
   B 
   =
   x_is + ry_i'.
   \end{equation}
Thus,
\begin{equation}
A^\dagger
=
A + M,
\quad
{X^\dagger}^T y^\dagger
=
X^T y + B.
\end{equation}
Hence the new solution is
\begin{equation}
\theta^\dagger
=
\bigl(A + M\bigr)^{-1}
\bigl[X^T y + B\bigr].
\end{equation}
Subtract the old solution \(\theta^* = A^{-1} X^T y\):
\begin{equation}\label{eqn:2p.13}
\theta^\dagger - \theta^*
=
(A + M)^{-1}\bigl[X^T y + B\bigr]
-
A^{-1}X^T y.
\end{equation}
Plugging $B = x_i (y_i' -y_i) + (x_i - x_i') y_i'$ in \cref{eqn:2p.13}, we get
\begin{equation}\label{eqn:2p.14}
\begin{aligned}
\theta^\dagger - \theta^*
= 
&(A + M)^{-1}\bigl[X^T y + x_i (y_i' -y_i) \\&+ (x_i - x_i') y_i'\bigr]
-
A^{-1}X^T y.    
\end{aligned}
\end{equation}
\end{proof}

Although it is not straightforward to tell, there are cases in which surrogate-based unlearning preserves the original parameter better than exact unlearning. The following \cref{cor:3} enables easier comparison.
\begin{corollary}[Comparison]\label{cor:3}
    Let $X, y, \lambda, \theta^*, X^\dagger, y^\dagger, \theta^\dagger, A$ the same to the \cref{thm:2}. If $x_i = x_i' \neq 0$ and $y_i \neq y_i'$, then 
    \begin{equation}\label{eqn:3.1}
        \frac{\| \widetilde{\theta} - \theta^* \|}{\| \theta^\dagger - \theta^* \|} = \frac{| x_i^T \theta^* - y_i |}
{| y_i' - y_i | |1 - x_i^T A^{-1} x_i |},
    \end{equation}
\textit{i.e.}, surrogate-based unlearning can preserve the original solution more than exact unlearning.
\end{corollary}
\definecolor{customdarkgreen}{rgb}{0.0, 0.35, 0.0}
\begin{figure}
    \centering
        \includegraphics[width=\linewidth]{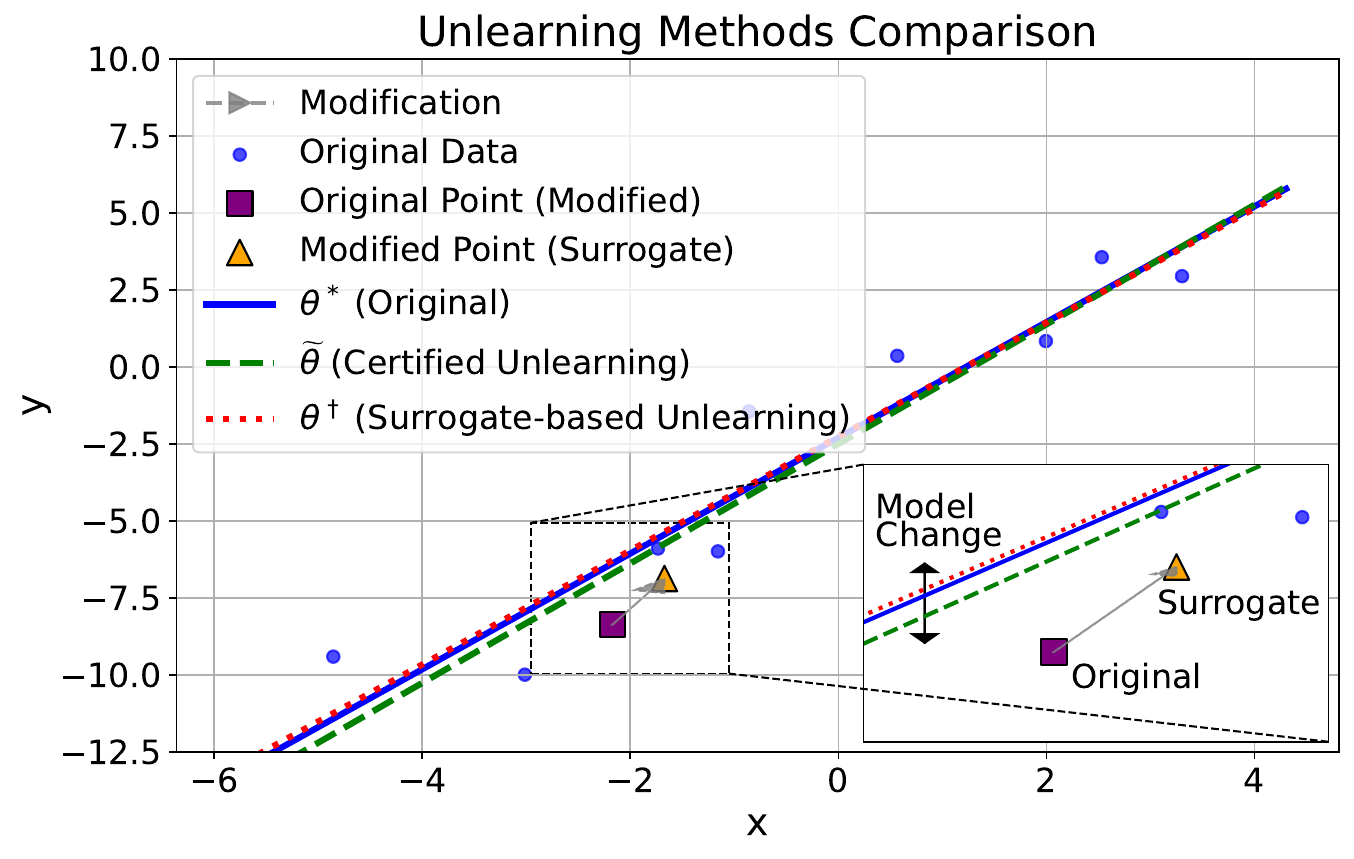}
    \vskip -0.1in
    \caption{Surrogate-based unlearning (\textcolor{red}{$\theta^\dagger$}) can be better than exact unlearing (\textcolor{customdarkgreen}{$\Tilde{\theta}$}) in mapping preservation, \textit{i.e.}, the line is closer to the original (\textcolor{blue}{$\theta^*$}).}
    \label{fig:ridge-regression}
    \vspace{-0.3cm}
\end{figure}
\begin{takeaway}
\Cref{fig:ridge-regression} illustrates the results of \cref{cor:3} implemented on a simple 1D ridge-regression problem. In this simple example, exact unlearning removes any influence of \((x_i, y_i)\) at the cost of a potentially larger parameter change. In contrast, if we carefully choose a surrogate \(\bigl(x_i', y_i'\bigr)\), the resulting parameter \(\theta^\dagger\) can stay closer to the original \(\theta^*\). \cref{subsec:surrogate_selection} discusses how to generalize this idea and select practical surrogates.  
\end{takeaway} 

\begin{proof}
By \cref{eqn:2p.3} and the assumption that $x_i = x_i'$, we have $r=0$. 
By the definition of $M$ in \Cref{eqn:2p.8}, we have $M=0$.
Plugging $M=0$ and $x_i = x_i'$ into \cref{eqn:2p.14}, we get
\begin{align}
\theta^\dagger - \theta^*
& =
A^{-1}\bigl[X^T y + x_i (y_i' -y_i) \bigr]
-
A^{-1}X^T y \label{eqn:3p.1} \\
& =  A^{-1} x_i (y_i' -y_i).  \label{eqn:3p.2}   
\end{align}
Taking norm to \cref{eqn:3p.2} makes
\begin{equation}\label{eqn:3p.3}
  \| \theta^\dagger - \theta^* \| =  \| A^{-1} x_i \| |y_i' -y_i|
\end{equation}
and taking norm to \cref{eqn:1.5} makes
\begin{equation}\label{eqn:3p.4}
\| \widetilde{\theta} - \theta^* \|
=
\frac{| x_i^T \theta^* - y_i |}
{|1 - x_i^T A^{-1} x_i |}
\| A^{-1} x_i \|.    
\end{equation}
By the assumption, \cref{eqn:3p.3} is nonzero. Dividing  \cref{eqn:3p.4} by  \cref{eqn:3p.3}, we have \cref{eqn:3.1}.
\end{proof}


\subsection{Surrogate Dataset Construction}
\label{subsec:surrogate_selection}

\noindent\textbf{Editing tools. }
\label{subsubsec:surrogate_editing}
In practice, any mild editing method that modifies only the undesired attributes while preserving the overall structure can serve as a valid surrogate. We employ three different approaches in Fig. \ref{fig:edit_example}: TediGAN~\cite{xia2021tedigan} for CelebA, SDEdit~\cite{meng2022sdedit} for the \emph{``Xerxes''} case in SD3~\cite{esser2024scaling} , and manual painting for the national flags like \emph{``Japan flag''}. Further discussion on diverse surrogate construction strategies, including noising and flipping, is provided in Table \ref{tab:ablation_surrogate} and Fig. \ref{fig:surrogate_comparison}.

\begin{figure}[t]
    \centering
    \begin{tikzpicture}
        \node[inner sep=0pt] (p1a) at (0, 0) {\includegraphics[width=0.25\linewidth]{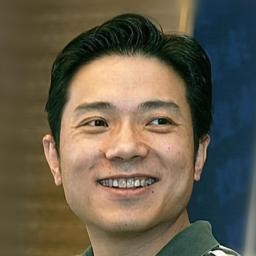}};
        \node[fill=gray!30, fill opacity=0.8, text=black, text opacity=1, inner sep=2pt]
              at ([yshift=-0.25cm]p1a.north) {\footnotesize\sffamily To forget};
    \end{tikzpicture}
    \begin{tikzpicture}
        \node[inner sep=0pt] (p2a) at (0, 0) {\includegraphics[width=0.25\linewidth]{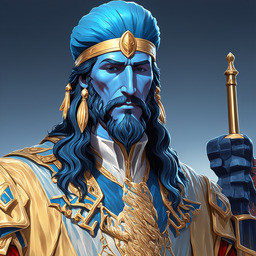}};
        \node[fill=gray!30, fill opacity=0.8, text=black, text opacity=1, inner sep=2pt]
              at ([yshift=-0.25cm]p2a.north) {\footnotesize\sffamily To forget};
    \end{tikzpicture}
    \begin{tikzpicture}
        \node[inner sep=0pt] (p3a) at (0, 0) {\includegraphics[width=0.25\linewidth]{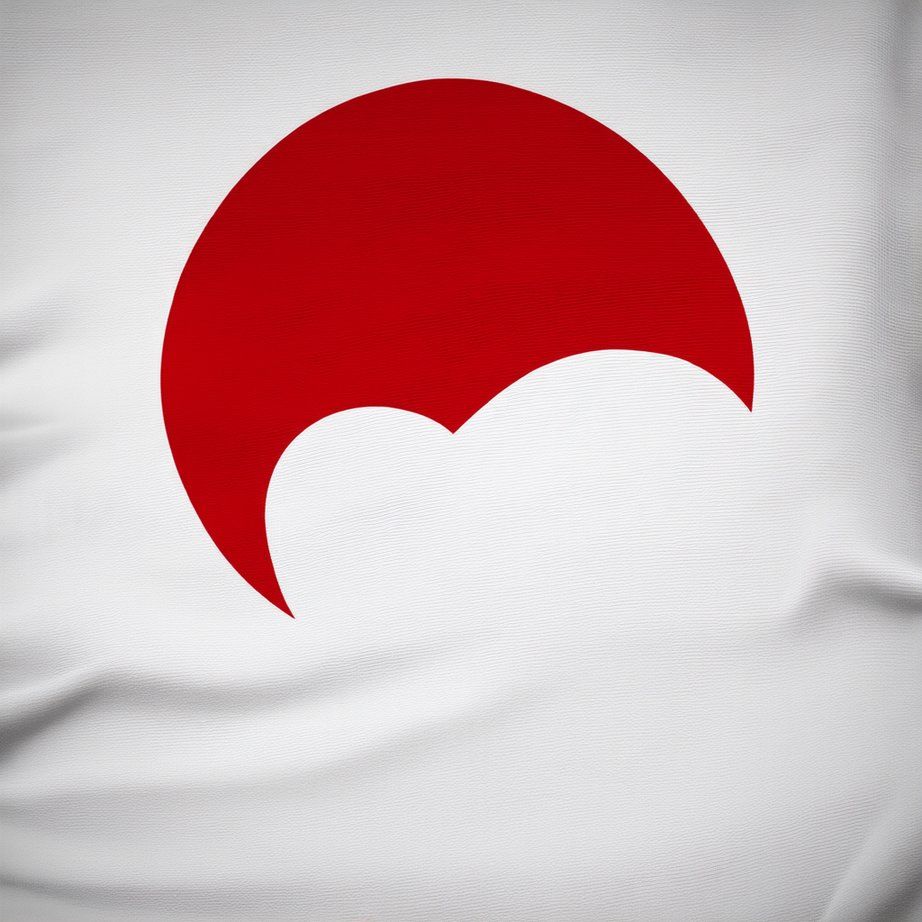}};
        \node[fill=gray!30, fill opacity=0.8, text=black, text opacity=1, inner sep=2pt]
              at ([yshift=-0.25cm]p3a.north) {\footnotesize\sffamily To forget};
    \end{tikzpicture}


    \begin{tikzpicture}
        \node[inner sep=0pt] (p1b) at (0, 0) {\includegraphics[width=0.25\linewidth]{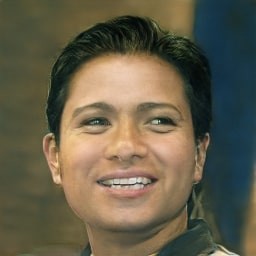}};
        \node[fill=gray!30, fill opacity=0.8, text=black, text opacity=1, inner sep=2pt]
              at ([yshift=-0.25cm]p1b.north) {\footnotesize\sffamily Surrogate};
    \end{tikzpicture}
    \begin{tikzpicture}
        \node[inner sep=0pt] (p2b) at (0, 0) {\includegraphics[width=0.25\linewidth]{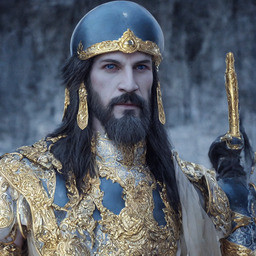}};
        \node[fill=gray!30, fill opacity=0.8, text=black, text opacity=1, inner sep=2pt]
              at ([yshift=-0.25cm]p2b.north) {\footnotesize\sffamily Surrogate};
    \end{tikzpicture}
    \begin{tikzpicture}
        \node[inner sep=0pt] (p3b) at (0, 0) {\includegraphics[width=0.25\linewidth]{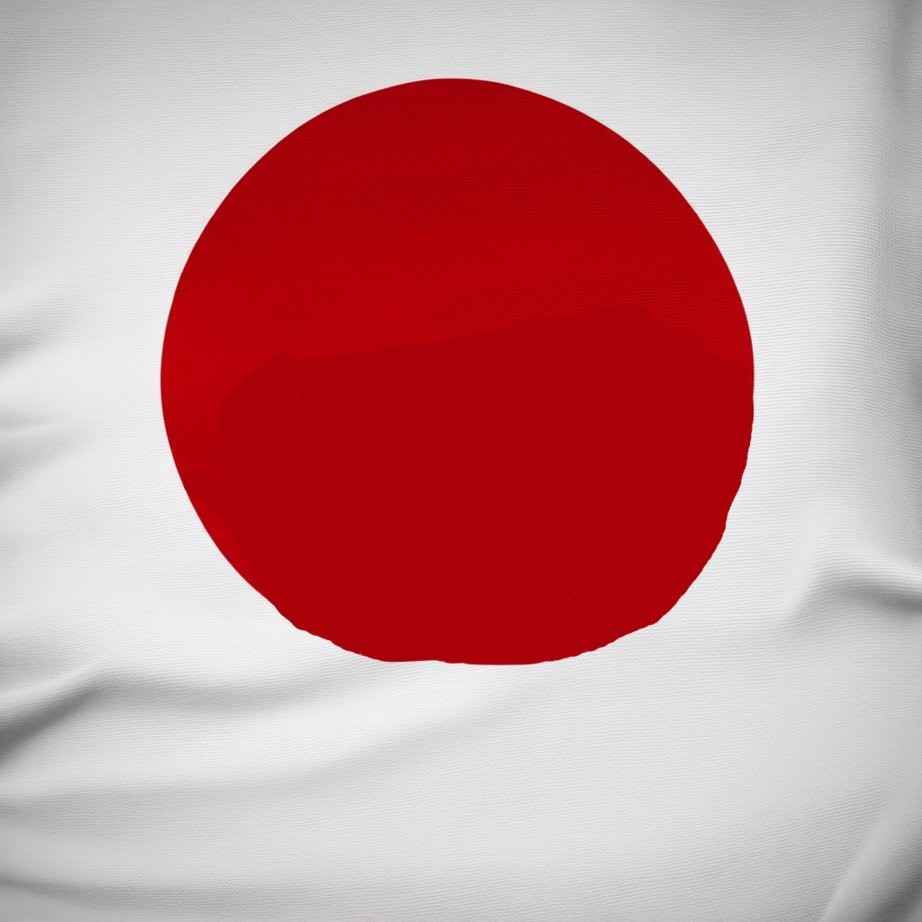}};
        \node[fill=gray!30, fill opacity=0.8, text=black, text opacity=1, inner sep=2pt]
              at ([yshift=-0.25cm]p3b.north) {\footnotesize\sffamily Surrogate};
    \end{tikzpicture}
    \caption{Surrogate data construction of (left) CelebA with TediGAN, (middle) SD3 with SDEdit, and (right) manual editing.}
    \label{fig:edit_example}
\end{figure}
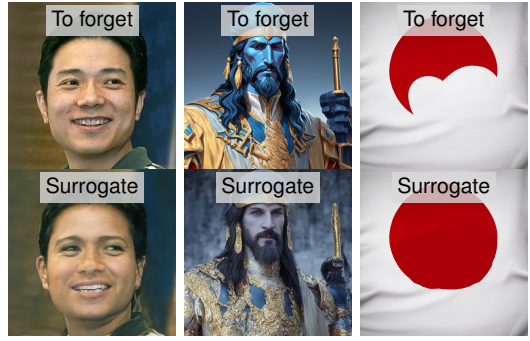

\noindent\textbf{Scalability, automation, and flexibility. }
These surrogate construction methods are straightforward, efficient, and automatable. Using TediGAN~\cite{xia2021tedigan} or SDEdit~\cite{meng2022sdedit}, we constructed the surrogate images within five seconds per image with fixed prompts and consistent hyperparameters. This deterministic pipeline ensures scalability to larger datasets. Additionally, manual editing provides a crucial advantage: if a user encounters an undesirable output on their device, they can manually edit the image in just a few seconds, creating a suitable surrogate that prevents future occurrences of similar undesired generations. Our approach accommodates diverse editing techniques, combining automation and flexibility to streamline surrogate data creation efficiently.

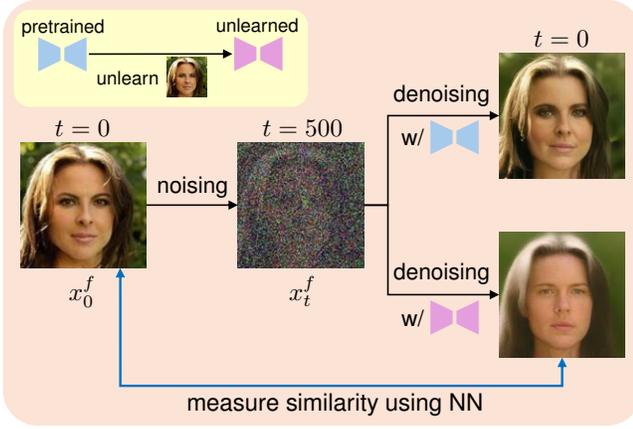
\begin{figure}[t!]
    \centering
    \begin{tikzpicture}
        \node[anchor=south west,inner sep=0] (image) at (0,0) {\includegraphics[width=0.97\linewidth]{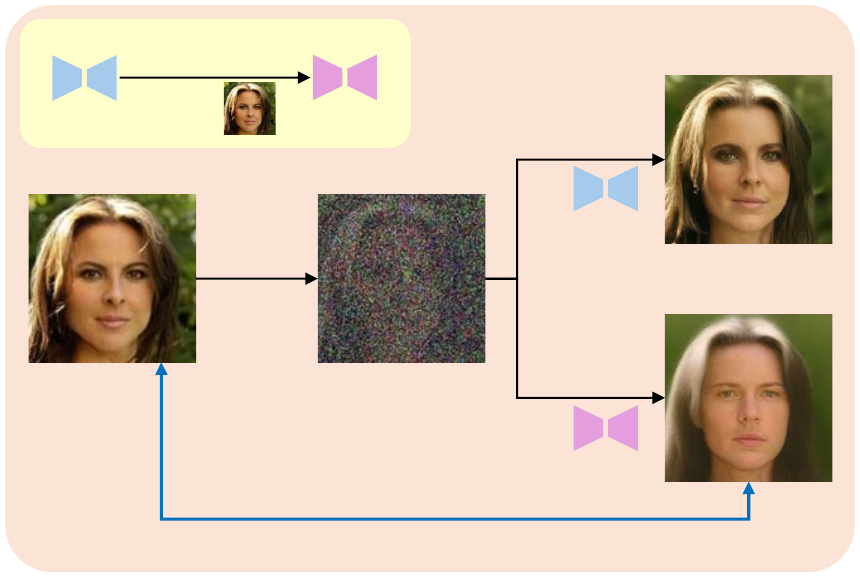}};
        \begin{scope}[x={(image.south east)},y={(image.north west)}]
            \node[scale=0.7] at (0.1, 0.925) {\sffamily pretrained};
            \node[scale=0.7] at (0.4, 0.93) {\sffamily unlearned};
            \node[scale=0.7] at (0.2, 0.81) {\sffamily unlearn};
    
            \node[scale=0.95] at (0.13, 0.695) {$t = 0$};
            \node[scale=0.95] at (0.13, 0.32) {$x_0^f$};
    
            \node[scale=0.95] at (0.47, 0.695) {$t = 500$};
            \node[scale=0.95] at (0.47, 0.32) {$x_t^f$};
            \node[scale=0.85] at (0.3, 0.56) {\sffamily noising};
    
            \node[scale=0.95] at (0.87, 0.9) {$t = 0$};
    
            \node[scale=0.85] at (0.685, 0.77) {\sffamily denoising};
            \node[scale=0.85] at (0.685, 0.36) {\sffamily denoising};
            \node[scale=0.85] at (0.64, 0.675) {\sffamily w/};
            \node[scale=0.85] at (0.64, 0.26) {\sffamily w/};
            \node[scale=0.9] at (0.52, 0.06) {\sffamily measure similarity using NN};
        \end{scope}
    \end{tikzpicture}
    \caption{Overview of SSCD to validate the success of forgetting.}
    \label{fig:definition_of_sscd}
\end{figure}

\begin{figure*}[!t]
    \centering
      \subfloat[Qualitative comparisons for forgetting.\label{fig:quali_forgetting}]{%
        \begin{minipage}[c]{0.48\linewidth}
            \centering
            \small
            \setlength{\tabcolsep}{1.5pt}
            \begin{tabular}{ccccc}
                \toprule
                Target & NegGrad & EraseDiff & SISS & Ours \\
                \midrule
                \includegraphics[width=0.184\textwidth]{images/main/sscd_figure/00068.jpg} &
                \begin{overpic}[width=0.184\textwidth]{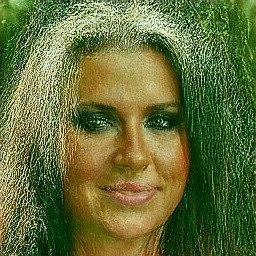}
                \end{overpic} &
                \begin{overpic}[width=0.184\textwidth]{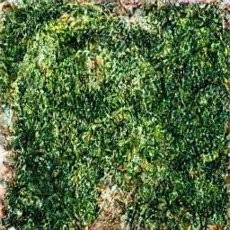}
                \end{overpic} &
                \begin{overpic}[width=0.184\textwidth]{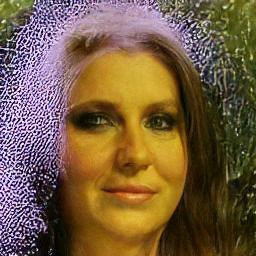}
                \end{overpic} &
                \begin{overpic}[width=0.184\textwidth]{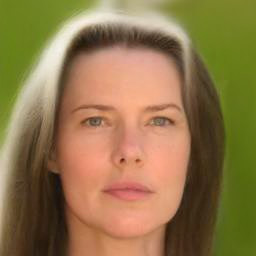}
                \end{overpic} \\
                \includegraphics[width=0.184\textwidth]{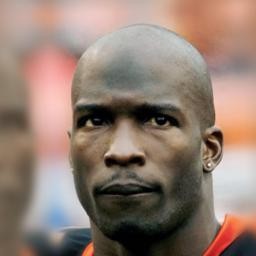} &
                \begin{overpic}[width=0.184\textwidth]{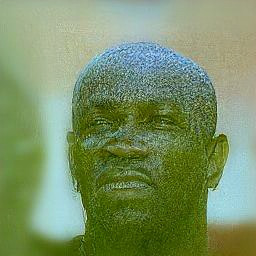}
                \end{overpic} &
                \begin{overpic}[width=0.184\textwidth]{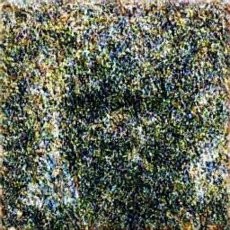}
                \end{overpic} &
                \begin{overpic}[width=0.184\textwidth]{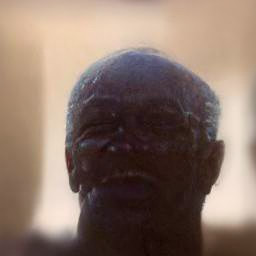}
                \end{overpic} &
                \begin{overpic}[width=0.184\textwidth]{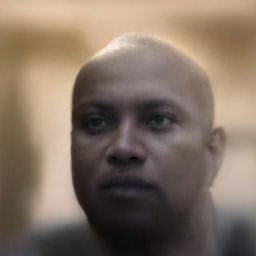}
                \end{overpic} \\
                \includegraphics[width=0.184\textwidth]{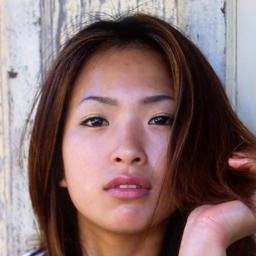} &
                \begin{overpic}[width=0.184\textwidth]{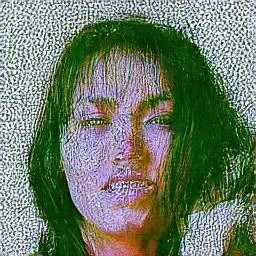}
                \end{overpic} &
                \begin{overpic}[width=0.184\textwidth]{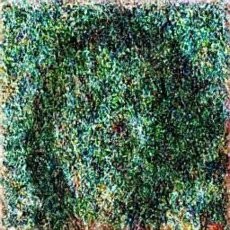}
                \end{overpic} &
                \begin{overpic}[width=0.184\textwidth]{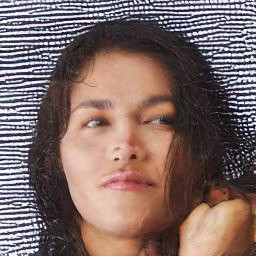}
                \end{overpic} &
                \begin{overpic}[width=0.184\textwidth]{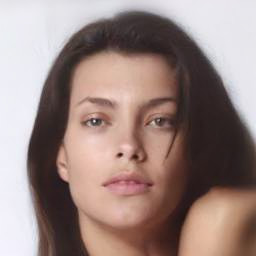}
                \end{overpic} \\
                \bottomrule
            \end{tabular}
        \end{minipage}
        }
      \subfloat[Qualitative comparisons for model integrity.\label{fig:quali_integrity}]{%
        \begin{minipage}[c]{0.48\linewidth}
            \centering
            \small
            \setlength{\tabcolsep}{1.5pt}
            \begin{tabular}{ccccc}
                \toprule
                Pretrained & NegGrad & EraseDiff & SISS & Ours \\
                \midrule
                \includegraphics[width=0.184\textwidth]{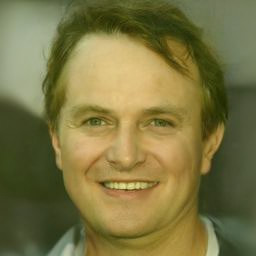} &
                \includegraphics[width=0.184\textwidth]{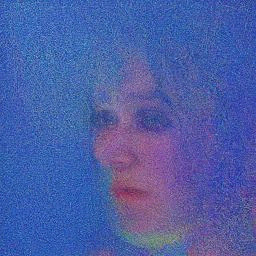} &
                \includegraphics[width=0.184\textwidth]{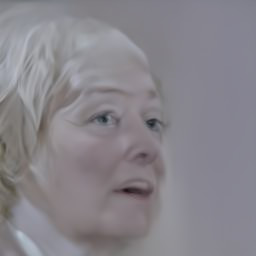} &
                \includegraphics[width=0.184\textwidth]{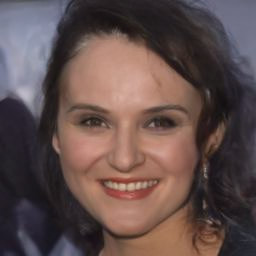} &
                \includegraphics[width=0.184\textwidth]{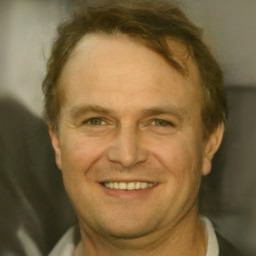} \\
                \includegraphics[width=0.184\textwidth]{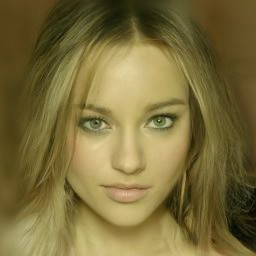} &
                \includegraphics[width=0.184\textwidth]{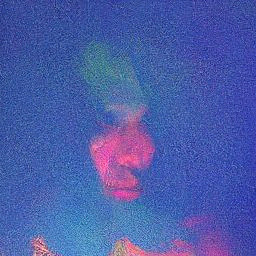} &
                \includegraphics[width=0.184\textwidth]{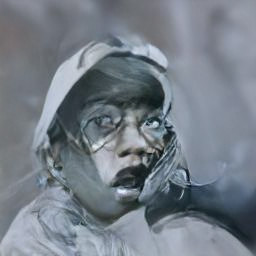} &
                \includegraphics[width=0.184\textwidth]{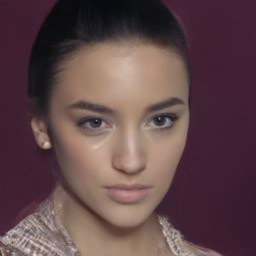} &
                \includegraphics[width=0.184\textwidth]{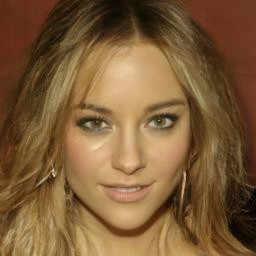} \\
                \includegraphics[width=0.184\textwidth]{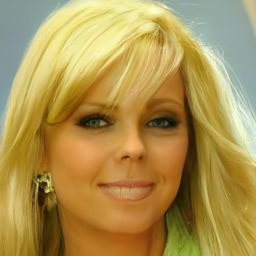} &
                \includegraphics[width=0.184\textwidth]{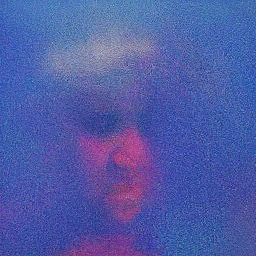} &
                \includegraphics[width=0.184\textwidth]{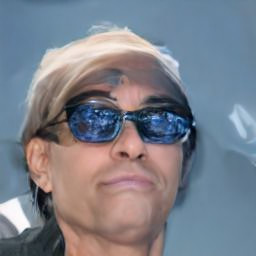} &
                \includegraphics[width=0.184\textwidth]{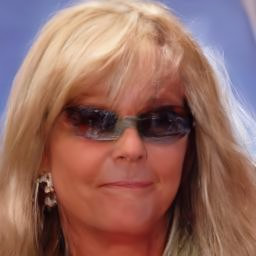} &
                \includegraphics[width=0.184\textwidth]{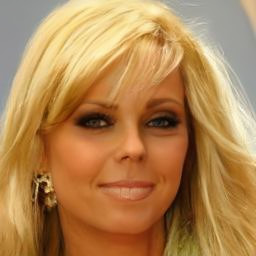} \\
                \bottomrule
            \end{tabular}
    \end{minipage}
    }
    \caption{Qualitative results on (a) forgetting target instance on (b) preserving model integrity on generation after unlearning.}
    \label{fig:sscd_combine}
\end{figure*}

\begin{table*}[t!]
\caption{Quantitative results on forgetting single instance and multiple instances. Ours$^\dagger$ refers to randomly sampling the surrogate data $x_0^s$ from the remember dataset $\mathcal{D}_{r}$ at each update.}
\label{tab:celeba}
\begin{center}
\resizebox{\textwidth}{!}{
\begin{tabular}{l
@{  }>{\centering\arraybackslash}p{1.7cm}@{ }
@{    }>{\centering\arraybackslash}p{1.4cm}@{}
@{}>{\centering\arraybackslash}p{1.4cm}@{}
@{}>{\centering\arraybackslash}p{1.1cm}@{  }
@{    }>{\centering\arraybackslash}p{1.3cm}@{  }
@{    }>{\centering\arraybackslash}p{1.7cm}@{    }
cccc}
\toprule
& \multicolumn{5}{c|}{Single Instance} & \multicolumn{5}{c}{Multiple Instances} \\
\cmidrule(r){2-6} \cmidrule(l){7-11}
            & SSCD$<$0.4  & LPIPS$\downarrow$   & SSIM$\uparrow$ & $\text{FID}_{pre}$$\downarrow$   & $\text{FID}_{real}$$\downarrow$ 
            & SSCD$<$0.4  & LPIPS$\downarrow$   & SSIM$\uparrow$ & $\text{FID}_{pre}$$\downarrow$   & $\text{FID}_{real}$$\downarrow$ \\
\midrule
NegGrad~\cite{golatkar2020eternal}     & \checkmark (0.18) & 1.00 & 0.34 & 223.5 & 226.1 
            & \checkmark (0.14) & 0.96 & 0.07 & 415.4 & 410.8 \\
EraseDiff~\cite{wu2024erasediff}   & \checkmark (0.03) & 0.74 & 0.61 & 80.7  & 87.8 
            & \checkmark (0.04) & 0.66 & 0.71 & 45.8  & 51.1 \\
SISS~\cite{silas2024data}        & \checkmark (0.21) & 0.43 & \cellcolor{tabsecond}0.84 & \cellcolor{tabsecond}15.7 & \cellcolor{tabsecond}21.3 
            & \checkmark (0.22) & 0.51 & \cellcolor{tabsecond}0.81 & 26.8 & 31.5 \\
Ours$^\dagger$
            & \checkmark (0.17) & \cellcolor{tabsecond}0.41 & \cellcolor{tabsecond}0.84 & \cellcolor{tabsecond}15.7 & 24.1 
            & \checkmark (0.16) & \cellcolor{tabfirst}0.36 & 0.80 & \cellcolor{tabfirst}13.9 & \cellcolor{tabsecond}27.2 \\
Ours        & \checkmark (0.19) & \cellcolor{tabfirst}0.35 & \cellcolor{tabfirst}0.87 & \cellcolor{tabfirst}9.8 & \cellcolor{tabfirst}18.0 
            & \checkmark (0.21) & \cellcolor{tabsecond}0.37 & \cellcolor{tabfirst}0.88 & \cellcolor{tabsecond}17.5 & \cellcolor{tabfirst}23.4 \\
\hline
\end{tabular}
}
\end{center}
\end{table*}

\section{Experiments}
\label{sec:experiments}
We show that our prompt-free unlearning method unlearns specific data while ensuring model integrity for both single and multiple images in unconditional DMs, and is applicable to unlearning undesirable images in conditional DMs.

\subsection{Instance Unlearning on Unconditional Diffusion Models}
\label{sec:details}

\noindent\textbf{Datasets. }
CelebA~\cite{liu2015faceattributes} is a large-scale facial attributes dataset containing over 200,000 celebrity images with 40 attribute annotations, widely used for facial recognition and attribute prediction tasks. CelebA-HQ~\cite{karras2018progressive} is a high-quality version of the CelebA dataset, containing 30,000 images of celebrity faces with improved resolution and visual fidelity. FFHQ~\cite{karras2019style} is a high-quality dataset of 70,000 human face images with diverse ages, ethnicities, and accessories, designed for generative model training and evaluation.

\noindent\textbf{Models. }
CelebA-HQ data unlearning experiments utilized a pretrained model checkpoint provided by ~\cite{ho2020denoising}, which is available at \url{https://huggingface.co/google/ddpm-celebahq-256}. Experiments on correcting misrepresentations applied Stable Diffusion 3~\cite{esser2024scaling}.

\noindent\textbf{Metrics. }
We utilize LPIPS~\cite{zhang2018unreasonable}, SSIM~\cite{wang2004image} and FID~\cite{heusel2017gans} to measure the model integrity. LPIPS~\cite{zhang2018unreasonable} measures the perceptual similarity between two images by comparing their deep feature representations extracted from a pretrained neural network. Lower LPIPS values indicate higher similarity, meaning fewer semantic changes between the compared images. In our context, LPIPS is used to evaluate the semantic consistency of a model's outputs before and after unlearning, given the same random seed. SSIM~\cite{wang2004image} quantifies the structural similarity between two images by comparing luminance, contrast, and structure. It ranges from 0 to 1, where higher values indicate greater similarity. In our experiments, SSIM is used to assess how much structural information is preserved in the model's outputs after unlearning. FID~\cite{heusel2017gans} measures the distance between two distributions of images in the feature space of a pretrained Inception network. It is widely used to evaluate the quality of generated images. A lower FID indicates that the distance between two distributions is closer. We calculate all metrics by averaging the results from 6 experiments, each evaluated using 10,000 outputs from the unlearned model and pretrained model. Self Supervised Copy Detection (SSCD)~\cite{pizzi2022self} is employed to verify whether the forgetting objective is achieved. As shown in Fig. \ref{fig:definition_of_sscd}, SSCD is a similarity metric between denoised images before and after the unlearning. Based on previous studies~\cite{somepalli2023understanding, wen2024detecting, chen2024extracting}, an SSCD below 0.4 suggests sufficient forgetting. However, as shown in Fig. \ref{fig:quali_forgetting}, severe artifacts lead to excessively low SSCD, implying an excessive loss of model integrity. Therefore, a low SSCD is not a sufficient indicator of successful unlearning: both SSCD$<$0.4 (1st column in Table \ref{tab:celeba}) and preserving model integrity should be achieved.

\noindent\textbf{Experiment setup. }
All experiments were conducted on a single A100 80GB GPU with a batch size of 8. DDPM was trained with learning rate $5 \times 10^{-6}$, 240 steps, and $\beta = 5 \times 10^{-5}$. Stable Diffusion 3 was trained with learning rate $1 \times 10^{-5}$, 100 steps, and $\beta = 2 \times 10^{-4}$. $\beta$ was calculated with respect to $t_{max}=1000$. For the multiple instance unlearning task, the following identities are unlearned sequentially: \textit{`Robin Li', `Kate del Castillo', `Elva Hsiao'}, and \textit{`Ed Harris'}. For the Stable Diffusion 3 experiments, we utilized LoRA~\cite{hu2022lora} to handle the large size of the model. Also, for the Stable Diffusion 3 experiments, since we cannot access the training dataset, we used 100 images generated from the same prompt as the remember dataset ($D_r$) instead. For the surrogate construction of flag images, \textit{Drawing} software is used for manual painting.

As for other baselines, The learning rate for all baselines was set to $5 \times 10^{-6}$ for DDPM experiments and $1 \times 10^{-5}$ for Stable Diffusion 3 experiments. In the DDPM experiments, for SISS~\cite{silas2024data}, we followed the implementation provided in the code. NegGrad~\cite{golatkar2020eternal} and EraseDiff~\cite{wu2024erasediff} were trained for 10 and 60 steps, respectively. In the Stable Diffusion 3 experiments, SISS, NegGrad, and EraseDiff were trained for 100, 60, and 100 steps, respectively.

\begin{figure*}[ht]
    \centering
    \small
    \setlength{\tabcolsep}{0.5pt}
    \begin{tabular}{@{}ccc|ccc|ccc|c@{}}
        \toprule
        & & & \multicolumn{3}{c|}{Prompt-based} & \multicolumn{3}{c|}{Prompt-free} &  \\
        & & Original & ESD~\cite{gandikota2023erasing} & SPM~\cite{lyu2024one} & DUO~\cite{park2024direct} &  \makecell{\small{Neggrad}\!\small{\cite{golatkar2020eternal}}} & \makecell{\small{EraseDiff}\!\small{\cite{wu2024erasediff}}} & SISS~\cite{silas2024data} & Ours \\
        \midrule
        \multirow{3}{*}{\parbox[c][2.6cm][c]{0.5cm}{\centering\rotatebox{90}{\textbf{To forget}}}} &
        \raisebox{2.6ex}{\rotatebox[origin=l]{90}{\small\emph{``Xerxes''}}} &
        \includegraphics[width=0.1\textwidth]{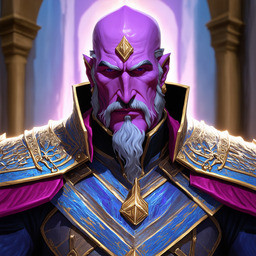} &
        \includegraphics[width=0.1\textwidth]{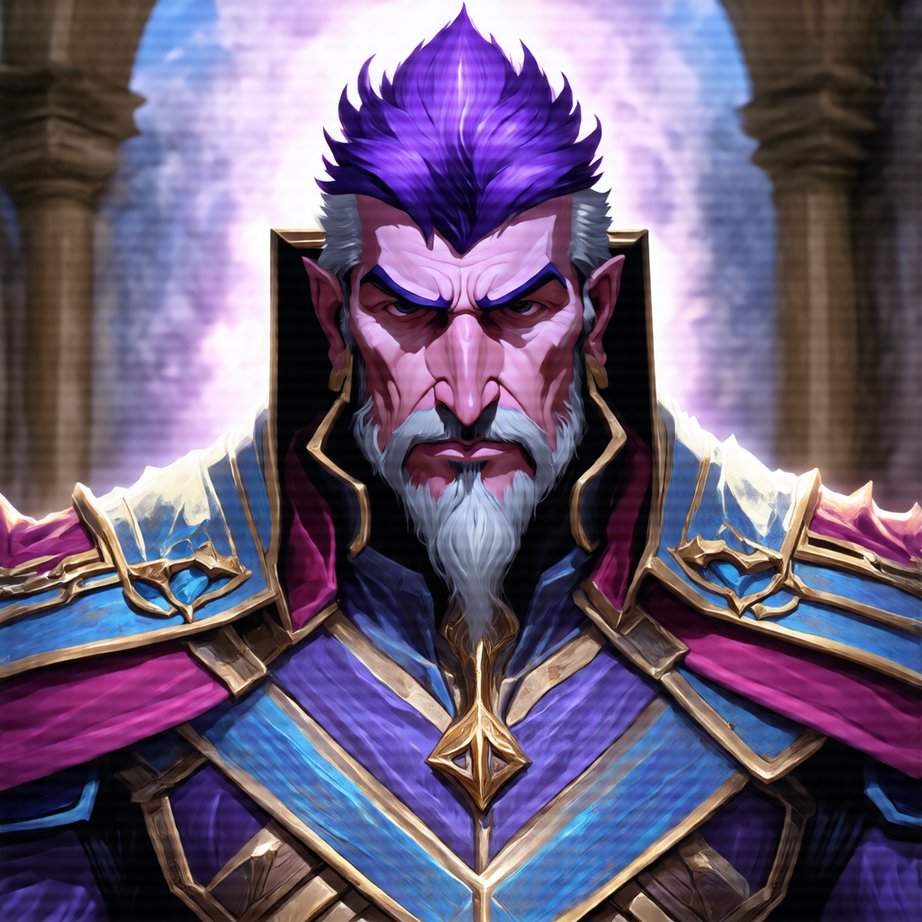} &
        \includegraphics[width=0.1\textwidth]{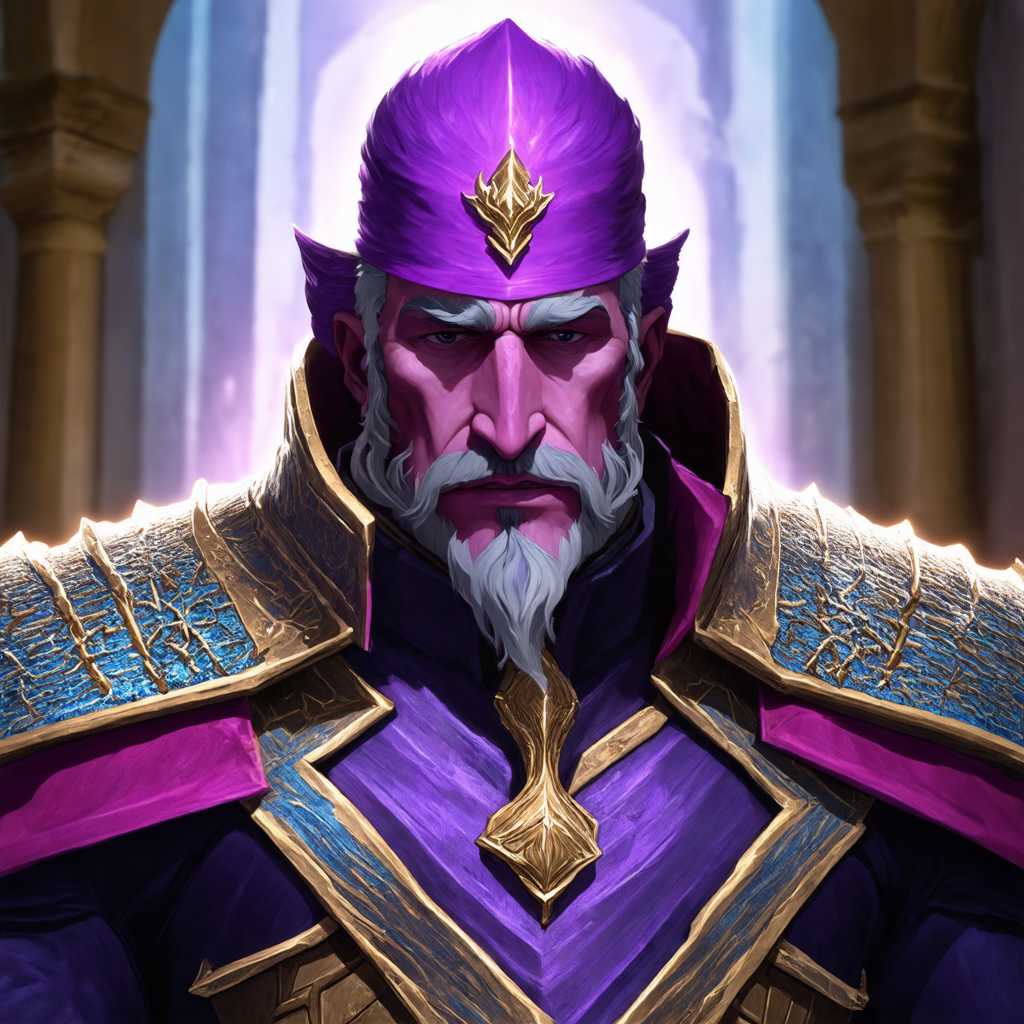} &
        \includegraphics[width=0.1\textwidth]{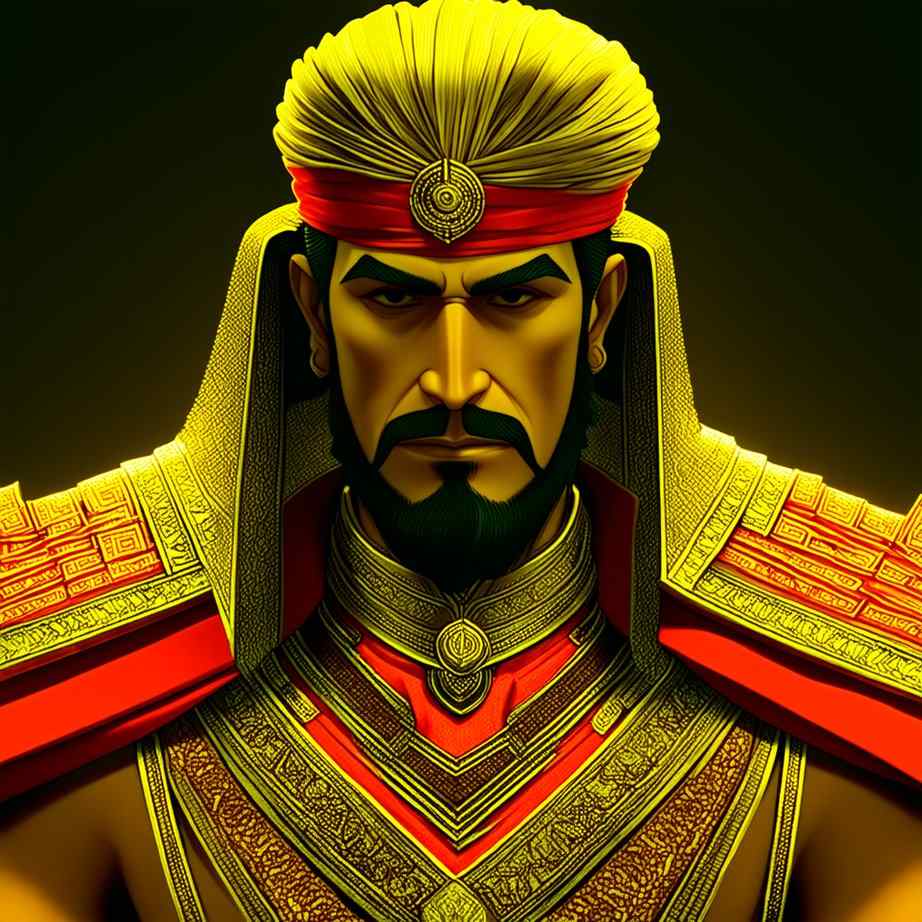} &
        \includegraphics[width=0.1\textwidth]{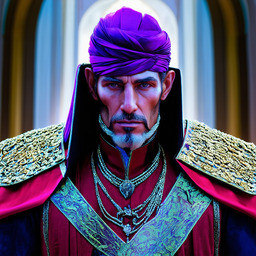} &
        \includegraphics[width=0.1\textwidth]{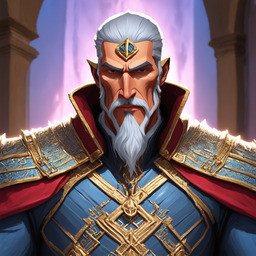} &
        \includegraphics[width=0.1\textwidth]{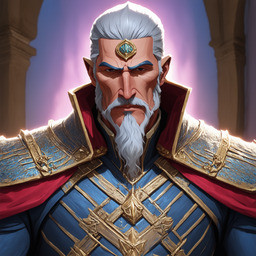} &
        \includegraphics[width=0.1\textwidth]{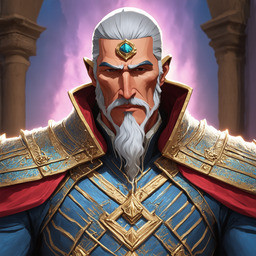} \\

        & \raisebox{0.6ex}{\rotatebox[origin=l]{90}{\small\emph{``Japan flag''}}} &
        \includegraphics[width=0.1\textwidth]{images/main/sd3_figure_down/Japan/target_original.jpg} &
        \includegraphics[width=0.1\textwidth]{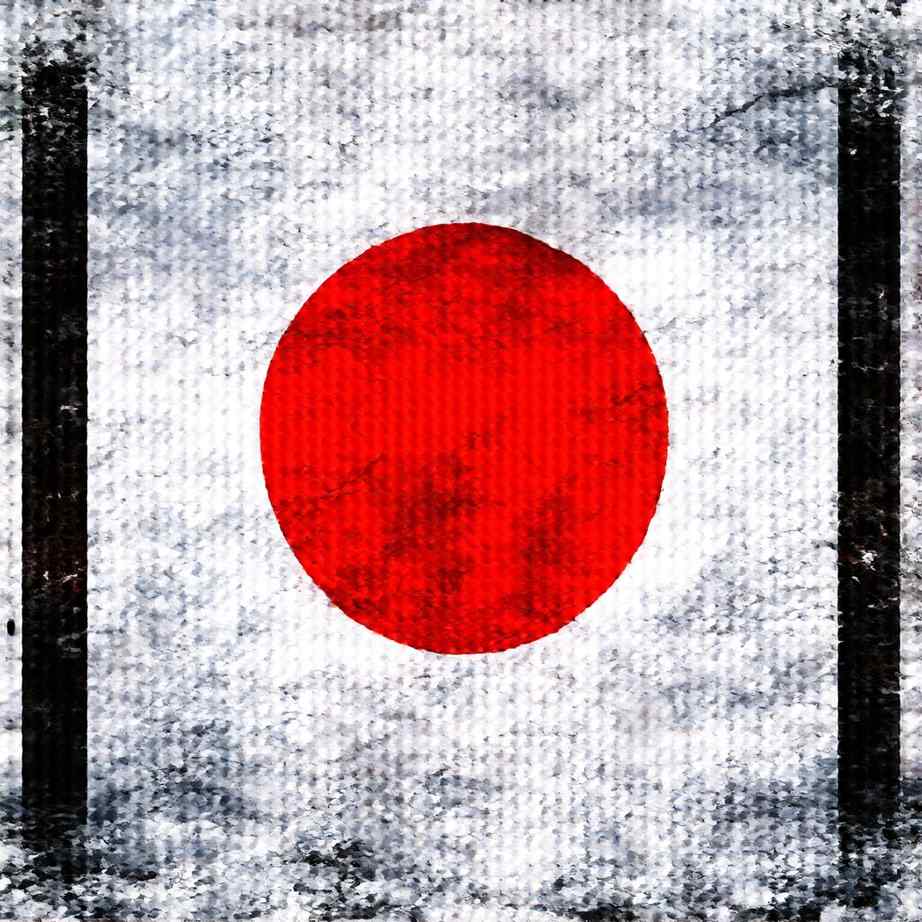} &
        \includegraphics[width=0.1\textwidth]{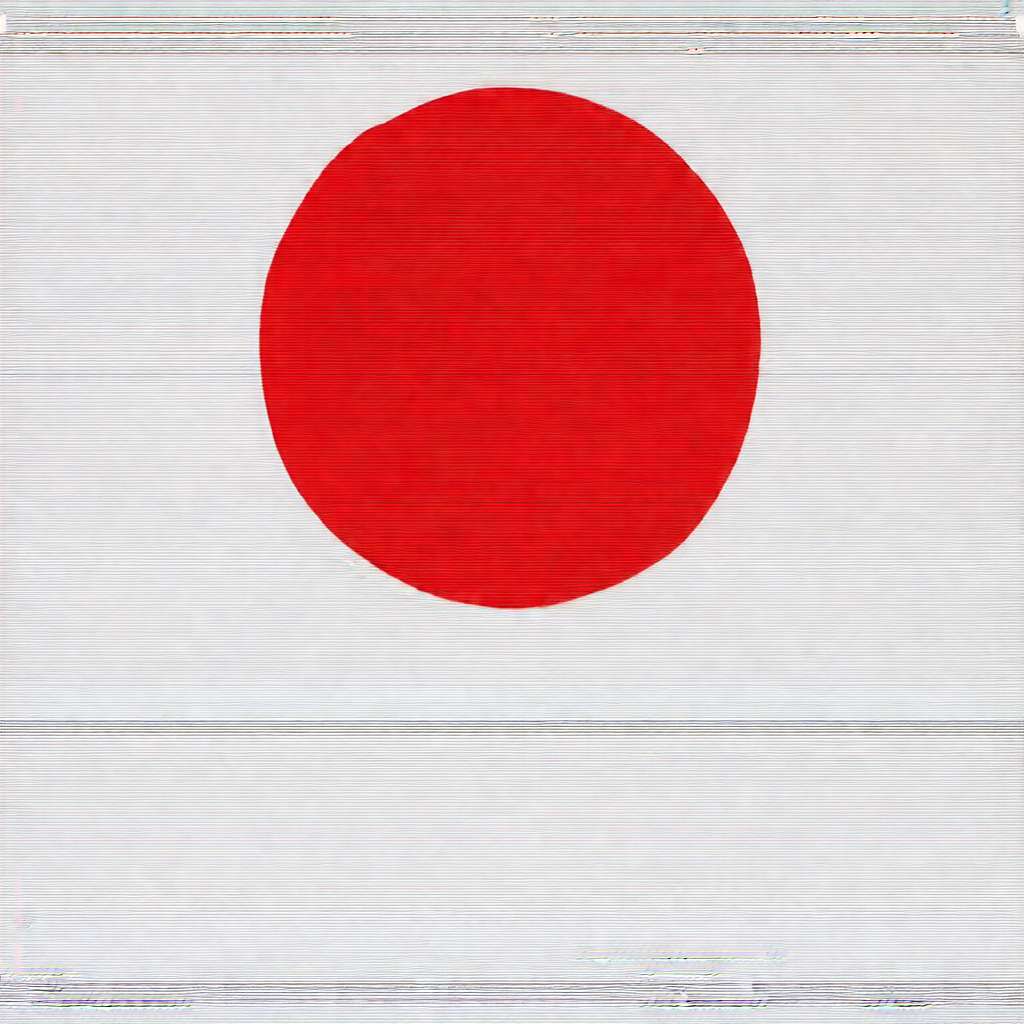} &
        \includegraphics[width=0.1\textwidth]{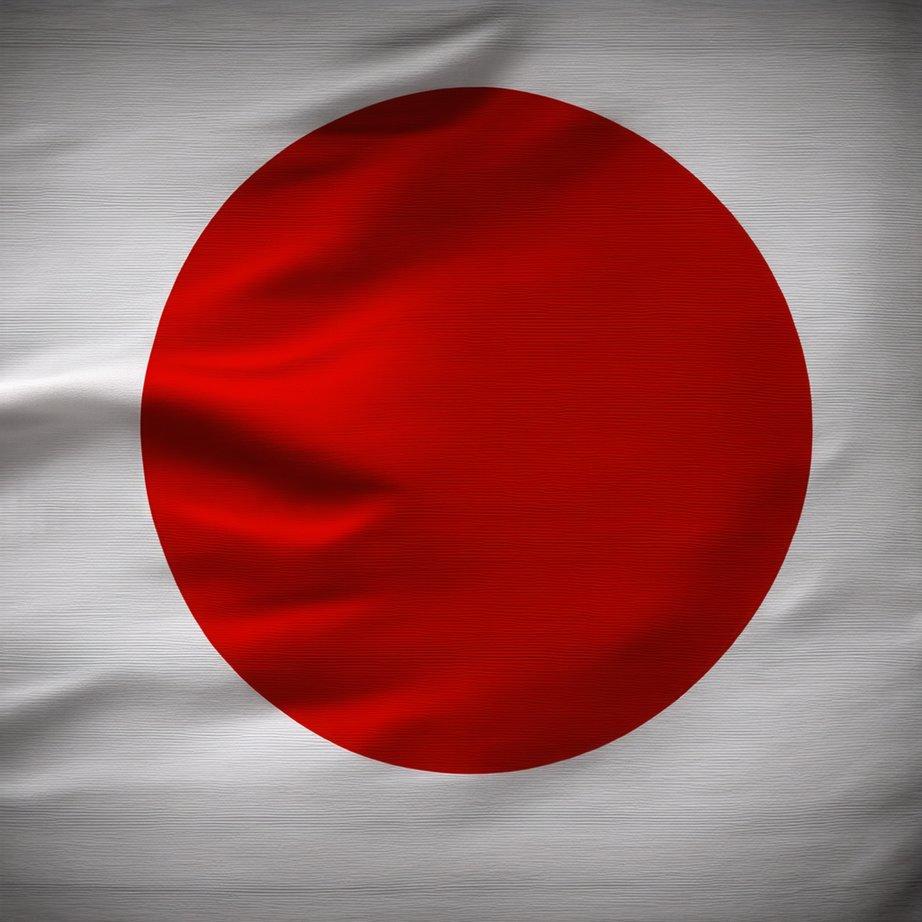} &
        \includegraphics[width=0.1\textwidth]{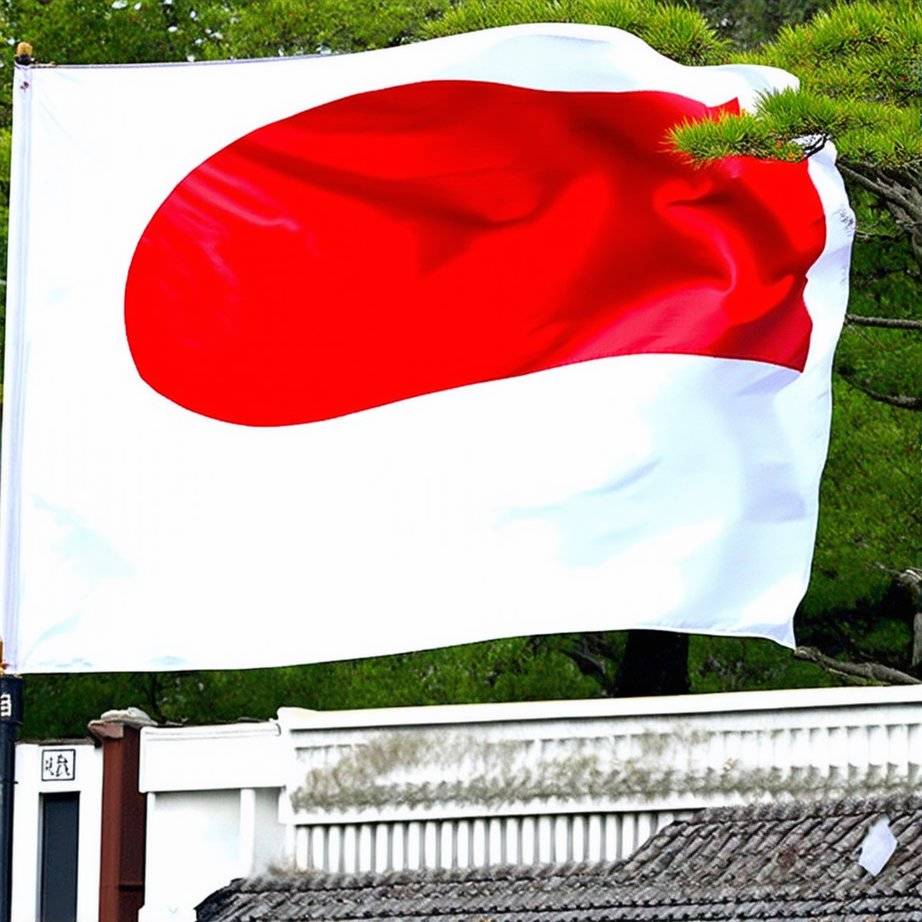} &
        \includegraphics[width=0.1\textwidth]{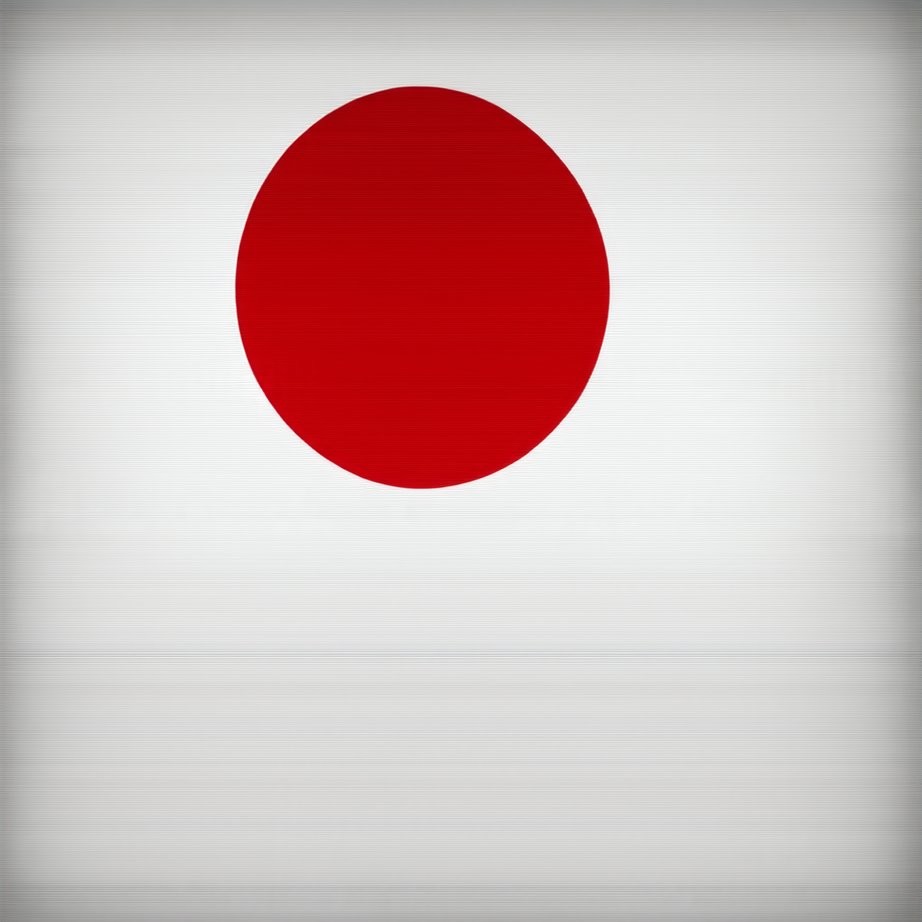} &
        \includegraphics[width=0.1\textwidth]{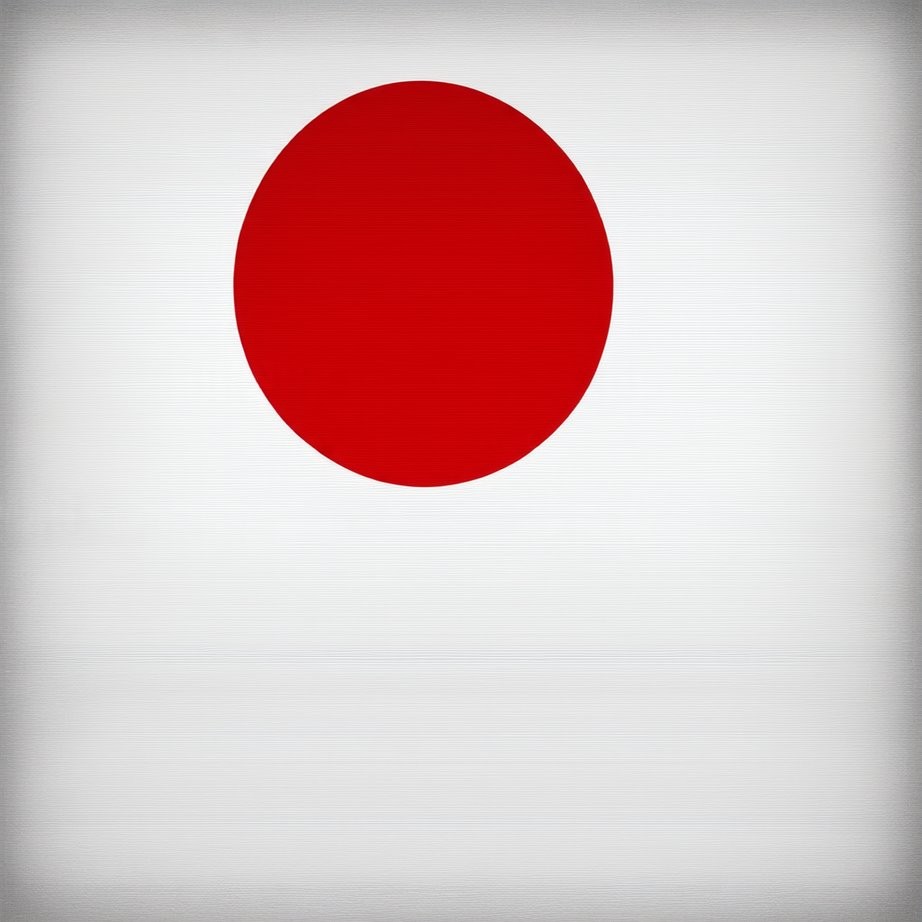} &
        \includegraphics[width=0.1\textwidth]{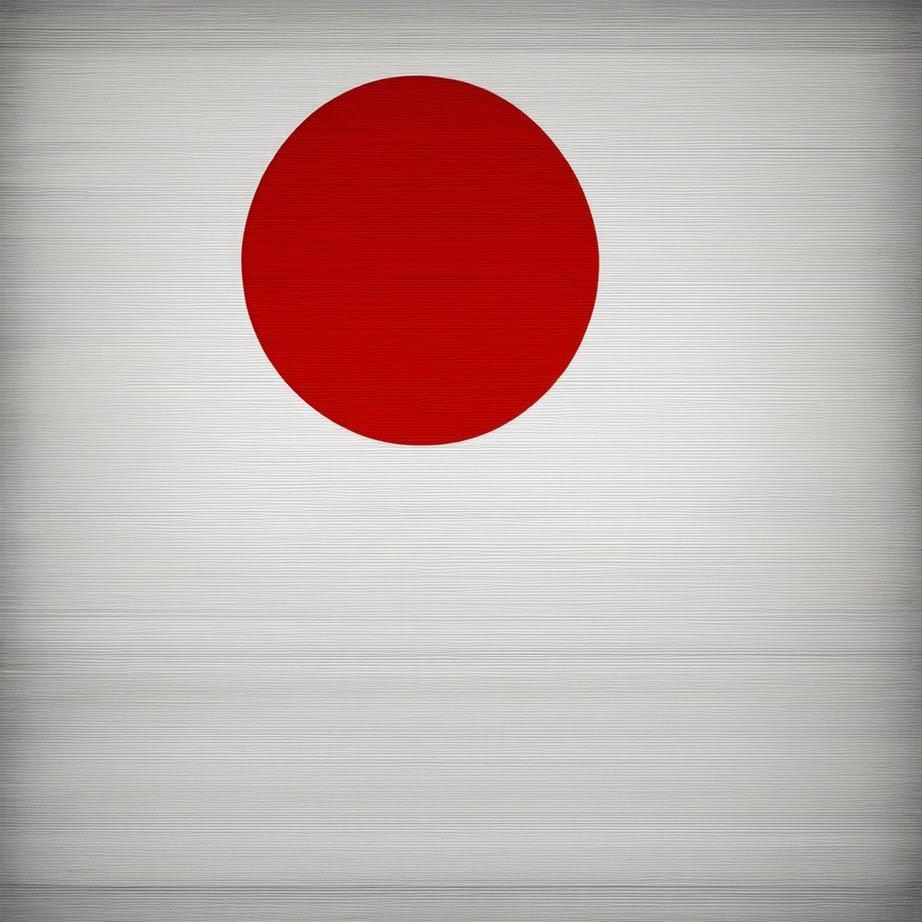} \\

        & \raisebox{0.3ex}{\rotatebox[origin=l]{90}{\small\emph{``Ireland flag''}}} &
        \includegraphics[width=0.1\textwidth]{images/main/sd3_figure_down/Ireland/target_pretrained.jpg} &
        \includegraphics[width=0.1\textwidth]{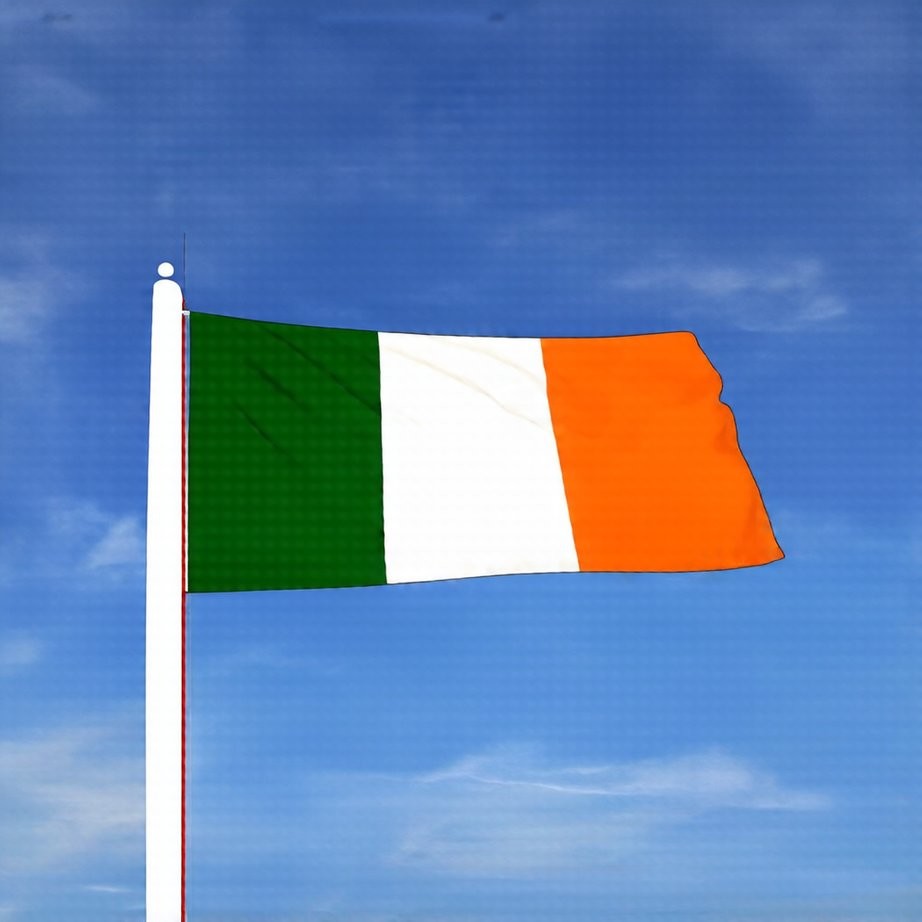} &
        \includegraphics[width=0.1\textwidth]{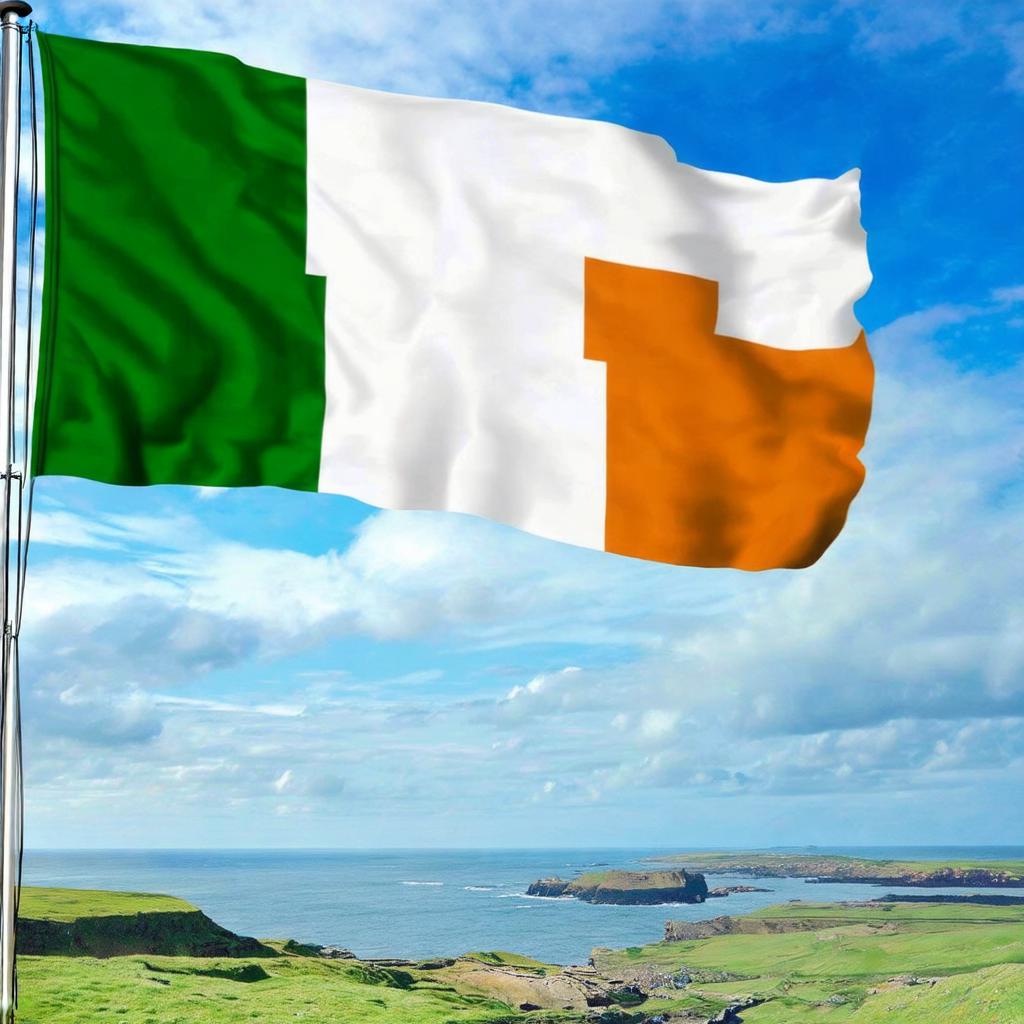} &
        \includegraphics[width=0.1\textwidth]{images/main/sd3_figure_down/Ireland/target_duo.jpg} &
        \includegraphics[width=0.1\textwidth]{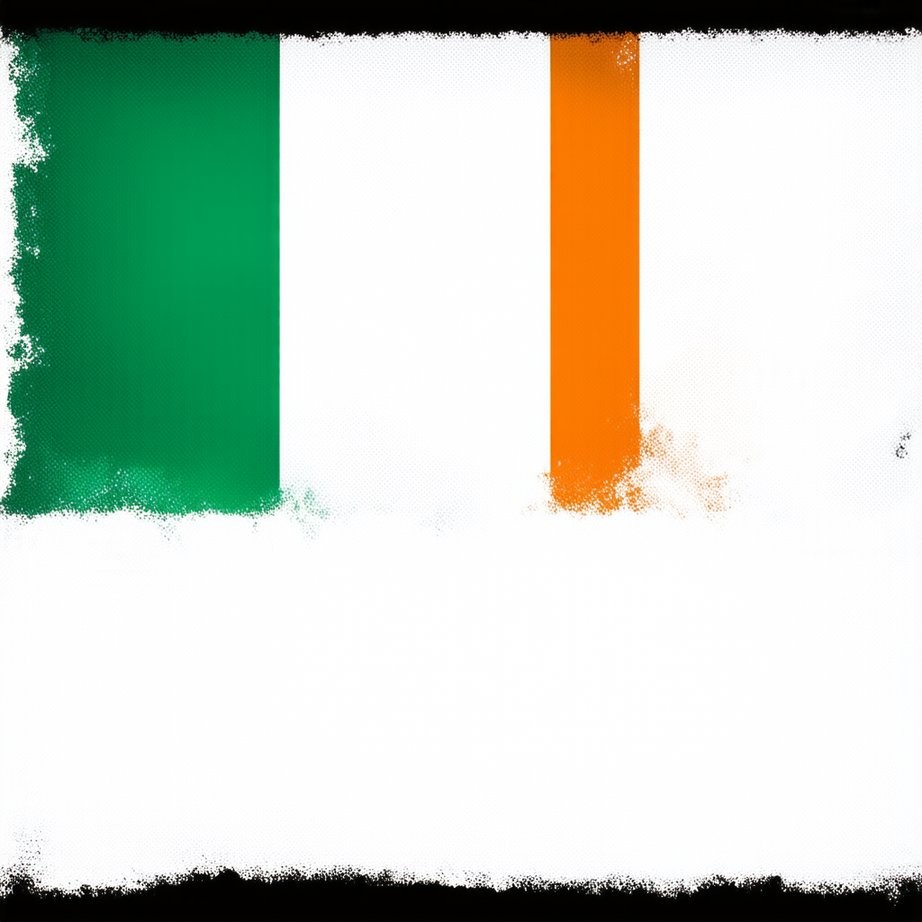} &
        \includegraphics[width=0.1\textwidth]{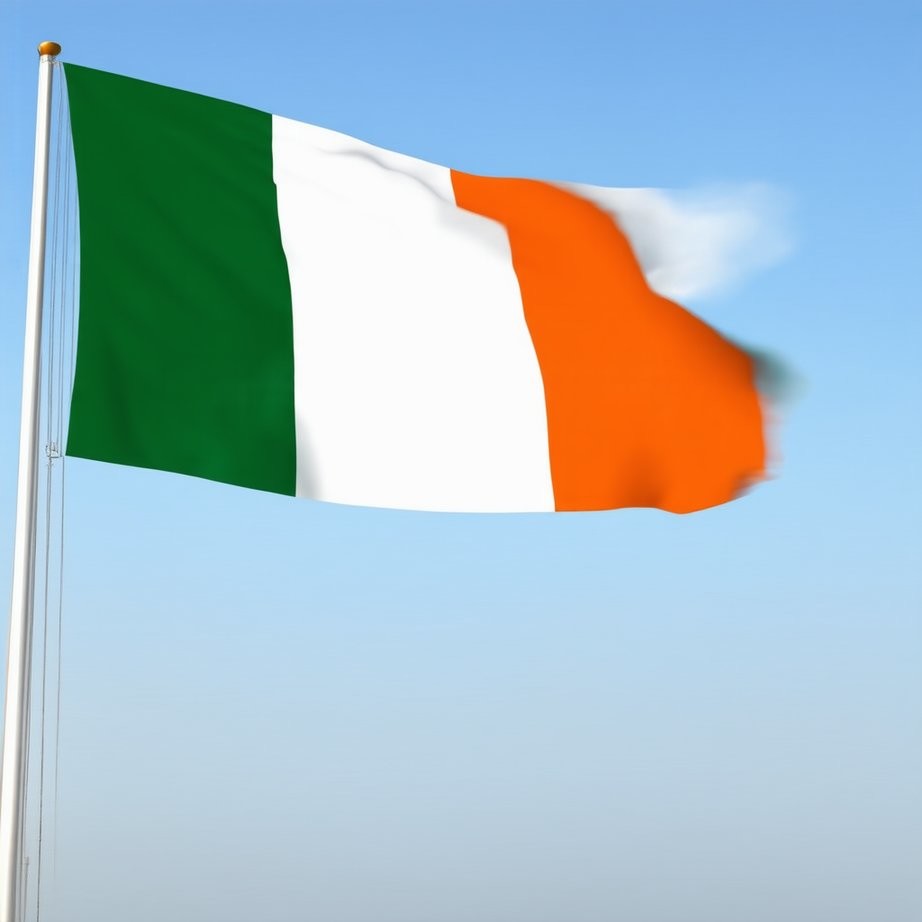} &
        \includegraphics[width=0.1\textwidth]{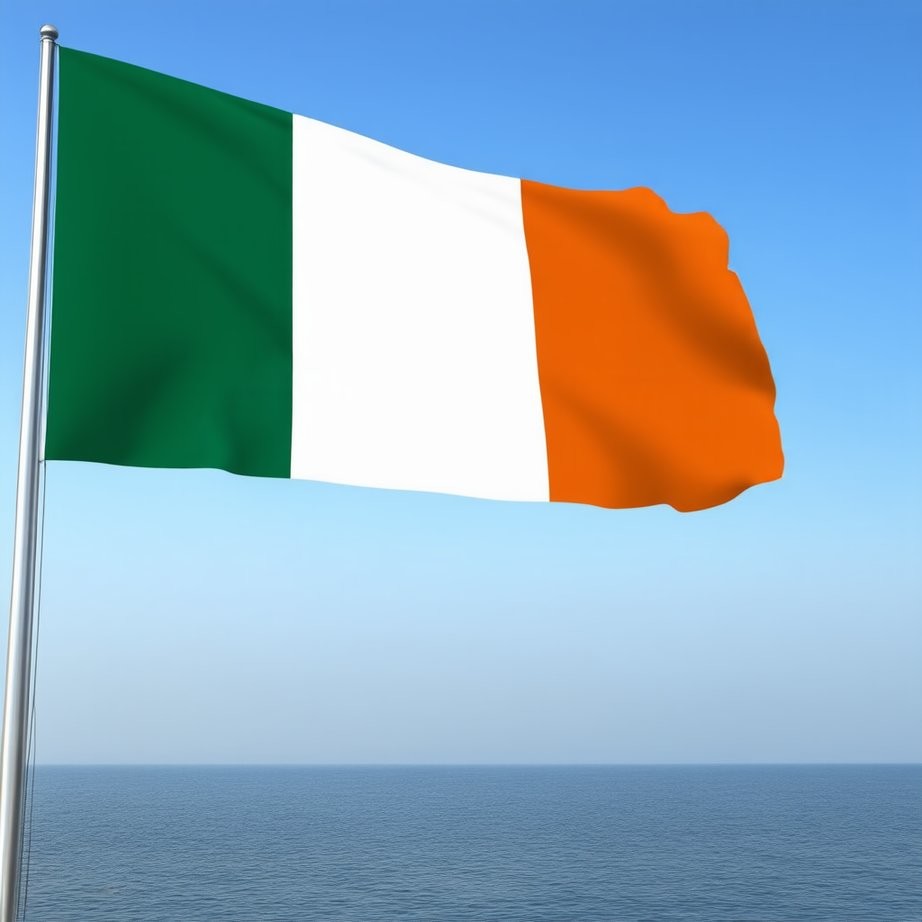} &
        \includegraphics[width=0.1\textwidth]{images/main/sd3_figure_down/Ireland/target_ours.jpg} \\
        \midrule

        \multirow{3}{*}{\parbox[c][2.6cm][c]{0.5cm}{\centering\rotatebox{90}{\textbf{To preserve}}}} &
        \raisebox{2.6ex}{\rotatebox[origin=l]{90}{\small\emph{``Xerxes''}}} &
        \includegraphics[width=0.1\textwidth]{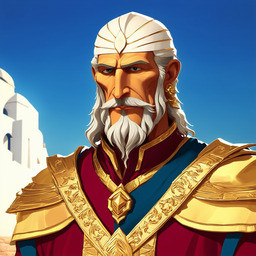} &
        \includegraphics[width=0.1\textwidth]{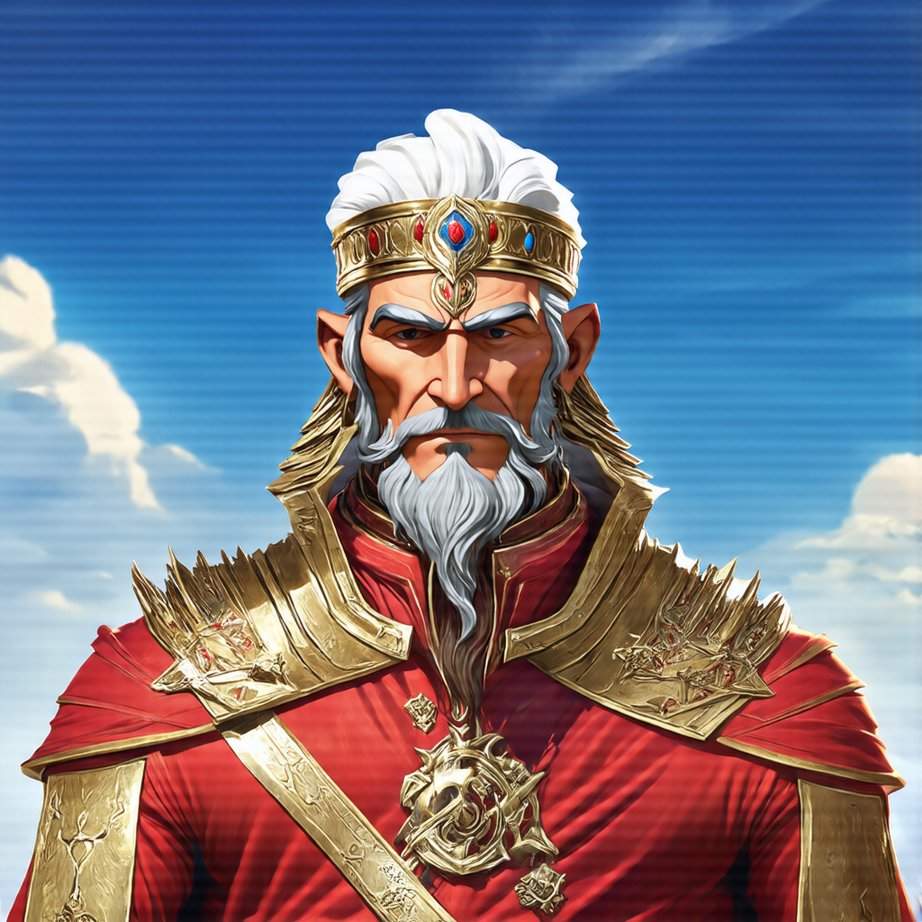} &
        \includegraphics[width=0.1\textwidth]{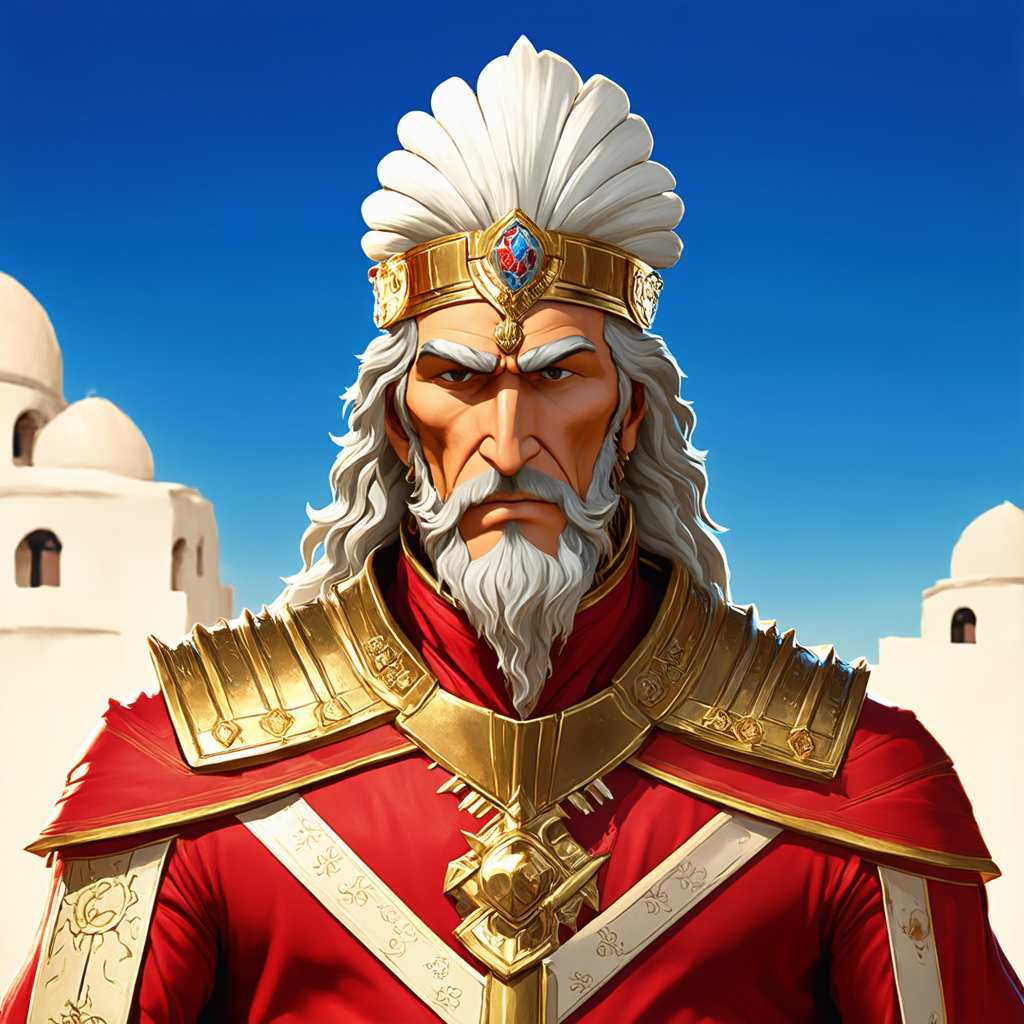} &
        \includegraphics[width=0.1\textwidth]{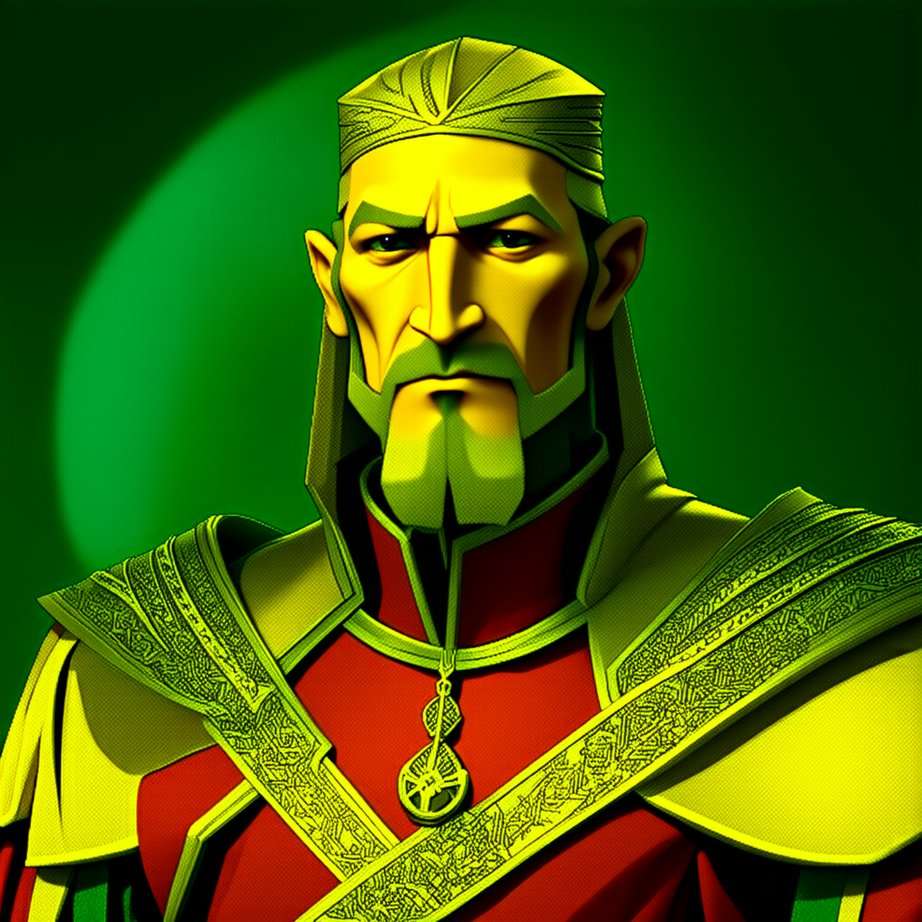} &
        \includegraphics[width=0.1\textwidth]{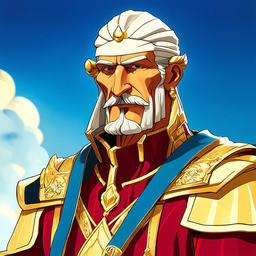} &
        \includegraphics[width=0.1\textwidth]{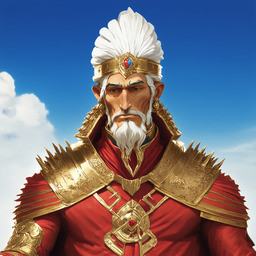} &
        \includegraphics[width=0.1\textwidth]{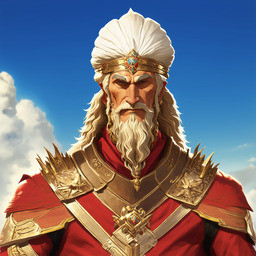} &
        \includegraphics[width=0.1\textwidth]{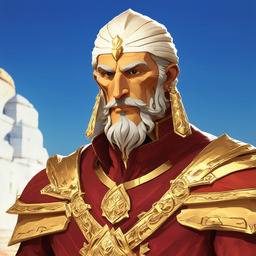} \\

        & \raisebox{0.6ex}{\rotatebox[origin=l]{90}{\small\emph{``Japan flag''}}} &
        \includegraphics[width=0.1\textwidth]{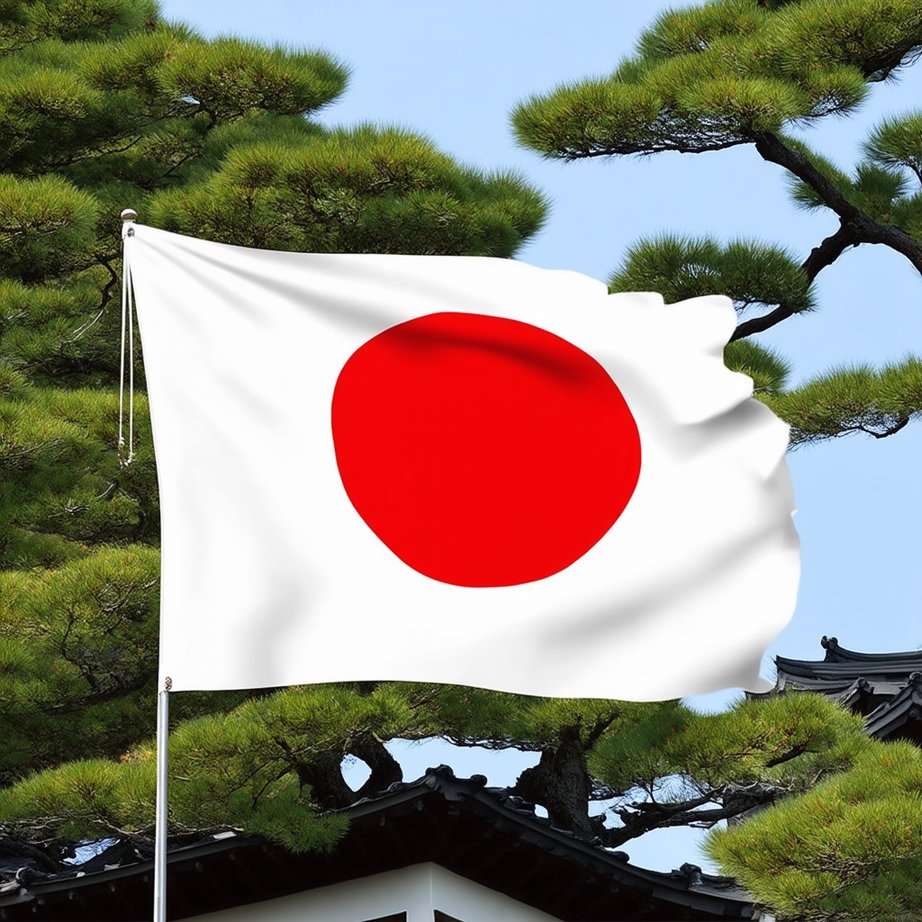} &
        \includegraphics[width=0.1\textwidth]{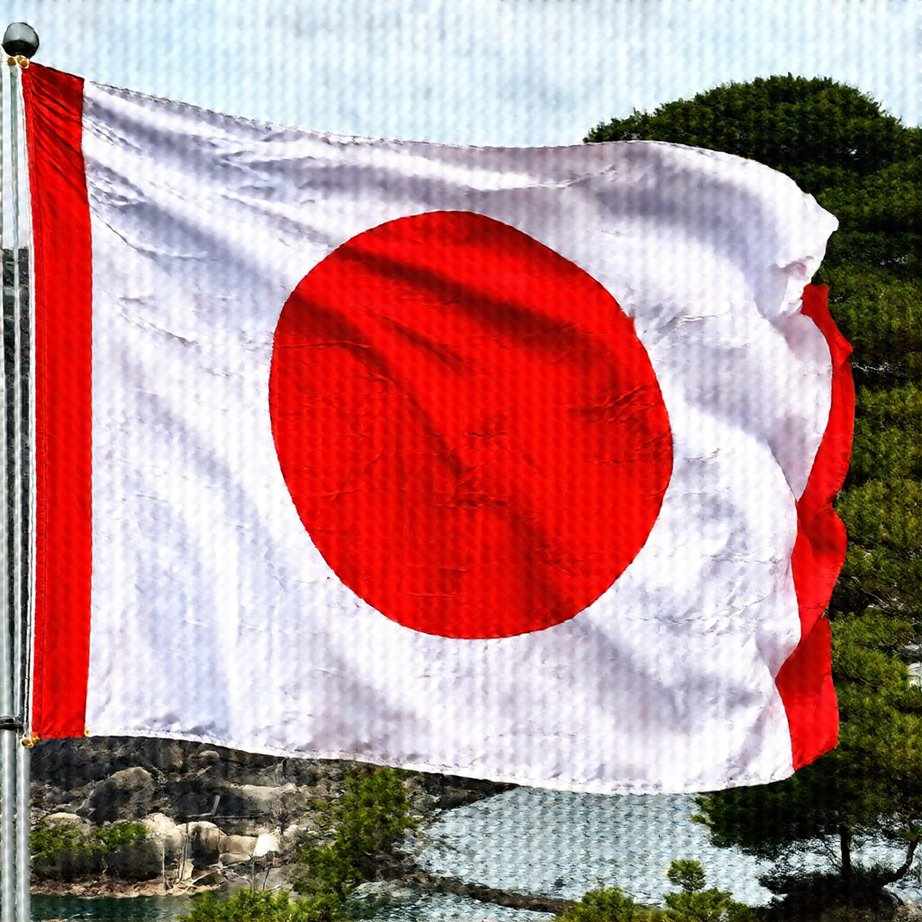} &
        \includegraphics[width=0.1\textwidth]{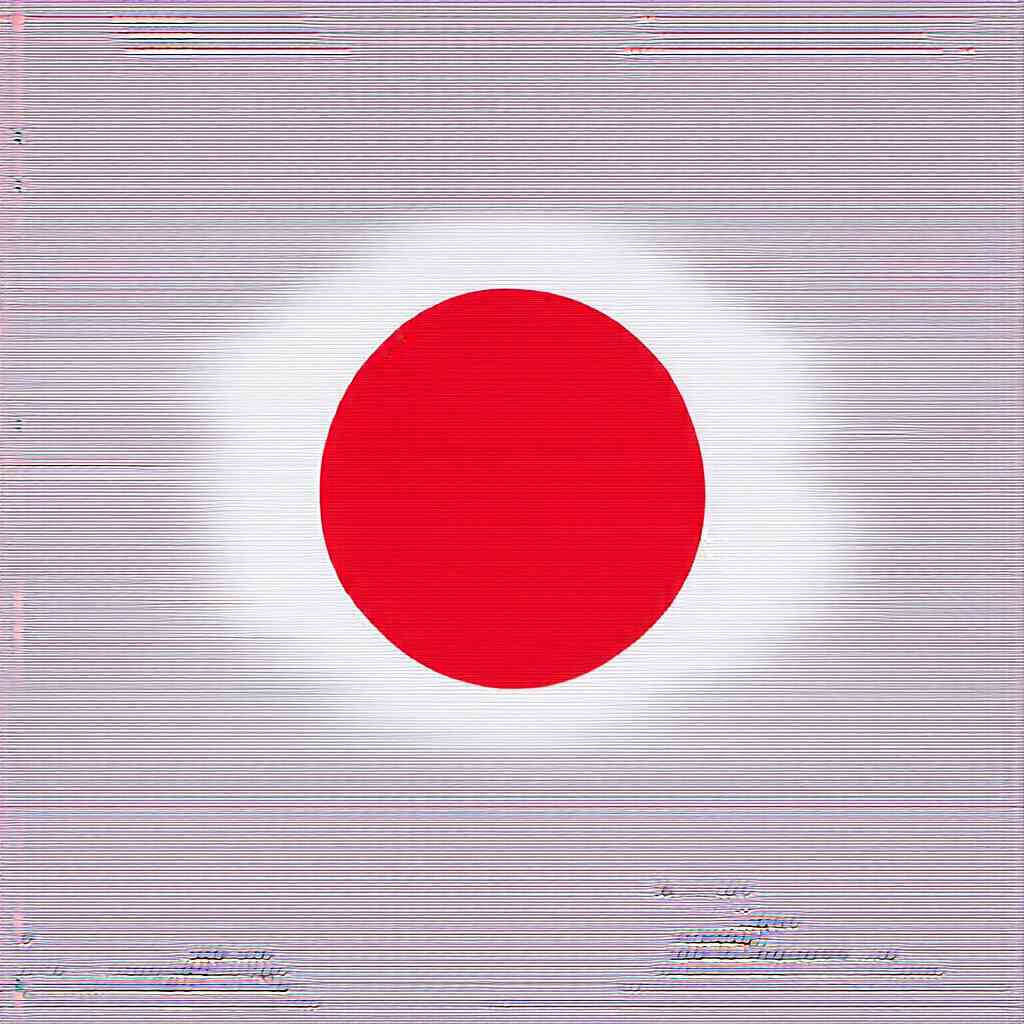} &
        \includegraphics[width=0.1\textwidth]{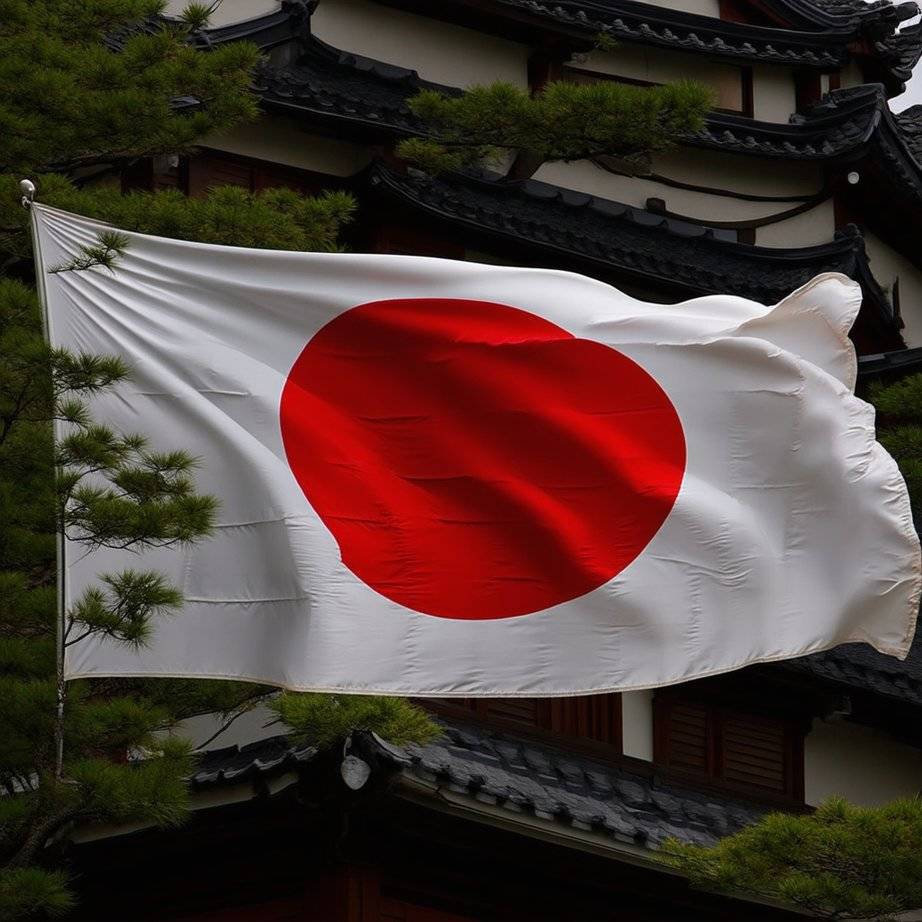} &
        \includegraphics[width=0.1\textwidth]{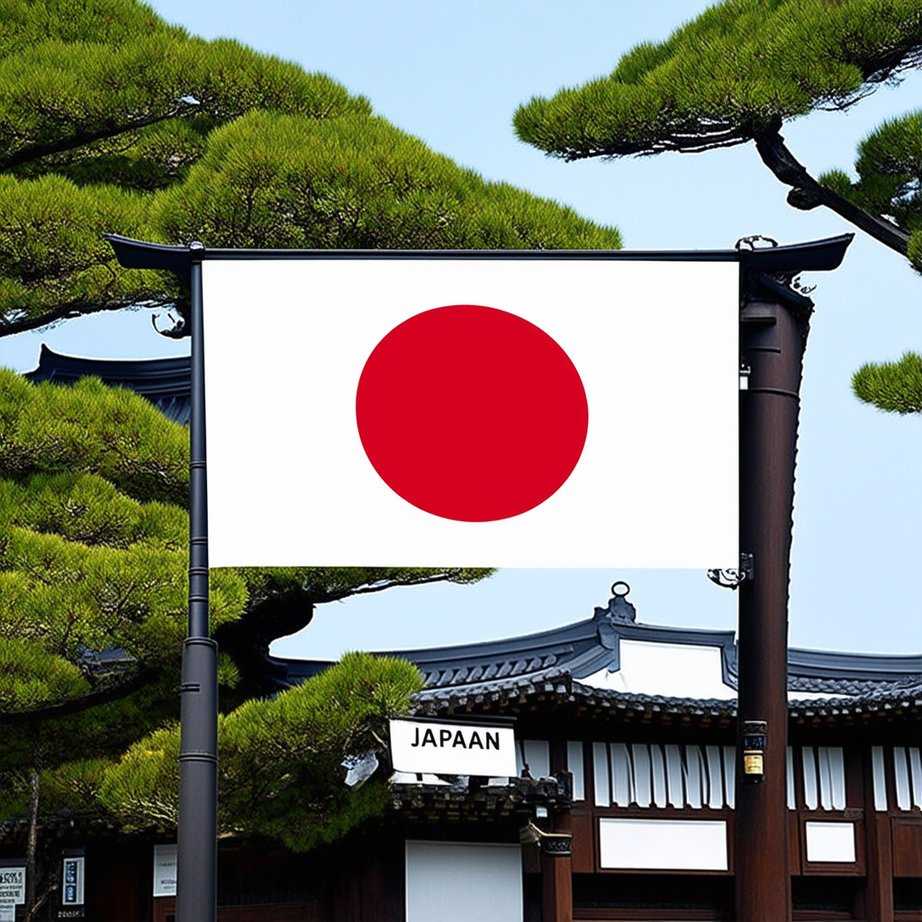} &
        \includegraphics[width=0.1\textwidth]{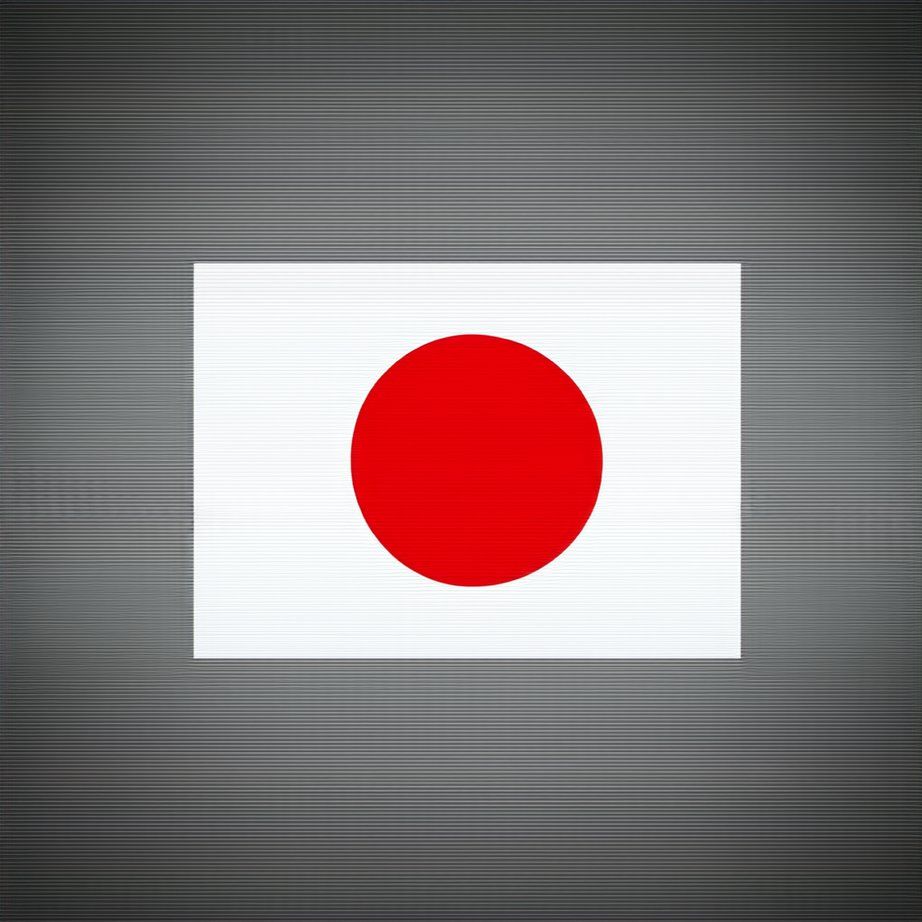} &
        \includegraphics[width=0.1\textwidth]{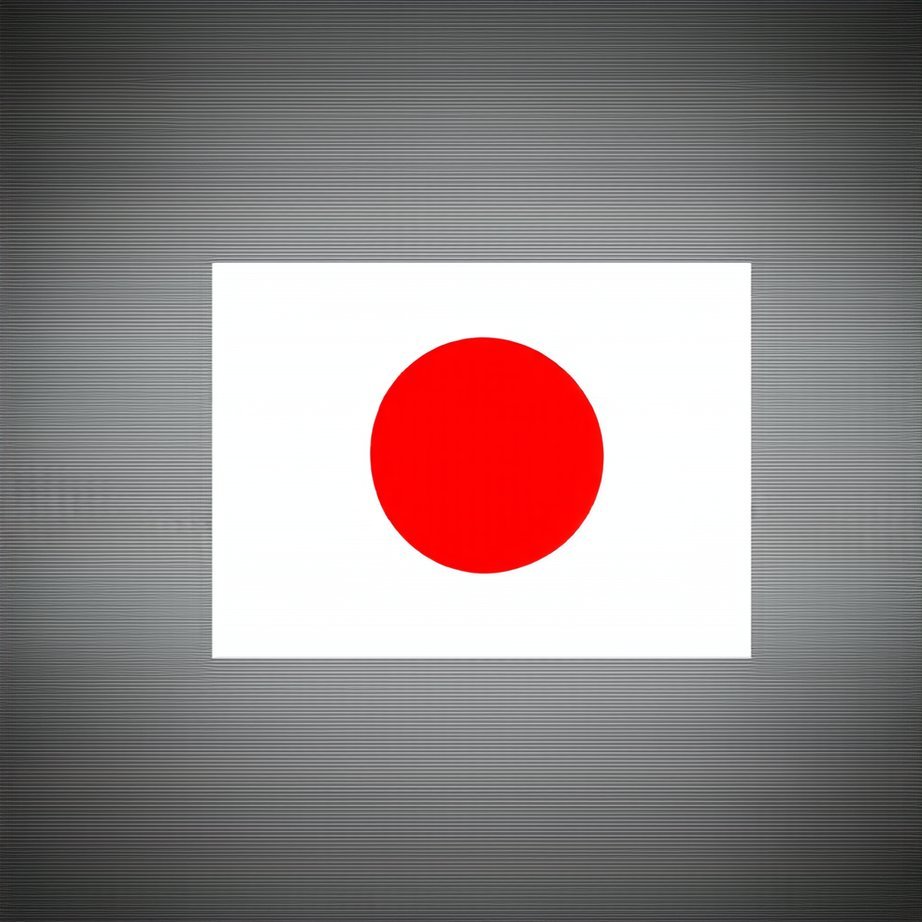} &
        \includegraphics[width=0.1\textwidth]{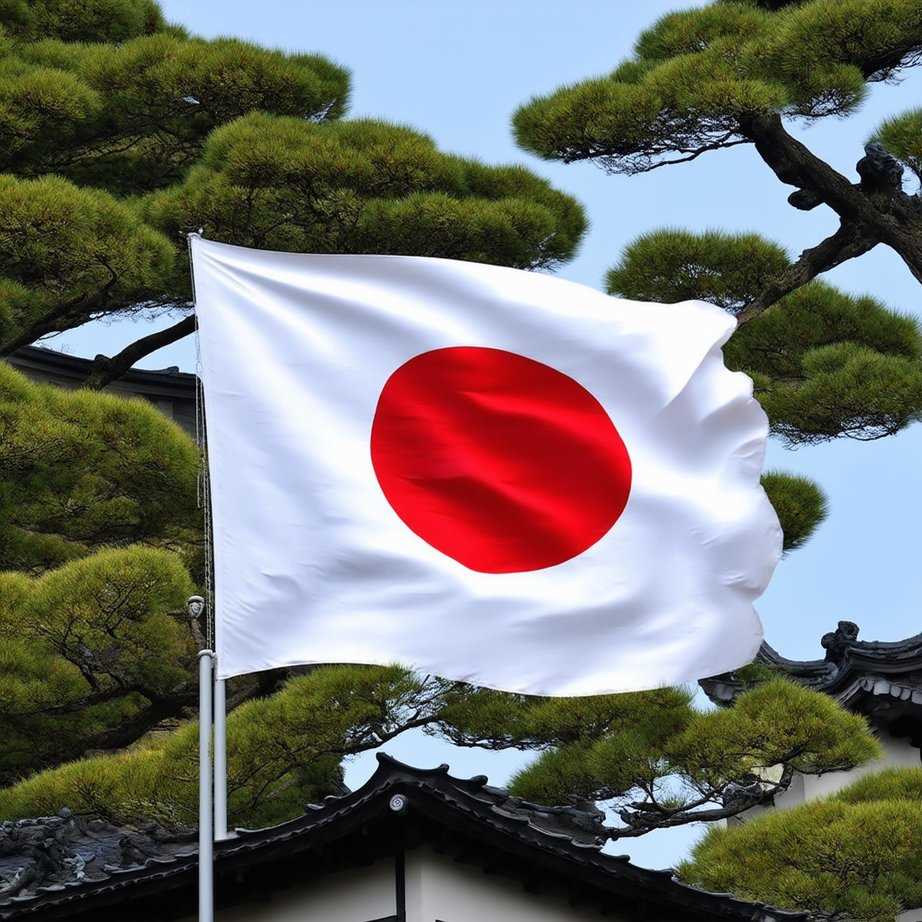} \\

        & \raisebox{0.3ex}{\rotatebox[origin=l]{90}{\small\emph{``Ireland flag''}}} &
        \includegraphics[width=0.1\textwidth]{images/main/sd3_figure_down/Ireland/445_pretrained.jpg} &
        \includegraphics[width=0.1\textwidth]{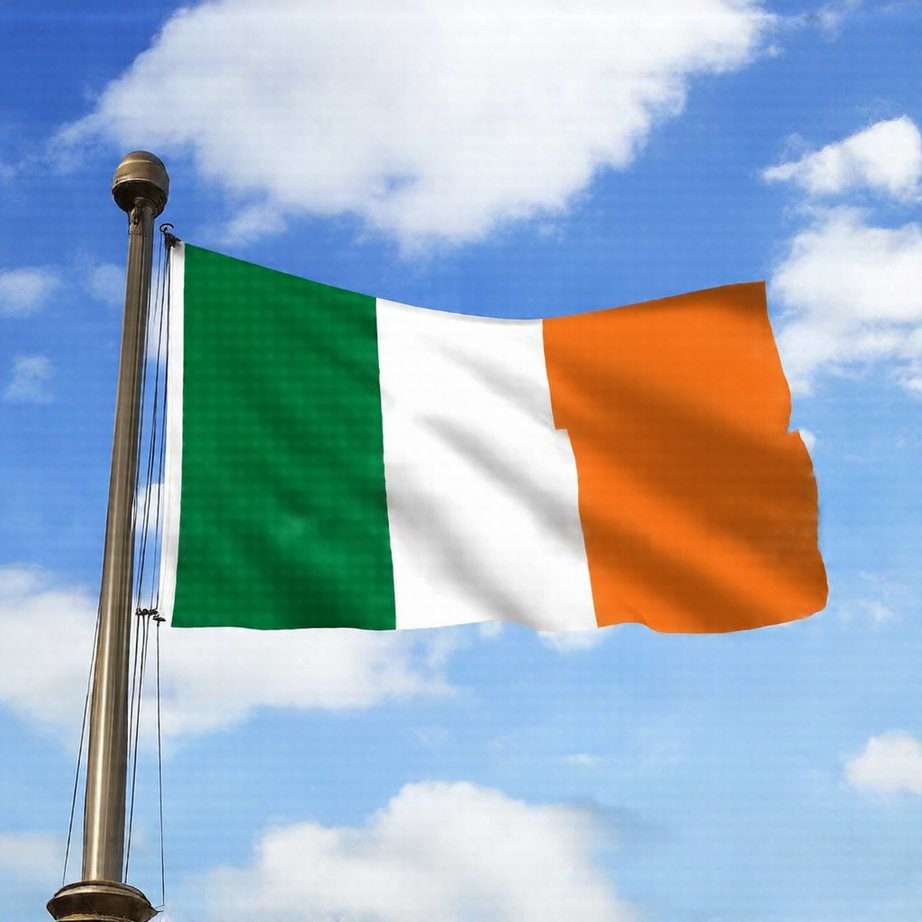} &
        \includegraphics[width=0.1\textwidth]{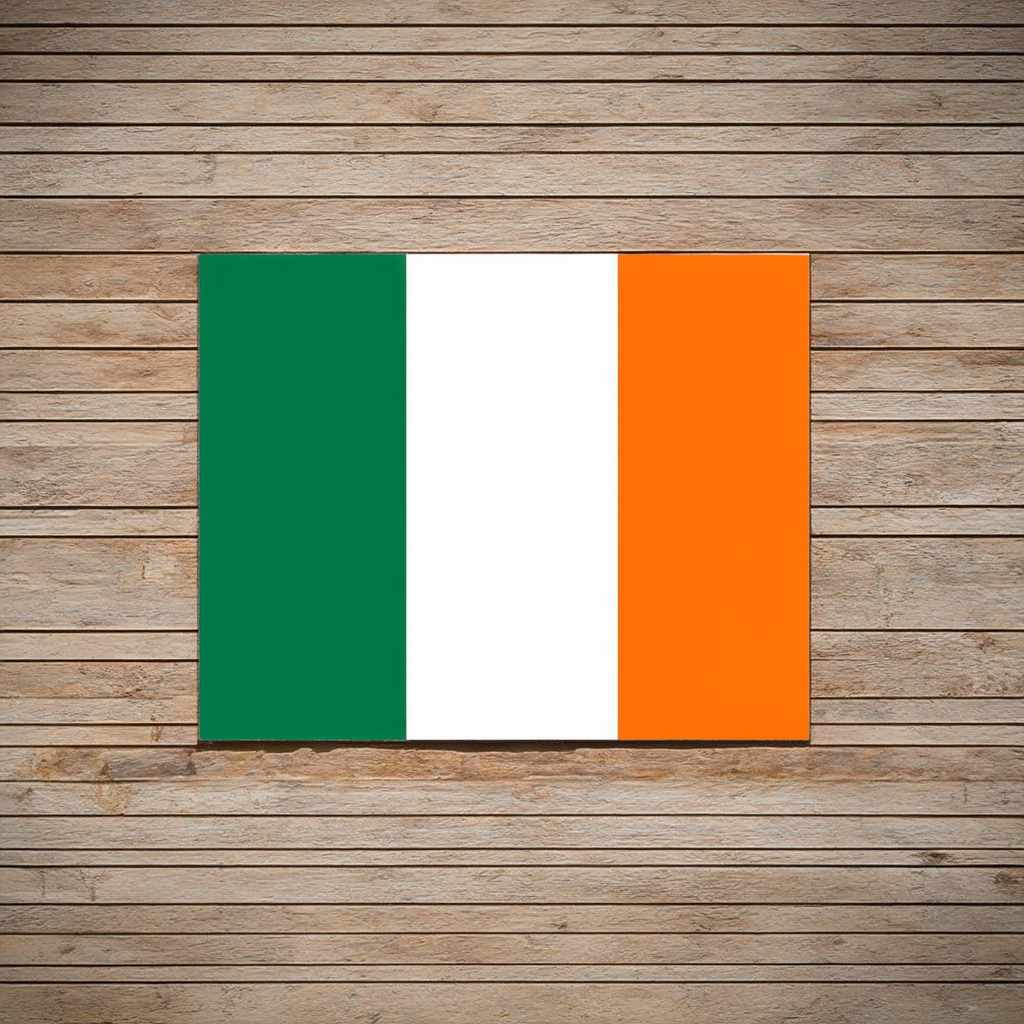} &
        \includegraphics[width=0.1\textwidth]{images/main/sd3_figure_down/Ireland/445_duo.jpg} &
        \includegraphics[width=0.1\textwidth]{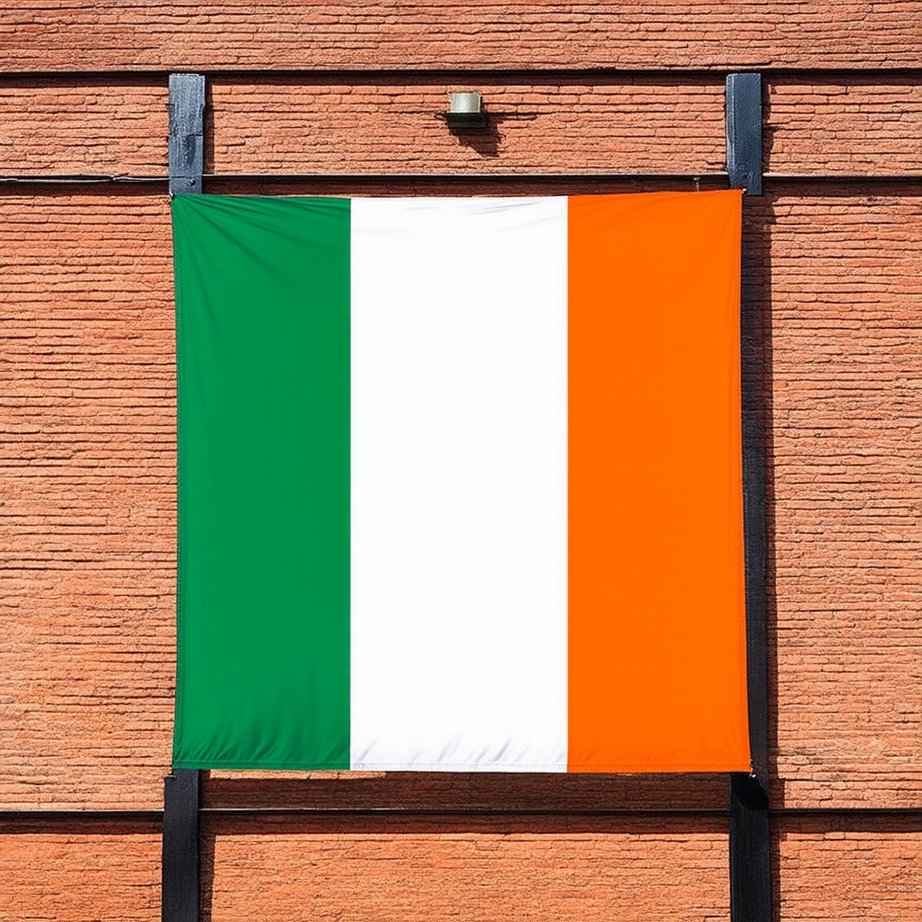} &
        \includegraphics[width=0.1\textwidth]{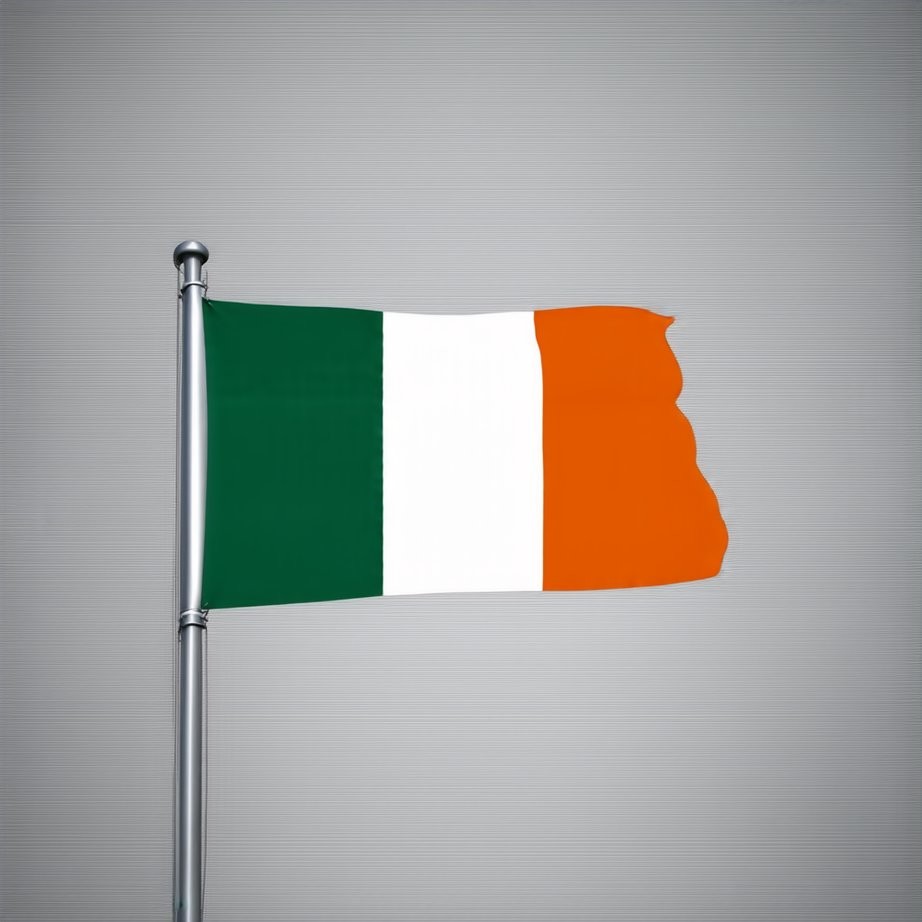} &
        \includegraphics[width=0.1\textwidth]{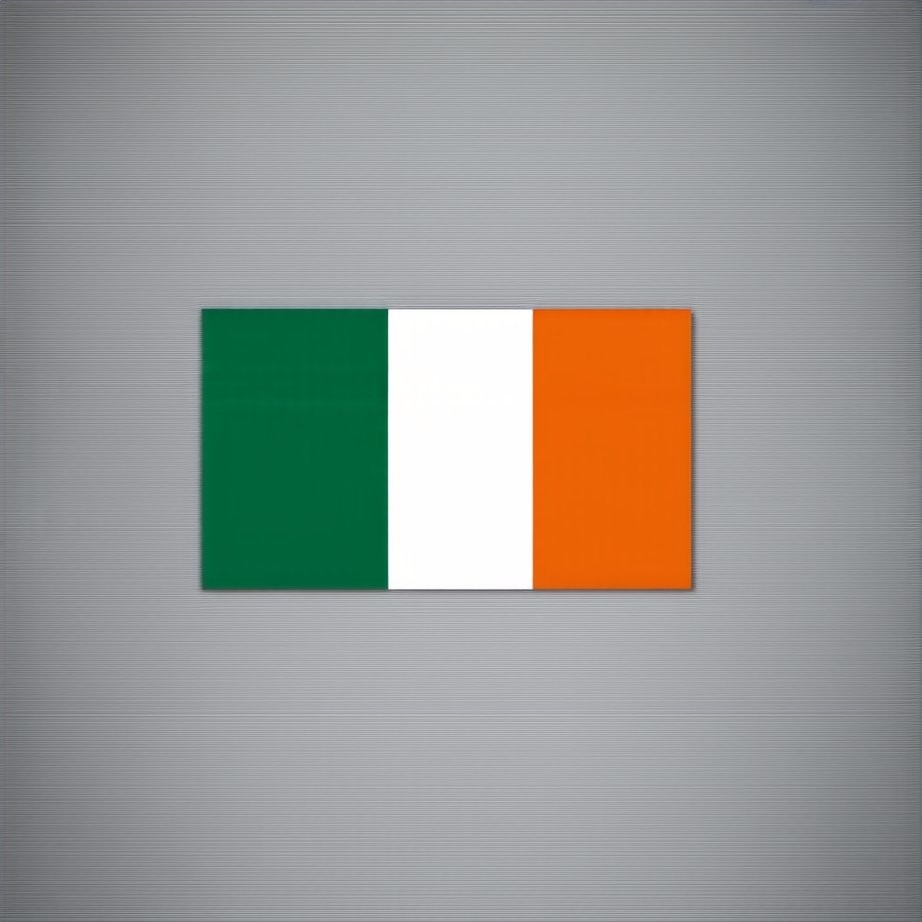} &
        \includegraphics[width=0.1\textwidth]{images/main/sd3_figure_down/Ireland/445_ours.jpg} \\
        \bottomrule
    \end{tabular}
    \caption{Comparison of instance unlearning methods on (top) forgetting target instance and (bottom) preserving related instances.}
    \label{fig:sd3}
\end{figure*}

\noindent\textbf{Single instance unlearning. }
\label{subsec:exp_ddpm_one}
We show that our approach effectively unlearns the target image while preserving model integrity in generated outputs. As shown in Fig. \ref{fig:quali_forgetting}, prior methods such as NegGrad~\cite{golatkar2020eternal} and EraseDiff~\cite{wu2024erasediff} produce artifacts during SSCD computation. In contrast, our method generates natural-looking images that still differ from the target, demonstrating superior model integrity. Furthermore, we evaluate the methods in terms of model integrity by comparing the outputs generated by the pretrained model with those generated by the unlearned model using the same initial noise. As illustrated in Fig. \ref{fig:quali_integrity}, our approach produces outputs nearly identical to the pretrained model's when using the same seed (unrelated to $\mathcal{D}_f$), indicating higher model integrity (Eq. \ref{eq:model_integrity}). SISS~\cite{silas2024data} did not significantly degrade quality but changed features such as hair or glasses.
As shown in Table~\ref{tab:celeba}, all methods achieve forgetting, though some baselines introduce distribution shifts or degrade output quality.
In contrast, our method balances forgetting and model performance, standing out in various quantitative evaluations with the highest overall integrity and image quality.

\noindent\textbf{Multiple instance unlearning. }
\label{subsec:exp_ddpm_multi}
Sequentially unlearning multiple celebrities while preserving model integrity is a challenging task. This issue has also received considerable attention in the field of concept erasing~\cite{lu2024mace, lyu2024one}. Therefore, we performed experiments to verify whether the model could maintain generative stability while successfully unlearning four celebrities sequentially. As shown in Table~\ref{tab:celeba}, our method outperforms others in terms of model integrity.

\subsection{Correcting Unpromptable Misrepresentations through Instance Unlearning}
\label{subsec:exp_sd3}
The issue of misrepresentations for specific instances in recent generative models has been widely recognized as an ethical concern, with previous studies highlighting its presence~\cite{aldahoul2024ai, cheong2024investigating}. We observed that SD3~\cite{esser2024scaling} exhibits undesirable outputs when generating images using specific prompts, as shown in the `Original' column of Fig. \ref{fig:sd3}. Errors involving national figures (\textit{e.g.}, \emph{``Xerxes''}) or flags (\textit{e.g.}, \emph{``Japan flag'', ``Ireland flag''}) are occasional, yet such errors are generally unpromptable. Fig. \ref{fig:sd3} demonstrates the results for applying our instance unlearning method to such a scenario. The top rows, labeled `To forget,' clearly illustrate that our method effectively forgets specific undesirable representations. Conversely, the bottom rows, labeled `To preserve,' confirm that our method maintains the same output before unlearning, \textit{i.e.,} model integrity.

\section{Discussion}
\label{sec:discussion}

\begin{figure}[t]
    \centering
    \subfloat[Target]{\includegraphics[width=0.11\textwidth]{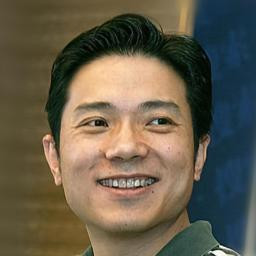}}
    \subfloat[0.5]{\includegraphics[width=0.11\textwidth]{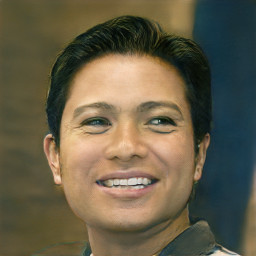}}
    \subfloat[1.0]{\includegraphics[width=0.11\textwidth]{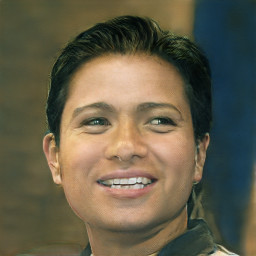}}
    \subfloat[2.0]{\includegraphics[width=0.11\textwidth]{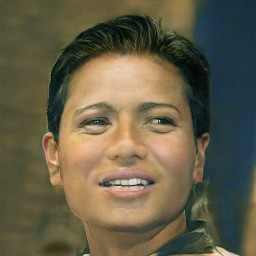}} \\
    \subfloat[4.0]{\includegraphics[width=0.11\textwidth]{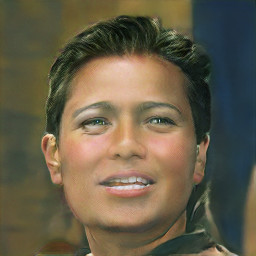}}
    \subfloat[8.0]{\includegraphics[width=0.11\textwidth]{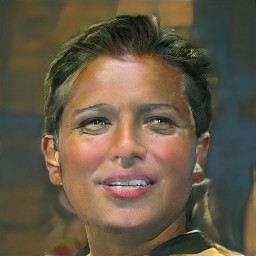}}
    \subfloat[16.0]{\includegraphics[width=0.11\textwidth]{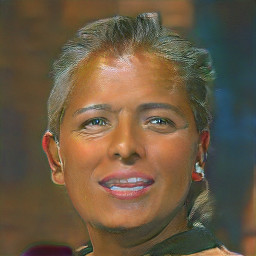}}
    \subfloat[32.0]{\includegraphics[width=0.11\textwidth]{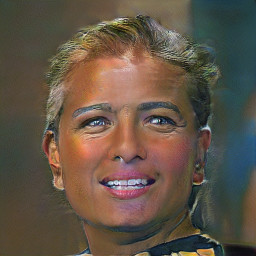}}
    \caption{Edited images (surrogates) with varying CLIP loss weights of TediGAN.}
    \label{fig:tedigan_clip}
\end{figure}

\subsection{Ablation Studies}
\noindent\textbf{Editing scale. }
We first study on how different levels of editing affect the unlearning outcomes. Stronger facial editing more effectively unlearns the targeting face, but simultaneously harm the model's integrity. To address this trade-off, we construct the surrogates by CLIP-guided editing, keeping them close to the original instance while enabling forgetting. The effect of editing with varying CLIP loss weights is visualized in Fig.~\ref{fig:tedigan_clip}.

Edited images are then used as surrogates to compare the unlearning performance. The forgetting-model integrity trade-off is illustrated in Fig.~\ref{fig:forgetting_tradeoff}, and the values are provided on Table.~\ref{tab:tedigan_clip}. Our method achieves a Pareto front: it retains higher model integrity while satisfying the forgetting threshold (vertical red dashed line).


\begin{figure*}[t]  
    \centering
    \begin{minipage}[t]{0.37\textwidth}
        \centering
        \begin{tikzpicture}
            \begin{axis}[
                xlabel={SSCD (forgetting)},
                ylabel={SSIM$\uparrow$ (model integrity)},
                xlabel style={font=\sffamily\small},
                ylabel style={font=\sffamily\small},
                xticklabel style={font=\sffamily\small},
                yticklabel style={font=\sffamily\small},
                ymin=0.8,
                ymax=0.95,
                width=7cm,
                height=5cm,
                grid=both,
                grid style={dashed,gray!30},
            ]
            
            \addplot+[
                mark=*,
                color=blue,
                error bars/.cd,
                y dir=both,
                y explicit
            ] coordinates {
                (0.433533333, 0.899516667) +- (0, 0.004517847)
                (0.402633333, 0.900616667) +- (0, 0.004530667)
                (0.329733333, 0.899633333) +- (0, 0.004602656)
                (0.30155, 0.896483333) +- (0, 0.004434517)
                (0.2777, 0.8907) +- (0, 0.003369273)
                (0.22815, 0.8847) +- (0, 0.003136771)
                (0.194383333, 0.872233333) +- (0, 0.004594465)
            };
            \addplot+[
                only marks,
                mark=square*,
                mark options={draw=black, fill=white, thick},
                color=black,
                error bars/.cd,
                y dir=both,
                y explicit,
                error bar style={black}
            ] coordinates {
                (0.210866667, 0.839066667) +- (0, 0.012451551)
            };
            \addplot+[
                only marks,
                mark=triangle*,
                mark options={draw=cyan, fill=white, thick},
                color=cyan,
                error bars/.cd,
                y dir=both,
                y explicit,
                error bar style={cyan}
            ] coordinates {
                (0.169333333, 0.835616667) +- (0, 0.007668004)
            };
            
            \addplot[red, thick, dashed] coordinates {(0.4, 0) (0.4, 1)};
            
            \node[font=\sffamily\small, align=center, anchor=center, xshift=2pt, yshift=6pt] at (axis cs:0.4, 0.85) {Forgotten?};
            
            \draw[->, red, thick] (axis cs:0.4, 0.85) -- (axis cs:0.38, 0.85);
            \node[font=\sffamily\small, align=right, anchor=east, yshift=-6pt] at (axis cs:0.4, 0.85) {Yes};
            
            \draw[->, red, thick] (axis cs:0.4, 0.85) -- (axis cs:0.42, 0.85);
            \node[font=\sffamily\small, align=left, anchor=west, yshift=-6pt] at (axis cs:0.4, 0.85) {No};
            
            \node[align=left, below right, font=\sffamily\small] at (axis cs:0.14,0.95) {We seek\\top \& left};
            \end{axis}
            \end{tikzpicture}
    \end{minipage}
    \hspace{0.2cm}
    \begin{minipage}[t]{0.37\textwidth}
        \centering
        \begin{tikzpicture}
            \begin{axis}[
                xlabel={SSCD (forgetting)},
                ylabel={LPIPS$\downarrow$ (model integrity)},
                xlabel style={font=\sffamily\small},
                ylabel style={font=\sffamily\small},
                xticklabel style={font=\sffamily\small},
                yticklabel style={font=\sffamily\small},
                ymin=0.2,
                ymax=0.55,
                width=7cm,
                height=5cm,
                grid=both,
                grid style={dashed,gray!30},
            ]
            
            \addplot+[
                mark=*,
                color=blue,
                error bars/.cd,
                y dir=both,
                y explicit
            ] coordinates {
                (0.433533333, 0.304783333) +- (0, 0.024082627)
                (0.402633333, 0.303666667) +- (0, 0.024135170)
                (0.329733333, 0.30665)     +- (0, 0.024007537)
                (0.30155, 0.315283333)     +- (0, 0.026369849)
                (0.2777, 0.324666667)      +- (0, 0.021733042)
                (0.22815, 0.337833333)     +- (0, 0.020917611)
                (0.194383333, 0.351633333) +- (0, 0.018836926)
            };
            \addplot+[
                only marks,
                mark=square*,
                mark options={draw=black, fill=white, thick},
                color=black,
                error bars/.cd,
                y dir=both,
                y explicit,
                error bar style={black}
            ] coordinates {
                (0.210866667, 0.43035) +- (0, 0.030681555)
            };
            \addplot+[
                only marks,
                mark=triangle*,
                mark options={draw=cyan, fill=white, thick},
                color=cyan,
                error bars/.cd,
                y dir=both,
                y explicit,
                error bar style={cyan}
            ] coordinates {
                (0.169333333, 0.4133) +- (0, 0.043065888)
            };
            
            \addplot[red, thick, dashed] coordinates {(0.4, 0) (0.4, 1)};
            
            \node[font=\sffamily\small, align=center, anchor=center, xshift=2pt, yshift=6pt] at (axis cs:0.4, 0.45) {Forgotten?};
            
            \draw[->, red, thick] (axis cs:0.4, 0.45) -- (axis cs:0.38, 0.45);
            \node[font=\sffamily\small, align=right, anchor=east, yshift=-6pt] at (axis cs:0.4, 0.45) {Yes};
            
            \draw[->, red, thick] (axis cs:0.4, 0.45) -- (axis cs:0.42, 0.45);
            \node[font=\sffamily\small, align=left, anchor=west, yshift=-6pt] at (axis cs:0.4, 0.45) {No};
            
            \node[align=left, below right, font=\sffamily\small] at (axis cs:0.14,0.29) {We seek\\bottom \& left};
            \end{axis}
        \end{tikzpicture}
    \end{minipage}
    \hspace{0.3cm}
    \begin{minipage}[t]{0.17\textwidth}
    \centering
    \begin{tikzpicture}
    \begin{axis}[
        hide axis,
        xmin=0, xmax=1,
        ymin=0, ymax=1,
        scale only axis,
        legend columns=1,
        legend cell align={left},
        legend style={
            at={(0,1.55)},  
            anchor=north east,
            draw=black,
            fill=white,
            rounded corners,
            font=\sffamily\small,
            row sep=5pt
        }
    ]
    \addlegendimage{only marks, blue, mark=*}
    \addlegendentry{Ours}
    \addlegendimage{only marks, mark=triangle*, draw=cyan, fill=white, thick}
    \addlegendentry{\small\begin{tabular}{@{}l@{}}Ours w/o \\ surrogate\\\end{tabular}} 
    \addlegendimage{only marks, mark=square*, draw=black, fill=white, thick}
    \addlegendentry{\small\begin{tabular}{@{}l@{}}SISS~\cite{silas2024data}\end{tabular} }
    \addlegendimage{red, thick, dashed}
    \addlegendentry{\small\begin{tabular}{@{}l@{}}Forgetting \\ threshold\\\end{tabular}} 
    \end{axis}
    \end{tikzpicture}
    \end{minipage}
    \caption{Trade-off between forgetting (SSCD) and model integrity (SSIM, LPIPS). We can select how the surrogate is close to the original instance by adjusting CLIP loss weight during editing.}
    \label{fig:forgetting_tradeoff}
\end{figure*}
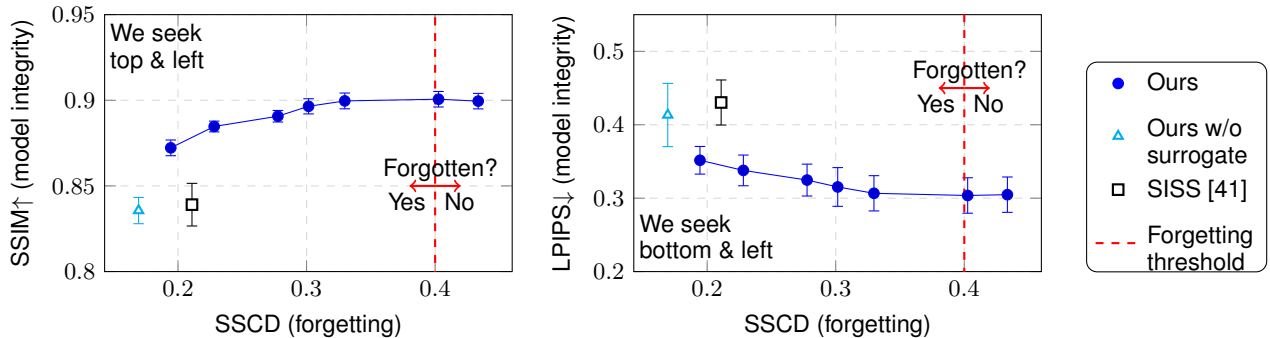

\begin{table}[t]
    \caption{Quantitative results with varying CLIP loss weights of TediGAN.}
    \label{tab:tedigan_clip}
    \centering
    \normalsize
    \setlength{\tabcolsep}{4.5pt}
    \begin{tabular}{@{}c|ccccc@{}}
        \toprule
        CLIP weight & SSCD & LPIPS$\downarrow$ & SSIM$\uparrow$ & FID$_{\text{pre}}$$\downarrow$ & FID$_{\text{real}}$$\downarrow$ \\
        \midrule
        0.5  & 0.434 & 0.305 & 0.900 & 8.08 & 16.25 \\
        1.0  & 0.403 & 0.304 & 0.901 & 8.04 & 16.33 \\
        2.0  & 0.330 & 0.307 & 0.900 & 8.10 & 16.62 \\
        4.0  & 0.302 & 0.315 & 0.896 & 8.34 & 17.09 \\
        8.0  & 0.278 & 0.325 & 0.891 & 8.80 & 17.62 \\
        16.0 & 0.228 & 0.338 & 0.885 & 8.94 & 17.48 \\
        32.0 & 0.194 & 0.352 & 0.872 & 9.83 & 18.03 \\
        \bottomrule
    \end{tabular}
\end{table}

\noindent\textbf{Timestep-aware weighting. }
Also, we conducted an ablation study testing various $\lambda$ for the loss design under two conditions: 1) constant $\lambda$, and 2) timestep-aware $\lambda$ ($= 1- \beta t$).
Experiment is conducted on a single instance and the metrics are evaluated using 10,000 generations.
As shown in Table ~\ref{tab:ablation_timestep}, with $\lambda = 1- \beta t$ with $\beta = 5 \times 10^{-5}$, the model achieved the best overall performance.

\begin{table}[t]
\caption{Ablation on timestep scheduling of $\lambda$.}
\label{tab:ablation_timestep}
\begin{center}
\setlength{\tabcolsep}{3pt}  
\begin{normalsize}
\begin{tabular}{@{}l|ccccc@{}}
\toprule
Condition & SSCD  & LPIPS$\downarrow$   & SSIM$\uparrow$ & $\text{FID}_{pre}$$\downarrow$   & $\text{FID}_{real}$$\downarrow$ \\
\midrule
$\lambda=\text{0.95}$ & 0.279 & 0.174          & 0.897 & 7.82 & 24.06 \\
$\lambda=\text{0.99}$ & 0.289          & 0.173 & 0.894 & 7.18 & 20.31 \\
\midrule
$\lambda = \text{1} \!-\! \text{2}\!\cdot\!\text{10}^{-\text{5}}t $           & 0.356          & 0.213  & 0.920 & 6.06 & 15.43 \\
$\lambda = \text{1} \!-\! \text{5}\!\cdot\!\text{10}^{-\text{5}}t$ & 0.355          & 0.213  & 0.920 & 6.06 & 15.43 \\
$\lambda = \text{1} \!-\! \text{10}^{-\text{4}}t $           & 0.311          & 0.235  & 0.910          & 6.59 & 15.58 \\
\bottomrule
\end{tabular}
\end{normalsize}
\end{center}
\end{table}

\begin{table}[t]
\caption{Ablation on gradient projection.}
\label{tab:ablation_gradient}
\begin{center}
\setlength{\tabcolsep}{3pt}
\begin{normalsize}
\begin{tabular}{l|ccccc}
\toprule
Condition & SSCD & LPIPS$\downarrow$   & SSIM$\uparrow$ & $\text{FID}_{pre}$$\downarrow$   & $\text{FID}_{real}$$\downarrow$ \\
\midrule
no projection                       & 0.309 & 0.275          & 0.899          & 7.68          & 14.72 \\
project $g_f$ & 0.309 & 0.274 & 0.900          & 7.65 & 14.70 \\
project $g_f$, $g_r$           & 0.310          & 0.276          & 0.901 & 8.15          & 16.47 \\
\bottomrule
\end{tabular}
\end{normalsize}
\end{center}
\end{table}

\noindent\textbf{Gradient surgery. }
We evaluated the effectiveness of using gradient surgery by testing on the variants to validate the effectiveness of our method. As shown in Table~\ref{tab:ablation_gradient}, the PCGrad~\cite{yu2020gradient} method, which projects both the remembering gradient and the forgetting gradient, was not suitable for our unlearning method as it led to a decrease in model quality. Projecting only the forgetting gradient demonstrated the best overall performance.

\begin{table}[t]
\caption{Ablation on surrogate construction methods.}
\label{tab:ablation_surrogate}
\centering
\begin{normalsize}
\begin{tabular}{lccc}
\toprule
Surrogate method & SSCD$<$0.4 & LPIPS$\downarrow$ & SSIM$\uparrow$ \\
\midrule
Edit (TediGAN) & \checkmark (0.35) & \textbf{0.21} & \textbf{0.92} \\
Flip & \checkmark (0.40) & 0.31 & 0.85 \\
Add noise ($\sigma=50$) & \ding{55} (0.54) & 0.23 & 0.91 \\
Add noise ($\sigma=150$) & \checkmark (0.32) & 0.28 & 0.89 \\
\bottomrule
\end{tabular}
\end{normalsize}
\end{table}

\subsection{Validation of Using an Edited Image as a Surrogate}
\label{subsec:discuss_weight_diff}
Surrogate construction plays a critical role in our unlearning framework. We compare several strategies for generating surrogate data, including flipping the original image and adding Gaussian noise at varying levels, as shown in Table~\ref{tab:ablation_surrogate}. These simple manipulations lead to degraded surrogate fidelity, resulting in distribution shifts or weaker preservation, despite some success in SSCD. In contrast, using carefully edited images (e.g., via TediGAN) provides high-fidelity surrogates that enable effective forgetting while maintaining model integrity. Fig.~\ref{fig:surrogate_comparison} illustrates that edited surrogates forget the target identity while better preserving unrelated concepts compared to other surrogates. These findings emphasize the importance of surrogate quality: low-fidelity surrogates introduce unnecessary perturbations, while edited surrogates support targeted unlearning with less side effects.

\begin{figure}[t!]
    \centering
    \small
    \setlength{\tabcolsep}{2pt}
    \begin{tabular}{ccccc}
        \toprule
        To forget & Surrogate & Forgotten? & To preserve & Preserved? \\
        \midrule
        \includegraphics[width=0.18\linewidth]{images/main/edit_example/00000.jpg} &
        \begin{overpic}[width=0.18\linewidth]{images/main/edit_example/00000_edit.jpg}
            \put(3,83){\contour{black}{\color{white}\scriptsize \textbf{Edited}}}
        \end{overpic} &
        \includegraphics[width=0.18\linewidth]{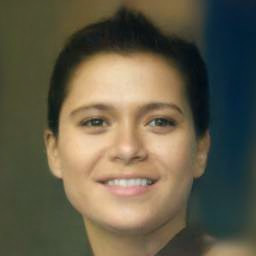} &
        \includegraphics[width=0.18\linewidth]{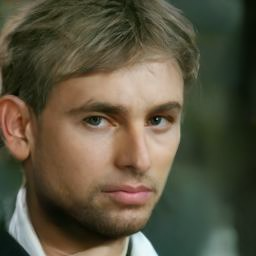} &
        \begin{overpic}[width=0.18\linewidth]{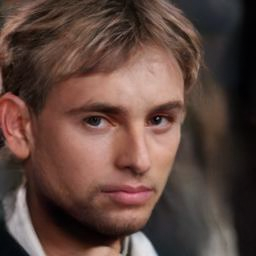}
            \put(3,78){\contour{black}{\textcolor{green!50!black}{\textbf{\Large\checkmark}}}}
        \end{overpic} \\
        
        \includegraphics[width=0.18\linewidth]{images/main/edit_example/00000.jpg} &
        \begin{overpic}[width=0.18\linewidth]{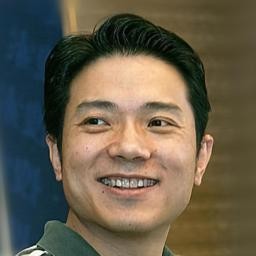}
            \put(3,83){\contour{black}{\color{white}\scriptsize \textbf{Flipped}}}
        \end{overpic} &
        \includegraphics[width=0.18\linewidth]{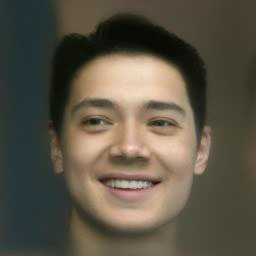} &
        \includegraphics[width=0.18\linewidth]{images/main/ablation_surrogate/0_pretrained.jpg} &
        \begin{overpic}[width=0.18\linewidth]{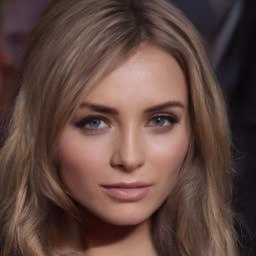}
            \put(4,81){\contour{black}{\textcolor{red}{\textbf{\large\ding{55}}}}}
        \end{overpic} \\
        
        \includegraphics[width=0.18\linewidth]{images/main/edit_example/00000.jpg} &
        \begin{overpic}[width=0.18\linewidth]{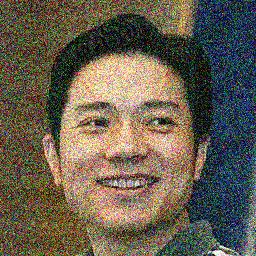}
            \put(3,83){\contour{black}{\color{white}\scriptsize \textbf{$\sigma=50$}}}
        \end{overpic} &
        \includegraphics[width=0.18\linewidth]{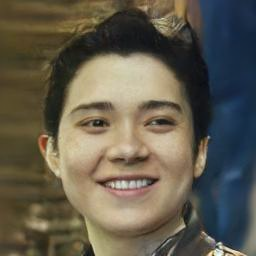} &
        \includegraphics[width=0.18\linewidth]{images/main/ablation_surrogate/0_pretrained.jpg} &
        \begin{overpic}[width=0.18\linewidth]{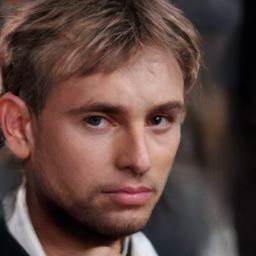}
            \put(3,78){\contour{black}{\textcolor{green!50!black}{\textbf{\Large\checkmark}}}}
        \end{overpic} \\
        
        \includegraphics[width=0.18\linewidth]{images/main/edit_example/00000.jpg} &
        \begin{overpic}[width=0.18\linewidth]{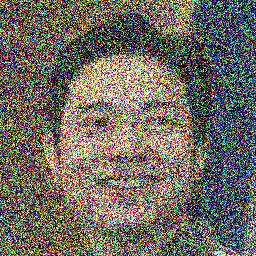}
            \put(3,83){\contour{black}{\color{white}\scriptsize \textbf{$\sigma=150$}}}
        \end{overpic} &
        \includegraphics[width=0.18\linewidth]{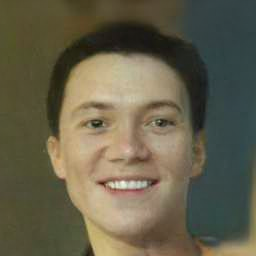} &
        \includegraphics[width=0.18\linewidth]{images/main/ablation_surrogate/0_pretrained.jpg} &
        \begin{overpic}[width=0.18\linewidth]{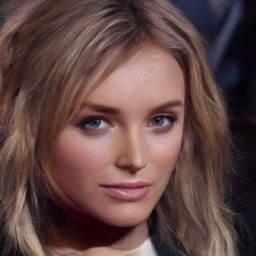}
            \put(4,81){\contour{black}{\textcolor{red}{\textbf{\large\ding{55}}}}}
        \end{overpic} \\
        \bottomrule
    \end{tabular}
    \caption{Qualitative comparison of surrogates. $\sigma=50$ and $\sigma=150$ denote adding Gaussian noise with standard deviations of 50 and 150 to the target image.}
    \label{fig:surrogate_comparison}
\end{figure}

\subsection{Forget out-of-domain data}
The OOD (Out-of-Domain) scenario is necessary to evaluate the domain generalization performance of our framework. By using the FFHQ~\cite{karras2019style} dataset, which differs from the training dataset (CelebA-HQ), we test whether our method can generalize beyond the training domain and effectively unlearn a specific instance. As shown in Table \ref{tab:FFHQ_table}, the OOD case also demonstrated similar model integrity to the InD (In-Domain) case.

\begin{table}[t!]
\caption{Single instance unlearning for InD and OoD data.}
\label{tab:FFHQ_table}
\centering
\setlength{\tabcolsep}{2.5pt}  
\begin{normalsize}
\begin{tabular}{lccccc}
\toprule
            & SSCD  & LPIPS$\downarrow$   & SSIM$\uparrow$ & $\text{FID}_{pre}\downarrow$   & $\text{FID}_{real}\downarrow$ \\
\midrule    
Ours (InD)           & 0.34  & 0.30  & 0.90  & 8.7 & 16.6 \\
Ours (OoD)           & 0.32  & 0.31  & 0.90  & 9.0 & 17.1 \\
\bottomrule
\end{tabular}
\end{normalsize}
\end{table}


\subsection{Limitations}
\label{subsec:limitations}
We focused on a single open-sourced DM for each of the unconditional and conditional cases. At the same time, broader insights could be gained by applying our method to other models, including those in Fig.~\ref{fig:commercial_example}, which are unfortunately closed-source. Also, our experiments focus on open-sourced models, which tend to exhibit fewer problematic instances due to improved data curation in recent versions (\textit{e.g.}, SD3). Nonetheless, we consider this a strength, as unlearning in such advanced models reflects more realistic and subtle challenges.

\section{Conclusion}
\label{sec:conclusion}
We address the underexplored challenge of instance unlearning in diffusion models, targeted at unlearning specific outputs that are unpromptable. We propose a surrogate-based unlearning method leveraging diverse techniques that works in a prompt-free manner. Successfully performing unlearning while preserving the model integrity, our work contributes to ensuring the privacy and ethical compliance of generative models while maintaining generation performance.

\bibliographystyle{IEEEtran}
\bibliography{ref_IEEE_TIFS}

@article{zhong2024diffusion,
  title={Diffusion Tuning: Transferring Diffusion Models via Chain of Forgetting},
  author={Zhong, Jincheng and Guo, Xingzhuo and Dong, Jiaxiang and Long, Mingsheng},
  journal={NeurIPS},
  year={2024}
}

@inproceedings{xia2021tedigan,
  title={Tedigan: Text-guided diverse face image generation and manipulation},
  author={Xia, Weihao and Yang, Yujiu and Xue, Jing-Hao and Wu, Baoyuan},
  booktitle={CVPR},
  year={2021}
}

@article{wang2004image,
  title={Image quality assessment: from error visibility to structural similarity},
  author={Wang, Zhou and Bovik, Alan C and Sheikh, Hamid R and Simoncelli, Eero P},
  journal={IEEE TIP},
  year={2004},
  publisher={IEEE}
}

@article{heusel2017gans,
  title={Gans trained by a two time-scale update rule converge to a local nash equilibrium},
  author={Heusel, Martin and Ramsauer, Hubert and Unterthiner, Thomas and Nessler, Bernhard and Hochreiter, Sepp},
  journal={NeurIPS},
  year={2017}
}

@inproceedings{pizzi2022self,
  title={A self-supervised descriptor for image copy detection},
  author={Pizzi, Ed and Roy, Sreya Dutta and Ravindra, Sugosh Nagavara and Goyal, Priya and Douze, Matthijs},
  booktitle={CVPR},
  year={2022}
}

@article{ho2020denoising,
  title={Denoising diffusion probabilistic models},
  author={Ho, Jonathan and Jain, Ajay and Abbeel, Pieter},
  journal={NeurIPS},
  year={2020}
}

@article{song2020denoising,
  title={Denoising diffusion implicit models},
  author={Song, Jiaming and Meng, Chenlin and Ermon, Stefano},
  journal={ICLR},
  year={2020}
}

@inproceedings{seo2024generative,
  title={Generative Unlearning for Any Identity},
  author={Seo, Juwon and Lee, Sung-Hoon and Lee, Tae-Young and Moon, Seungjun and Park, Gyeong-Moon},
  booktitle={CVPR},
  year={2024}
}

@inproceedings{malnick2024taming,
  title={Taming Normalizing Flows},
  author={Malnick, Shimon and Avidan, Shai and Fried, Ohad},
  booktitle={WACV},
  year={2024}
}

@inproceedings{fan2024salun,
  title={SalUn: Empowering Machine Unlearning via Gradient-Based Weight Saliency in Both Image Classification and Generation},
  author={Fan, Chongyu and Liu, Jiancheng and Zhang, Yihua and Wei, Dennis and Wong, Eric and Liu, Sijia},
  booktitle={ICLR},
  year={2024}
}

@inproceedings{gandikota2023erasing,
  title={Erasing concepts from diffusion models},
  author={Gandikota, Rohit and Materzynska, Joanna and Fiotto-Kaufman, Jaden and Bau, David},
  booktitle={ICCV},
  year={2023}
}

@inproceedings{somepalli2023diffusion,
  title={Diffusion art or digital forgery? investigating data replication in diffusion models},
  author={Somepalli, Gowthami and Singla, Vasu and Goldblum, Micah and Geiping, Jonas and Goldstein, Tom},
  booktitle={CVPR},
  year={2023}
}

@inproceedings{wen2024detecting,
  title={Detecting, explaining, and mitigating memorization in diffusion models},
  author={Wen, Yuxin and Liu, Yuchen and Chen, Chen and Lyu, Lingjuan},
  booktitle={ICLR},
  year={2024}
}

@inproceedings{zhang2018unreasonable,
  title={The unreasonable effectiveness of deep features as a perceptual metric},
  author={Zhang, Richard and Isola, Phillip and Efros, Alexei A and Shechtman, Eli and Wang, Oliver},
  booktitle={CVPR},
  year={2018}
}

@inproceedings{golatkar2020eternal,
  title={Eternal sunshine of the spotless net: Selective forgetting in deep networks},
  author={Golatkar, Aditya and Achille, Alessandro and Soatto, Stefano},
  booktitle={CVPR},
  year={2020}
}

@inproceedings{huang2024learning,
  title={Learning to unlearn for robust machine unlearning},
  author={Huang, Mark He and Foo, Lin Geng and Liu, Jun},
  booktitle={ECCV},
  year={2024}
}

@article{yu2020gradient,
  title={Gradient surgery for multi-task learning},
  author={Yu, Tianhe and Kumar, Saurabh and Gupta, Abhishek and Levine, Sergey and Hausman, Karol and Finn, Chelsea},
  journal={NeurIPS},
  year={2020}
}

@article{somepalli2023understanding,
  title={Understanding and mitigating copying in diffusion models},
  author={Somepalli, Gowthami and Singla, Vasu and Goldblum, Micah and Geiping, Jonas and Goldstein, Tom},
  journal={NeurIPS},
  year={2023}
}

@article{chen2024extracting,
  title={Extracting Training Data from Unconditional Diffusion Models},
  author={Chen, Yunhao and Ma, Xingjun and Zou, Difan and Jiang, Yu-Gang},
  journal={arXiv preprint arXiv:2406.12752},
  year={2024}
}

@article{karras2019style,
  title={A Style-Based Generator Architecture for Generative Adversarial Networks},
  author={Karras, Tero},
  journal={arXiv preprint arXiv:1812.04948},
  year={2019}
}

@inproceedings{sohl2015deep,
  title={Deep Unsupervised Learning using Nonequilibrium Thermodynamics},
  author={Sohl-Dickstein, Jascha and Weiss, Eric A and Maheswaranathan, Niru and Ganguli, Surya},
  booktitle={Proceedings of the 32nd International Conference on Machine Learning (ICML)},
  pages={2256--2265},
  year={2015}
}

@article{heng2023selective,
  title={Selective amnesia: A continual learning approach to forgetting in deep generative models},
  author={Heng, Alvin and Soh, Harold},
  journal={NeurIPS},
  year={2023}
}

@article{silas2024data,
 title={Data unlearning in diffusion models},
 author={Alberti, Silas and Hasanaliyev, Kenan and Shah, Manav and Ermon, Stefano},
 journal={ICLR},
 year={2025}
}

@article{
meng2022sdedit,
title={{SDE}dit: Guided Image Synthesis and Editing with Stochastic Differential Equations},
author={Chenlin Meng and Yutong He and Yang Song and Jiaming Song and Jiajun Wu and Jun-Yan Zhu and Stefano Ermon},
journal={ICLR},
year={2022}
}

@article{
karras2018progressive,
title={Progressive Growing of {GAN}s for Improved Quality, Stability, and Variation},
author={Tero Karras and Timo Aila and Samuli Laine and Jaakko Lehtinen},
journal={ICLR},
year={2018}
}

@article{wu2024erasediff,
  title={Erasediff: Erasing data influence in diffusion models},
  author={Wu, Jing and Le, Trung and Hayat, Munawar and Harandi, Mehrtash},
  journal={arXiv preprint arXiv:2401.05779},
  year={2024}
}

@article{lu2024mace,
  title={Mace: Mass concept erasure in diffusion models},
  author={Lu, Shilin and Wang, Zilan and Li, Leyang and Liu, Yanzhu and Kong, Adams Wai-Kin},
  journal={CVPR},
  year={2024}
}

@article{gong2024reliable,
  title={Reliable and Efficient Concept Erasure of Text-to-Image Diffusion Models},
  author={Gong, Chao and Chen, Kai and Wei, Zhipeng and Chen, Jingjing and Jiang, Yu-Gang},
  journal={ECCV},
  year={2024}
}

@article{huang2023receler,
  title={Receler: Reliable concept erasing of text-to-image diffusion models via lightweight erasers},
  author={Huang, Chi-Pin and Chang, Kai-Po and Tsai, Chung-Ting and Lai, Yung-Hsuan and Wang, Yu-Chiang Frank},
  journal={ECCV},
  year={2024}
}

@article{aldahoul2024ai,
  title={AI-generated faces free from racial and gender stereotypes},
  author={AlDahoul, Nouar and Rahwan, Talal and Zaki, Yasir},
  journal={arXiv preprint arXiv:2402.01002},
  year={2024}
}

@article{hu2022lora,
  title={Lo{RA}: Low-rank adaptation of large language models},
  author={Hu, Edward J and Shen, Yelong and Wallis, Phillip and Allen-Zhu, Zeyuan and Li, Yuanzhi and Wang, Shean and Wang, Lu and Chen, Weizhu},
  journal={ICLR},
  year={2022}
}

@article{lyu2024one,
  title={One-dimensional Adapter to Rule Them All: Concepts Diffusion Models and Erasing Applications},
  author={Lyu, Mengyao and Yang, Yuhong and Hong, Haiwen and Chen, Hui and Jin, Xuan and He, Yuan and Xue, Hui and Han, Jungong and Ding, Guiguang},
  journal={CVPR},
  year={2024}
}

@article{liu2015faceattributes,
  title = {Deep Learning Face Attributes in the Wild},
  author = {Liu, Ziwei and Luo, Ping and Wang, Xiaogang and Tang, Xiaoou},
  journal = {ICCV},
  year = {2015} 
}

@article{cheong2024investigating,
  title={Investigating gender and racial biases in DALL-E Mini Images},
  author={Cheong, Marc and Abedin, Ehsan and Ferreira, Marinus and Reimann, Ritsaart and Chalson, Shalom and Robinson, Pamela and Byrne, Joanne and Ruppanner, Leah and Alfano, Mark and Klein, Colin},
  journal={ACM Journal on Responsible Computing},
  year={2024},
}

@article{golub1979generalized,
  title={Generalized cross-validation as a method for choosing a good ridge parameter},
  author={Golub, Gene H and Heath, Michael and Wahba, Grace},
  journal={Technometrics},
  volume={21},
  number={2},
  pages={215--223},
  year={1979},
  publisher={Taylor \& Francis}
}

@article{
park2024direct,
title={Direct Unlearning Optimization for Robust and Safe Text-to-Image Models},
author={Yong-Hyun Park and Sangdoo Yun and Jin-Hwa Kim and Junho Kim and Geonhui Jang and Yonghyun Jeong and Junghyo Jo and Gayoung Lee},
journal={NeurIPS},
year={2024},
}

@article{zhang2024forget,
  title={Forget-me-not: Learning to forget in text-to-image diffusion models},
  author={Zhang, Gong and Wang, Kai and Xu, Xingqian and Wang, Zhangyang and Shi, Humphrey},
  journal={CVPR},
  year={2024}
}

@article{kumari2023ablating,
  title={Ablating concepts in text-to-image diffusion models},
  author={Kumari, Nupur and Zhang, Bingliang and Wang, Sheng-Yu and Shechtman, Eli and Zhang, Richard and Zhu, Jun-Yan},
  journal={ICCV},
  year={2023}
}

@article{gandikota2024unified,
  title={Unified concept editing in diffusion models},
  author={Gandikota, Rohit and Orgad, Hadas and Belinkov, Yonatan and Materzy{\'n}ska, Joanna and Bau, David},
  journal={WACV},
  year={2024}
}

@article{zhang2024defensive,
  title={Defensive Unlearning with Adversarial Training for Robust Concept Erasure in Diffusion Models},
  author={Zhang, Yimeng and Chen, Xin and Jia, Jinghan and Zhang, Yihua and Fan, Chongyu and Liu, Jiancheng and Hong, Mingyi and Ding, Ke and Liu, Sijia},
  journal={NeurIPS},
  year={2024}
}

@article{lee2025concept,
  title={Concept Pinpoint Eraser for Text-to-image Diffusion Models via Residual Attention Gate},
  author={Lee, Byung Hyun and Lim, Sungjin and Lee, Seunggyu and Kang, Dong Un and Chun, Se Young},
  journal={ICLR},
  year={2025}
}

@article{tsai2024ring,
  title={Ring-A-Bell! How Reliable are Concept Removal Methods For Diffusion Models?},
  author={Tsai, Yu-Lin and Hsu, Chia-yi and Xie, Chulin and Lin, Chih-hsun and Chen, Jia You and Li, Bo and Chen, Pin-Yu and Yu, Chia-Mu and Huang, Chun-ying},
  journal={ICLR},
  year={2024}
}

@article{phamcircumventing,
  title={Circumventing Concept Erasure Methods For Text-To-Image Generative Models},
  author={Pham, Minh and Marshall, Kelly O and Cohen, Niv and Mittal, Govind and Hegde, Chinmay},
  journal={ICLR},
  year={2024}
}

@article{zhang2023generate,
  title={To generate or not? safety-driven unlearned diffusion models are still easy to generate unsafe images... for now},
  author={Zhang, Yimeng and Jia, Jinghan and Chen, Xin and Chen, Aochuan and Zhang, Yihua and Liu, Jiancheng and Ding, Ke and Liu, Sijia},
  journal={ECCV},
  year={2024}
}

@article{kirkpatrick2017overcoming,
  title={Overcoming catastrophic forgetting in neural networks},
  author={Kirkpatrick, James and Pascanu, Razvan and Rabinowitz, Neil and Veness, Joel and Desjardins, Guillaume and Rusu, Andrei A and Milan, Kieran and Quan, John and Ramalho, Tiago and Grabska-Barwinska, Agnieszka and others},
  journal={Proceedings of the national academy of sciences},
  year={2017},
}

@article{rolnick2019experience,
  title={Experience replay for continual learning},
  author={Rolnick, David and Ahuja, Arun and Schwarz, Jonathan and Lillicrap, Timothy and Wayne, Gregory},
  journal={NeurIPS},
  year={2019}
}

@article{chengrowth,
  title={Growth inhibitors for suppressing inappropriate image concepts in diffusion models},
  author={Chen, Die and Li, Zhiwen and Fan, Mingyuan and Chen, Cen and Zhou, Wenmeng and Wang, Yanhao and Li, Yaliang},
  journal={ICLR},
  year={2025}
}

@article{meng2025concept,
  title={Concept Corrector: Erase concepts on the fly for text-to-image diffusion models},
  author={Meng, Zheling and Peng, Bo and Jin, Xiaochuan and Lyu, Yueming and Wang, Wei and Dong, Jing},
  journal={arXiv preprint arXiv:2502.16368},
  year={2025}
}

@article{buierasing,
  title={Erasing Undesirable Concepts in Diffusion Models with Adversarial Preservation},
  author={Bui, Anh Tuan and Vuong, Long Tung and Doan, Khanh and Le, Trung and Montague, Paul and Abraham, Tamas and Phung, Dinh},
  journal={NeurIPS},
  year={2024}
}

@article{yoonsafree,
  title={SAFREE: Training-Free and Adaptive Guard for Safe Text-to-Image And Video Generation},
  author={Yoon, Jaehong and Yu, Shoubin and Patil, Vaidehi and Yao, Huaxiu and Bansal, Mohit},
  journal={ICLR},
  year={2025}
}

@article{thakral2025fine,
  title={Fine-Grained Erasure in Text-to-Image Diffusion-based Foundation Models},
  author={Thakral, Kartik and Glaser, Tamar and Hassner, Tal and Vatsa, Mayank and Singh, Richa},
  journal={arXiv preprint arXiv:2503.19783},
  year={2025}
}

@article{fantastic,
  title={Fantastic Targets for Concept Erasure in Diffusion Models and Where To Find Them},
  author={Bui, Anh and Vu, Trang and Vuong, Long and Le, Trung and Montague, Paul and Abraham, Tamas and Kim, Junae and Phung, Dinh},
  journal={ICLR},
  year={2025}
}

@article{nichol2022glide,
  title={{GLIDE}: Towards Photorealistic Image Generation and Editing with Text-Guided Diffusion Models},
  author={Nichol, Alexander Quinn and Dhariwal, Prafulla and Ramesh, Aditya and Shyam, Pranav and Mishkin, Pamela and Mcgrew, Bob and Sutskever, Ilya and Chen, Mark},
  journal={ICML},
  year={2022},
}

@article{chang2023muse,
  title={Muse: Text-To-Image Generation via Masked Generative Transformers},
  author={Chang, Huiwen and Zhang, Han and Barber, Jarred and Maschinot, Aaron and Lezama, Jose and Jiang, Lu and Yang, Ming-Hsuan and Murphy, Kevin Patrick and Freeman, William T and Rubinstein, Michael and others},
  journal={ICML},
  year={2023},
}

@article{zhang2023adding,
  title={Adding conditional control to text-to-image diffusion models},
  author={Zhang, Lvmin and Rao, Anyi and Agrawala, Maneesh},
  journal={ICCV},
  year={2023}
}

@article{peebles2023scalable,
  title={Scalable diffusion models with transformers},
  author={Peebles, William and Xie, Saining},
  journal={ICCV},
  year={2023}
}

@article{hongcogvideo,
  title={CogVideo: Large-scale Pretraining for Text-to-Video Generation via Transformers},
  author={Hong, Wenyi and Ding, Ming and Zheng, Wendi and Liu, Xinghan and Tang, Jie},
  journal={ICLR},
  year={2023}
}

@article{blattmann2023stable,
  title={Stable video diffusion: Scaling latent video diffusion models to large datasets},
  author={Blattmann, Andreas and Dockhorn, Tim and Kulal, Sumith and Mendelevitch, Daniel and Kilian, Maciej and Lorenz, Dominik and Levi, Yam and English, Zion and Voleti, Vikram and Letts, Adam and others},
  journal={arXiv preprint arXiv:2311.15127},
  year={2023}
}

@article{guoanimatediff,
  title={{AnimateDiff}: Animate Your Personalized Text-to-Image Diffusion Models without Specific Tuning},
  author={Guo, Yuwei and Yang, Ceyuan and Rao, Anyi and Liang, Zhengyang and Wang, Yaohui and Qiao, Yu and Agrawala, Maneesh and Lin, Dahua and Dai, Bo},
  journal={ICLR},
  year={2024}
}

@article{chen2024videocrafter2,
  title={Videocrafter2: Overcoming data limitations for high-quality video diffusion models},
  author={Chen, Haoxin and Zhang, Yong and Cun, Xiaodong and Xia, Menghan and Wang, Xintao and Weng, Chao and Shan, Ying},
  journal={CVPR},
  year={2024}
}

@article{wan2025wan,
  title={Wan: Open and advanced large-scale video generative models},
  author={Wan, Team and Wang, Ang and Ai, Baole and Wen, Bin and Mao, Chaojie and Xie, Chen-Wei and Chen, Di and Yu, Feiwu and Zhao, Haiming and Yang, Jianxiao and others},
  journal={arXiv preprint arXiv:2503.20314},
  year={2025}
}

@misc{rombach2022stable2.0,
  title={Stable diffusion 2.0 release},
  author={Rombach, Robin},
  howpublished = {\url{https://stability.ai/news/stable-diffusion-v2-release}},
  year={2022}
}

@article{esser2024scaling,
  title={Scaling rectified flow transformers for high-resolution image synthesis},
  author={Esser, Patrick and Kulal, Sumith and Blattmann, Andreas and Entezari, Rahim and M{\"u}ller, Jonas and Saini, Harry and Levi, Yam and Lorenz, Dominik and Sauer, Axel and Boesel, Frederic and others},
  journal={ICML},
  year={2024}
}

@article{rando2022red,
  title={Red-Teaming the Stable Diffusion Safety Filter},
  author={Rando, Javier and Paleka, Daniel and Lindner, David and Heim, Lennart and Tramer, Florian},
  journal={NeurIPS ML Safety Workshop},
  year={2022}
}

@misc{laborde2020nsfw,
  title={Nsfw detection machine learning model},
  author={Laborde, Gant},
  howpublished = {\url{https://github.com/GantMan/nsfw_model}},
  year={2020}
}

@misc{dalle2preview2022,
  title={{DALL-E} 2 preview - risks and limitations},
  author={OpenAI},
  year={2022}
}

@article{hunter2023ai,
  title={{AI} porn is easy to make now. For women, that's a nightmare.},
  author={Hunter, Tatum},
  journal={The Washington Post},
  year={2023},
}

@article{schramowski2023safe,
  title={Safe latent diffusion: Mitigating inappropriate degeneration in diffusion models},
  author={Schramowski, Patrick and Brack, Manuel and Deiseroth, Bj{\"o}rn and Kersting, Kristian},
  journal={CVPR},
  year={2023}
}

@article{guangyuanrecon,
 title={Recon: Reducing Conflicting Gradients From the Root For Multi-Task Learning},
 author={Guangyuan, SHI and Li, Qimai and Zhang, Wenlong and Chen, Jiaxin and Wu, Xiao-Ming},
 journal={ICLR},
 year={2023}
}

@article{jin2025unlearning,
  title={Unlearning as multi-task optimization: A normalized gradient difference approach with an adaptive learning rate},
  author={Jin, Xiaomeng and Bu, Zhiqi and Amazon, AGI and Vinzamuri, Bhanukiran and Ramakrishna, Anil and Chang, Kai-Wei and Cevher, Volkan and Hong, Mingyi},
journal={NAACL},
  year={2025}
}

@misc{EU2016GDPR,
  title        = {{Regulation (EU) 2016/679 of the European Parliament and of the Council of 27 April 2016 on the Protection of Natural Persons with Regard to the Processing of Personal Data and on the Free Movement of Such Data (General Data Protection Regulation)}},
  howpublished = "{Official Journal of the European Union, L 119, pp. 1–88}",
  year         = {2016},
  month        = {May},
  note         = {Available at: \url{https://eur-lex.europa.eu/eli/reg/2016/679/oj}},
  institution  = {European Parliament and the Council of the European Union}
}

@article{patel2025learning,
  title={Learning to unlearn while retaining: Combating gradient conflicts in machine unlearning},
  author={Patel, Gaurav and Qiu, Qiang},
  journal={arXiv preprint arXiv:2503.06339},
  year={2025}
}

@article{shamsian2025go,
  title={Go Beyond Your Means: Unlearning with Per-Sample Gradient Orthogonalization},
  author={Shamsian, Aviv and Shaar, Eitan and Navon, Aviv and Chechik, Gal and Fetaya, Ethan},
  journal={arXiv preprint arXiv:2503.02312},
  year={2025}
}

@article{zhang2024forgetting,
  title={Forgetting and remembering are both you need: Balanced graph structure unlearning},
  author={Zhang, Chenhan and Wang, Weiqi and Tian, Zhiyi and Yu, Shui},
  journal={IEEE Transactions on Information Forensics and Security},
  year={2024},
  publisher={IEEE}
}

@article{liu2025privacy,
  title={Privacy-preserving federated unlearning with certified client removal},
  author={Liu, Ziyao and Ye, Huanyi and Jiang, Yu and Shen, Jiyuan and Guo, Jiale and Tjuawinata, Ivan and Lam, Kwok-Yan},
  journal={IEEE Transactions on Information Forensics and Security},
  year={2025},
  publisher={IEEE}
}

@article{wang2023machine,
  title={Machine unlearning via representation forgetting with parameter self-sharing},
  author={Wang, Weiqi and Zhang, Chenhan and Tian, Zhiyi and Yu, Shui},
  journal={IEEE Transactions on Information Forensics and Security},
  year={2023},
  publisher={IEEE}
}

@article{chundawat2023zero,
  title={Zero-shot machine unlearning},
  author={Chundawat, Vikram S and Tarun, Ayush K and Mandal, Murari and Kankanhalli, Mohan},
  journal={IEEE Transactions on Information Forensics and Security},
  volume={18},
  pages={2345--2354},
  year={2023},
  publisher={IEEE}
}

@article{jeong2025upsample,
 author = {Jeong, Wongi and Lee, Kyungryeol and Seo, Hoigi and Chun, Se Young},
 journal = {arXiv preprint arXiv:2507.08422},
 title = {Upsample what matters: Region-adaptive latent sampling for accelerated diffusion transformers},
 year = {2025}
}

@article{ma2024deepcache,
 author = {Ma, Xinyin and Fang, Gongfan and Wang, Xinchao},
 journal = {CVPR},
 title = {{Deepcache}: Accelerating diffusion models for free},
 year = {2024}
}

\newpage

\section*{Biography Section}
 

\vspace{-15pt}
\begin{IEEEbiography}[{\includegraphics[width=1in,height=1.25in,clip,keepaspectratio]{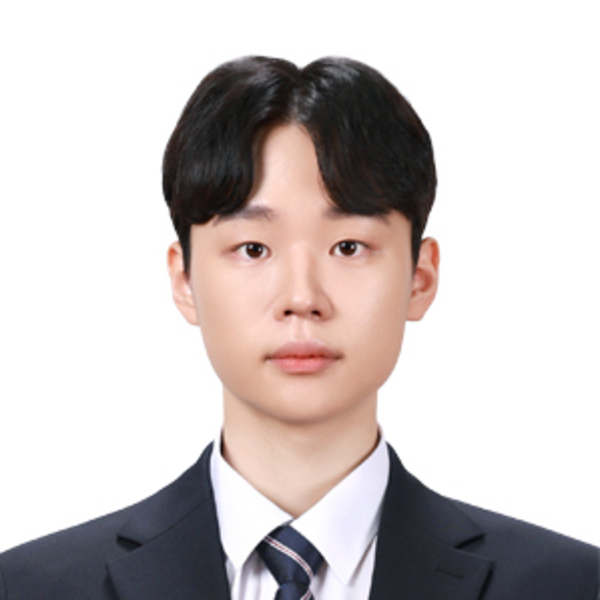}}]{Kyungryeol Lee} received his B.S. degrees in electrical and computer engineering (ECE) from Seoul National University (SNU), Seoul, South Korea, in 2024. He is currently pursuing an integrated M.S./Ph.D. degree in electrical and computer engineering (ECE) at Seoul National University. His research interests include deep generative models and efficient AI.
\end{IEEEbiography}
\vspace{-15pt}
\begin{IEEEbiography}[{\includegraphics[width=1in,height=1.25in,clip,keepaspectratio]{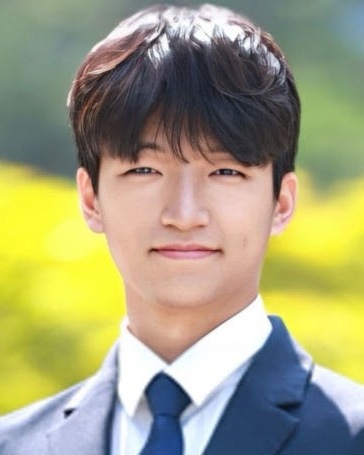}}]{Kyeonghyun Lee} received his B.S. degrees in physics and in electrical and computer engineering (ECE) from Seoul National University (SNU), Seoul, South Korea, in 2024. He is currently pursuing an integrated M.S./Ph.D. degree in electrical and computer engineering (ECE) at Seoul National University. His research interests include image generative models and imaging systems.
\end{IEEEbiography}
\vspace{-15pt}
\begin{IEEEbiography}[{\includegraphics[width=1in,height=1.25in,clip,keepaspectratio]{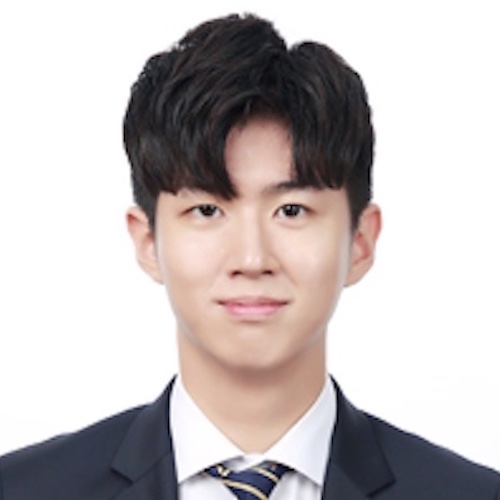}}]{Seongmin Hong}
 received the B.S. and Ph.D. degrees in electrical and computer engineering from Seoul National University (SNU), Seoul, South Korea, in 2020 and 2025, respectively. He received the Distinguished Ph.D. Dissertation Award in Communications and Signal Processing.
He is currently a Postdoctoral Researcher with the Institute of New Media and Communications, Seoul National University. His research interests primarily focus on deep generative models (especially diffusion models and their inversion), optimization algorithms, and computational imaging systems (such as MRI and automotive radar).
\end{IEEEbiography}
\vspace{-15pt}
\begin{IEEEbiography} [{\includegraphics[width=1in,height=1.25in,clip,keepaspectratio]{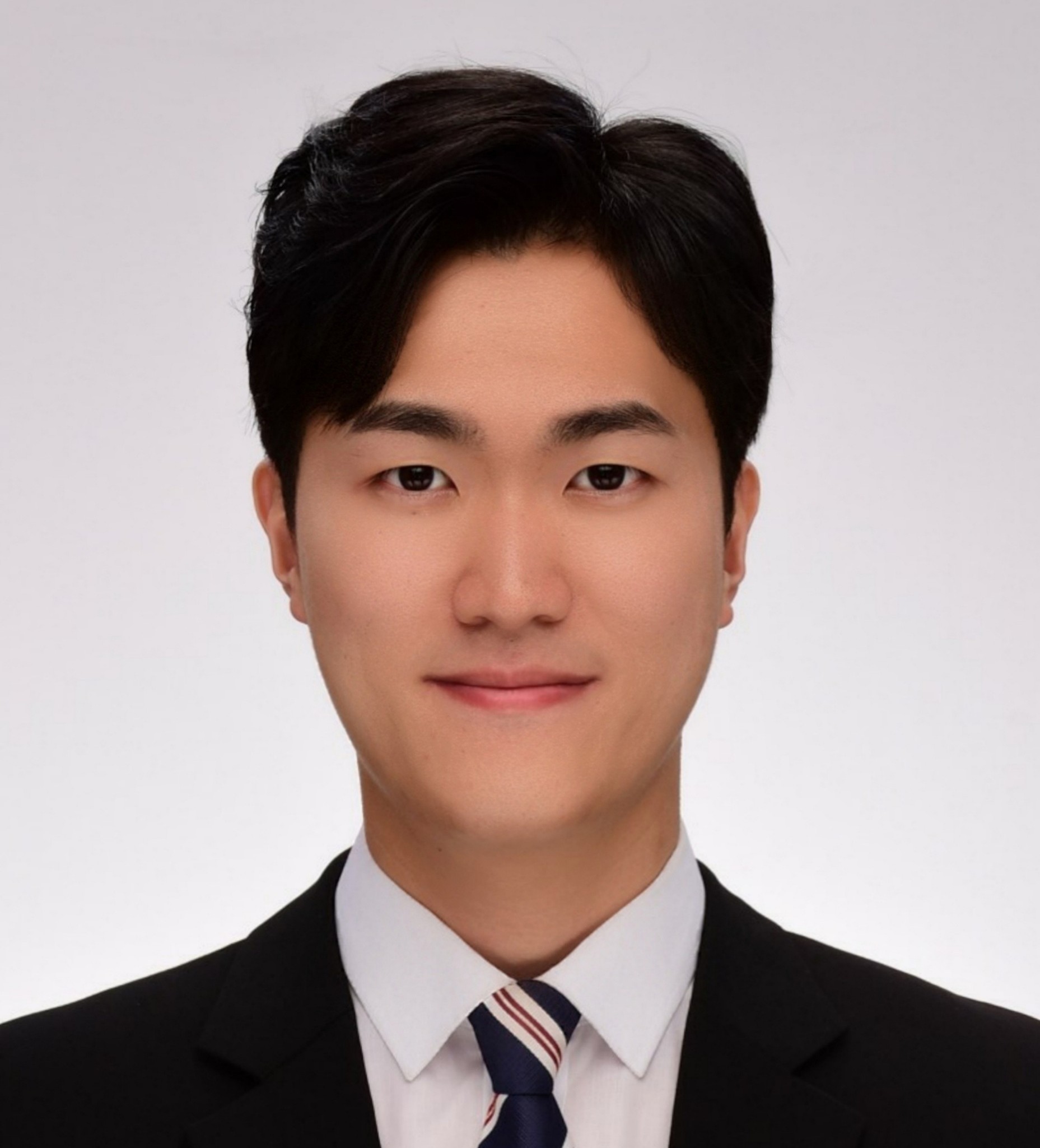}}]{Byung Hyun Lee} 
received his B.S. degree in ECE from Ulsan National Institute of Science and Technology (UNIST), Ulsan, South Korea, in 2018.
He is currently pursuing an integrated M.S./Ph.D. in ECE at Seoul National University, Seoul, South Korea.
His research interests include continual learning and unlearning in machine learning and its application to various computer vision tasks.
\end{IEEEbiography}
\vspace{-15pt}
\begin{IEEEbiography}[{\includegraphics[width=1in,height=1.25in,clip,keepaspectratio]{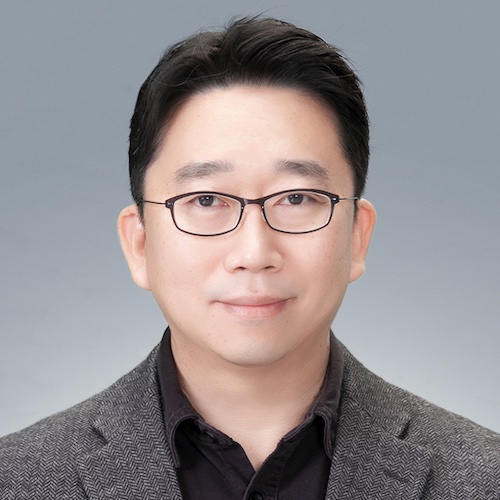}}]{Se Young Chun} received his B.S.E. degree in electrical engineering (EE) from Seoul National University in 1999 and his Ph.D. degree in EE: systems from the University of Michigan - Ann Arbor in 2009. He was a research fellow at Harvard Medical School and also a research fellow at the University of Michigan - Ann Arbor. He has been with UNIST from 2013 to 2021 as an Assistant / Associate Professor in EE / AI. Since 2021, he has joined the Department of Electrical and Computer Engineering, Seoul National University, Seoul, South Korea, as an Associate Professor and now he is currently a Full Professor with tenure. He is a senior area editor of IEEE Transactions on Computational Imaging, a member of IEEE Bio Imaging and Signal Processing Technical Committee and an IEEE Biometrics Council Representative of IEEE Vehicular Technology Society. He was the recipient of the 2015 Bruce Hasegawa Young Investigator Medical Imaging Science Award from the IEEE Nuclear and Plasma Sciences Society. His research interests include computational imaging algorithms, generative diffusion models, multimodal foundation models for the applications in imaging, and computer vision for robotics. 
\end{IEEEbiography}

\vspace{11pt}


\vfill

\end{document}